\documentclass{article}

\usepackage{iclr2027_conference,times}

\usepackage{amsmath,amssymb,amsthm,mathtools,mathrsfs,bm}
\usepackage{booktabs,tabularx,array}
\usepackage{enumitem}
\usepackage{microtype}
\usepackage{url}
\usepackage{algorithm}
\usepackage{algpseudocode}
\usepackage{tikz}
\usetikzlibrary{arrows.meta,calc,positioning}
\usepackage{wrapfig}
\usepackage[dvipsnames]{xcolor}
\usepackage[colorlinks, linkcolor=OliveGreen, urlcolor=BrickRed, citecolor=NavyBlue]{hyperref}
\usepackage{nicefrac}
\usepackage{titletoc}
\makeatletter
\newcommand\fs@compactruled{%
  \def\@fs@cfont{\bfseries}%
  \let\@fs@capt\floatc@ruled
  \def\@fs@pre{\hrule height.8pt depth0pt \kern2pt}%
  \def\@fs@post{\kern2pt\hrule\kern-.8em\relax}%
  \def\@fs@mid{\kern2pt\hrule\kern2pt}%
  \let\@fs@iftopcapt\iftrue}
\floatstyle{compactruled}
\restylefloat{algorithm}
\renewcommand{\theHALG@line}{\thealgorithm.\arabic{ALG@line}}
\makeatother
\usepackage[capitalise]{cleveref}
\crefname{equation}{}{}
\crefname{figure}{Fig.}{Figs.}
\crefname{section}{Sec.}{Secs.}
\crefname{appendix}{App.}{Apps.}
\crefname{table}{Tab.}{Tabs.}

\newtheorem{theorem}{Theorem}
\crefname{theorem}{Thm.}{Thms.}
\newtheorem{lemma}{Lemma}
\crefname{lemma}{Lem.}{Lems.}
\newtheorem{assumption}{Assumption}
\crefname{assumption}{Assump.}{Assumps.}
\newtheorem{proposition}{Proposition}
\crefname{proposition}{Prop.}{Props.}
\newtheorem{corollary}{Corollary}
\crefname{corollary}{Cor.}{Cors.}

\crefname{definition}{Def.}{Defs.}

\crefname{remark}{Rmk.}{Rmks.}
\crefname{algocf}{Alg.}{Algs.}
\crefname{algorithm}{Alg.}{Algs.}

\makeatletter
\newcommand{\appendixcontents}{%
  \section*{Table of Contents}%
  \setcounter{tocdepth}{3}%
  \@starttoc{toc}%
}
\makeatother

\newcommand{\R}{\mathbb R}
\newcommand{\E}{\mathbb E}
\newcommand{\Law}{\operatorname{Law}}
\newcommand{\TV}{\operatorname{TV}}
\newcommand{\KL}{\operatorname{KL}}

\newcommand{\cN}{\mathcal N}
\newcommand{\eps}{\varepsilon}
\renewcommand{\d}{\mathrm{d}}
\newcommand{\e}{\mathrm{e}}
\let\origi\i
\renewcommand{\i}{\ifmmode\mathrm{i}\else\origi\fi}
\newcommand{\tint}{\mathop{\textstyle\int}}

\newcommand{\norm}[1]{\lVert#1\rVert}
\newcommand{\abs}[1]{\lvert#1\rvert}
\newcommand{\ip}[2]{\langle#1,#2\rangle}
\newcommand{\Cov}{\operatorname{Cov}}
\newcommand{\Var}{\operatorname{Var}}
\newcolumntype{Y}{>{\raggedright\arraybackslash}X}
\renewcommand{\top}{\mathrm{T}} %
\renewcommand{\paragraph}[1]{\textbf{#1}\quad}

\title{Reflected Proximal Dynamics for Accelerated High-Accuracy Log-Concave Sampling and Beyond}

\author{%
Bo Yuan \quad Wei Guo \quad Yongxin Chen\\
Georgia Institute of Technology\\
\texttt{\{byuan48, wei.guo, yongchen\}@gatech.edu}
}

\iclrfinalcopy %
\begin{document}

\maketitle

\begin{abstract}

We study sampling from $\pi(x)\propto\mathrm e^{-V(x)}$ on $\mathbb R^d$. The proximal sampler preserves $\pi$ exactly and recovers overdamped Langevin dynamics in the small-step limit, but its Gaussian resampling step discards directional information, whereas underdamped Langevin dynamics accelerates exploration through momentum. We introduce {\em reflected proximal dynamics}, which brings directional persistence into proximal sampling by combining an exact reflection of the auxiliary variable, conditional kinetic transport, and controlled refreshment. We develop two realizations: a flux-resampled underdamped half-turn using temporary momentum and a persistent-momentum quarter-turn that carries momentum across iterations and invokes the restricted Gaussian oracle only on rare refresh steps. For smooth, strongly log-concave targets with condition number $\kappa$, both methods achieve total variation error $\varepsilon$ using $\widetilde O(\sqrt{\kappa}d^{1/4})$ expected oracle queries from a constant-$\chi^2$ warm start, and $\widetilde O(\kappa^{6/7}d^{1/4})$ queries end-to-end when combined with an optimized coarse-step proximal warm-up, both with polylogarithmic dependence on $1/\varepsilon$. Finally, strong log-concavity is not essential to the mechanism: under a Poincar\'e inequality with gap $m$ and $-KI\preceq\nabla^2V\preceq\beta I$, both realizations retain their warm-start guarantee in $\widetilde O(\sqrt{\beta/m}\sqrt{1+K/m}\,d^{1/4})$ expected oracle queries.

\end{abstract}

\section{Introduction}
\label{sec:introduction}

We study sampling from a distribution $\pi(x)\propto \e^{-V(x)}$ on $\R^d$, focusing first on $\alpha$-strongly convex and $\beta$-smooth potentials with condition number $\kappa=\beta/\alpha$ before extending to non-log-concave targets under a Poincar\'e inequality. This problem appears in fields including Bayesian inference, machine learning, and statistical physics \citep{gelman1995bayesian,ghosal2017fundamentals,newman1999monte,chewi2023log}, and a classical approach is Markov chain Monte Carlo (MCMC).
An MCMC method constructs a Markov chain with stationary distribution \(\pi\) and, under suitable conditions on the transition kernel and target distribution, converges to \(\pi\).
Such samples allow expectations under \(\pi\) to be approximated even when direct integration and the normalizing constant of \(\pi\) are unavailable.
Two workhorse continuous-time processes are (overdamped) Langevin dynamics (LD) and underdamped Langevin dynamics (ULD). LD is a gradient flow of the potential energy function \(V\) combined with Brownian noise, and by Euler discretization, it leads to the Langevin Monte Carlo (LMC, \citet{dalalyan2017theoretical,durmus2019analysis,vempala2019rapid,chewi2022analysis}) algorithm. ULD introduces momentum, allowing the process to retain its direction rather than move diffusively \citep{cheng2018underdamped,dalalyan2020sampling,ma2021is,betancourt2019hamiltonian}. This persistence can accelerate exploration, but discretizing the coupled position--momentum dynamics requires more sophisticated numerical schemes \citep{shen2019randomized,zhang2023improved}.

The proximal sampler is a discrete-time alternative to LD that preserves the target exactly \citep{lee2021structured,chen2022improved}.  It augments the position \(X\sim\pi\) with an auxiliary variable \(Y\sim\cN(X,\eta I_d)\), then alternates a Gaussian draw of \(Y\mid X\) with a localized draw of \(X\mid Y\), called a restricted Gaussian oracle (RGO) draw.  In the $\eta\to0$ limit, the chain reduces to LD \citep{mitra2025fast}.  This construction also exposes a gap: redrawing the Gaussian displacement \(Y-X\) at every iteration discards directional memory, whereas ULD retains precisely this information through momentum.

We bridge this gap with \textbf{reflected proximal dynamics}.  Instead of redrawing the auxiliary variable, we set $Y\gets2X-Y$, which leaves the position fixed and reverses the current displacement.  Because the Gaussian displacement is symmetric about zero, this reflection preserves the joint distribution of $(X,Y)$.  Reflection alone does not move $X$, so most iterations follow it with a conditional transition that turns the retained displacement into persistent motion; occasional independent conditional draws remove correlations that this persistent update may leave behind.  The resulting framework (\cref{alg:reflected-proximal-overview}) brings the directional persistence of ULD into a target-preserving construction.

We instantiate this framework with two implementable realizations: a \textit{flux-resampled half-turn} (\cref{sec:flux}) and a \textit{persistent quarter-turn} (\cref{sec:persistent}).  The former uses a damped conditional half-turn and then discards its temporary momentum; the latter rotates position information into momentum and carries it across outer iterations. Both use the same reflection--refresh framework but store directional information differently.  When a warm start is not supplied, we combine final sampling with a proximal warm-up at an optimized coarse step size and account for both stages in the query budget (\cref{sec:warm-start-removal}).  In either regime, geometric contraction makes the ideal mixing horizon logarithmic in $1/\eps$ (where $\eps$ is the desired accuracy), and the algorithmic complexity depends on $\eps$ only through polylogarithmic factors.

Concurrent work on the proximal bouncy particle sampler (PBPS, \citet{chen2026accelerated}) also employs Gaussian reflection and achieves the same warm-start oracle complexity; we therefore view PBPS as another implementation of the reflected proximal framework, and our flux half-turn shares its half-turn structure. Our work develops complementary conditional dynamics, including a momentum-persistent construction that avoids independent conditional sampling on most iterations, and it goes beyond the strongly log-concave setting analyzed by PBPS. The resulting analysis connects directional persistence, conditional invariance, and accelerated convergence within a unified proximal sampling framework. \cref{sec:related-work} gives a detailed comparison.

We summarize our main \textbf{contributions} as follows:
\begin{itemize}[leftmargin=*,itemsep=0.1em,topsep=0.1em]
\item We introduce reflected proximal dynamics and prove that it preserves the target distribution if the accelerated branch preserves the conditional distribution sampled by the RGO (\cref{sec:framework}).
\item We provide two realizations of the framework and prove accelerated contraction: both achieve \(\widetilde O(\sqrt\kappa d^{1/4})\) expected oracle complexity from a constant-\(\chi^2\) warm start (\cref{sec:flux,sec:persistent}).
\item We remove the warm start requirement with an optimized warm-up stage, yielding an end-to-end expected oracle complexity of \(\widetilde O(\kappa^{6/7}d^{1/4})\) (\cref{sec:warm-start-removal}).
\item We show that the acceleration mechanism does not require log-concavity: for targets satisfying a Poincar\'e inequality with gap \(m\) and potentials with \(-KI\preceq\nabla^2V\preceq\beta I\), both realizations keep their warm-start guarantee at the cost of \(\widetilde O(\sqrt{\beta/m}\sqrt{1+K/m}\,d^{1/4})\) expected oracle queries (\cref{sec:pi-extension}).
\end{itemize}

\section{Preliminaries}
\label{sec:preliminaries}

\paragraph{Notations.}
We write \(\Law(X)\) for the law of a random variable \(X\).
For probability densities \(\nu\) and \(\pi\), write the total variation distance as
\(\TV(\nu,\pi):=\sup_A\abs{\nu(A)-\pi(A)}\).  When \(\nu\ll\pi\), the
order-\(r\) R\'enyi divergence for \(r>1\) is
\(D_r(\nu\Vert\pi):=\frac{1}{r-1}\log\int
(\frac{\nu(x)}{\pi(x)})^r\pi(x)\,\d x\).  In
particular, the $\chi^2$ divergence is
\(\chi^2(\nu\Vert\pi):=\int(\frac{\nu(x)}{\pi(x)}-1)^2\pi(x)\,\d x
=\e^{D_2(\nu\Vert\pi)}-1\), and
\(\TV(\nu,\pi)\le\tfrac12\sqrt{\chi^2(\nu\Vert\pi)}\).  We call \(\nu\)
a constant-\(D_2\) (or constant-$\chi^2$) warm start of \(\pi\) if
\(D_2(\nu\Vert\pi)=O(1)\), equivalently if \(\chi^2(\nu\Vert\pi)=O(1)\).
We write $\Var_\pi(f):=\E_\pi f^2-(\E_\pi f)^2$ and $[a]_+:=\max\{a,0\}$.
The notations \(\widetilde O(\cdot), \widetilde\Theta(\cdot)\) suppress polylogarithmic factors in the displayed problem parameters.

We now state the main assumption on the target distribution considered through \cref{sec:warm-start-removal}. In \cref{sec:pi-extension}, we consider a weaker assumption that allows non-log-concave targets.

\begin{assumption}[Strong log-concavity and smoothness]
\label{ass:target}
We assume that the target density is
\(\pi(x)\propto\e^{-V(x)}\), where the potential
\(V:\R^d\to\R\) is \(\alpha\)-strongly convex and \(\beta\)-smooth,
i.e., \(\tfrac{\alpha}{2}\norm{x-x'}^2\le
V(x)-V(x')-\ip{\nabla V(x')}{x-x'}\le
\tfrac{\beta}{2}\norm{x-x'}^2\) for all \(x,x'\in\R^d\).  We denote the condition
number by \(\kappa:=\beta/\alpha\).
\end{assumption}

\paragraph{Query conventions.}
We assume oracle access to \(V\) and \(\nabla V\), and let \(N_{\mathrm{qry}}\) denote the total number of oracle evaluations used by an algorithm. This count includes rejected proposals, event-rate evaluations, mode-finding steps, and all work internal to an approximate RGO.

\subsection{Overdamped and Underdamped Langevin Dynamics}
\label{sec:preliminaries-ld}
For $d$-dimensional standard Brownian motion \((B_t)_{t\ge0}\), (overdamped) Langevin dynamics (LD) follows
\[\d X_t=-\nabla V(X_t)\,\d t+\sqrt2\,\d B_t,\]
whereas underdamped Langevin dynamics (ULD) introduces momentum
and follows
\[\d X_t=P_t\,\d t \quad \text{and} \quad
\d P_t=-\nabla V(X_t)\,\d t-\gamma P_t\,\d t+\sqrt{2\gamma}\,\d B_t.\]
Their stationary distributions are \(\pi\) and \(\pi\otimes\gamma_d\),
respectively, where \(\gamma_d=\cN(0,I_d)\).  The distinction relevant to this work is how the two
dynamics move through space.  In LD, Brownian noise acts directly on
position, so motion over short time scales is diffusive and carries no
separate directional state from one interval to the next; ULD instead
stores direction in the momentum: before friction substantially changes
\(P_t\), the position moves approximately as
\(X_{t+h}\approx X_t+hP_t\).  This short-time transport can cross a region
more coherently than a position-only diffusion.  More precisely, let
\(\mu_t=\Law(X_t)\) for LD and
\(\widehat\mu_t=\Law(X_t,P_t)\) for ULD.
For initial laws with finite \(\chi^2\) divergence from their respective stationary distributions, under \cref{ass:target}, there are universal constants \(c,C>0\) such that
$\chi^2(\mu_t\Vert\pi)
 \le \e^{-2\alpha t}\chi^2(\mu_0\Vert\pi)$ and
 $\chi^2(\widehat\mu_t\Vert\pi\otimes\gamma_d)
 \le C\e^{-c\sqrt\alpha t}
       \chi^2(\widehat\mu_0\Vert\pi\otimes\gamma_d),
$
where the ULD bound uses the low-friction choice
\(\gamma\asymp\sqrt\alpha\)
\citep{vempala2019rapid,chewi2023log,chen2026exact}.  Thus, after
normalizing \(\beta=1\), the two continuous-time complexities are
respectively \(O(\kappa)\) and \(O(\sqrt\kappa)\).

\subsection{The Proximal Sampler}
\label{sec:preliminaries-ps}
The proximal sampler \citep{lee2021structured,chen2022improved} augments \(X\sim\pi\) with an auxiliary random variable \(Y\) under the joint distribution $\Pi_\eta(x,y) \propto
\exp({-V(x)-\norm{x-y}^2/(2\eta)}),
$
where \(\eta>0\) is the \textit{step size}.  Its \(X\)-marginal is \(\pi\), and
the conditional distributions are
\begin{align}
Y\mid X=x\sim\cN(x,\eta I_d),
\qquad
X\mid Y=y\sim q_y(x)\propto\exp(-V(x)-\norm{x-y}^2/(2\eta)).
\label{eq:rgo}
\end{align}
An exact independent draw from \(q_y\) is called an ideal \textbf{restricted
Gaussian oracle (RGO)} draw.  The standard proximal sampler is the Gibbs
chain that draws
\(Y_{k+1}\sim\cN(X_k,\eta I_d)\) and then
\(X_{k+1}\sim q_{Y_{k+1}}\). $q_y$ has condition number
\((\beta+\nicefrac{1}{\eta})/(\alpha+\nicefrac{1}{\eta})\),
which approaches $1$ as \(\eta\to0\).  The small-\(\eta\) limit
recovers LD, while the ideal Gibbs chain preserves \(\pi\) for every
\(\eta\) \citep{mitra2025fast}.

\subsection{Poisson Process and Poisson Thinning}
\label{sec:preliminaries-poisson-process}
A \textit{Poisson process} is a random clock that determines when events occur.
Its intensity \((\lambda_t)_{t\ge0}\) describes how frequently the clock should ring near time \(t\).
Conditional on the history up to time \(t\), an event occurs during the next short interval \([t,t+\d t)\) with probability \(\lambda_t\,\d t+o(\d t)\);
equivalently, the next event time \(\tau\) after \(t\) satisfies
\(\Pr(\tau>s \mid (\lambda_r)_{r\in[t,s]})=\exp(-\int_t^s\lambda_r\,\d r)\), $\forall s\ge t$.
However, directly sampling this clock can be difficult because the integrated intensity $\int_t^\cdot\lambda_r\,\d r$ generally has no closed-form inverse.
\textit{Poisson thinning} avoids this inversion. Suppose \(\lambda_s\le\bar\lambda\) for all $s\ge t$. One first proposes candidate times from a homogeneous Poisson process of rate \(\bar\lambda\), by sampling the waiting time to the next event from the exponential density \(u\mapsto\bar\lambda\e^{-\bar\lambda u}1_{u>0}\).
At a candidate time \(s\), accept it as an event with probability \(\lambda_s/\bar\lambda\); otherwise, reject it and continue the trajectory.
The accepted times thus have the desired state-dependent intensity \((\lambda_t)_{t\ge 0}\).
\section{Reflected Proximal Dynamics}
\label{sec:framework}

The standard proximal sampler redraws the auxiliary variable \(Y_{k+1}\sim\cN(X_k,\eta I_d)\) at
every iteration.  This makes each update simple, but it also discards the
direction of the previous move.  Reflected proximal dynamics modifies this
outer update while preserving the augmented target.  We first present
the framework with state \((X,Y)\); the momentum-augmented version is
introduced in \cref{sec:persistent}.

\subsection{The Framework}

\begin{algorithm}[htbp]
\refstepcounter{algorithm}
\vspace{0.35em}
\label{alg:reflected-proximal-overview}
{\raggedright\textbf{Algorithm \thealgorithm: Ideal reflected proximal dynamics}\par}
\vspace{0.3em}
\hrule
\vspace{0.35em}
\begin{algorithmic}[1]
\Require Step size \(\eta\), initial states \((X_0,Y_0)\), total steps $N$, refresh probability \(\rho\), and kernels \(M_y\) preserving \(q_y\) for any \(y\in\R^d\)
\For{\(k=0,\ldots,N-1\)}
  \State \(Y_{k+1}\gets2X_k-Y_k\) \Comment{reflect}
  \State Draw \(U_k\sim\operatorname{Unif}[0,1]\)
  \If{\(U_k\le\rho\)} \Comment{RGO branch}
    \State Draw \(X_{k+1}\sim q_{Y_{k+1}}\) independently
  \Else \Comment{accelerated branch}
    \State \label{alg:framework-conditional-line}Draw
    \(X_{k+1}\sim M_{Y_{k+1}}(X_k,\cdot)\)
  \EndIf
\EndFor
\State \Return \(X_N\)
\end{algorithmic}
\vspace{0.1em}
\end{algorithm}

Recall the conditional distribution \(q_y(x)\) in \cref{eq:rgo}.  For each fixed \(y\), let \(M_y\) be a Markov kernel that preserves \(q_y\).
Unlike an RGO draw which depends only on \(Y_{k+1}\) when sampling \(X_{k+1}\), the kernel \(M_{Y_{k+1}}\) may also use the previous iterate \(X_k\).  The alternating updates of the ideal framework are given in \cref{alg:reflected-proximal-overview}.
The algorithm is \textit{ideal} because the RGO branch draws exactly from \(q_y\).
The realizations in \cref{sec:flux,sec:persistent} replace the required exact
conditional draws by controlled finite constructions, yielding
implementable algorithms.  We call the
second branch \textit{accelerated} because it is where directional persistence is
introduced; invariance of an arbitrary \(M_y\) alone does not imply
acceleration.

\begin{proposition}[Exactness of the framework]
\label{prop:framework-exactness}
If \(M_y\) preserves \(q_y\) for every \(y\),
then
\cref{alg:reflected-proximal-overview} preserves \(\Pi_\eta\).  Its
position marginal therefore has invariant law \(\pi\).
\end{proposition}
\begin{proof}
Let \(\varphi_\eta:=\cN(0,\eta I_d)\), so that \(\Pi_\eta(x,y)=\pi(x)\varphi_\eta(y-x)\).  For every \(f\),
\begin{align*}
    \tint f(x,2x-y)\,\pi(x)\varphi_\eta(y-x)\,\d x\,\d y
    =\tint f(x,y')\,\pi(x)\varphi_\eta(x-y')\,\d x\,\d y'
    =\tint f\,\d\Pi_\eta.
\end{align*}
Hence the reflection preserves \(\Pi_\eta\).  Now fix \(y\) and let
\(\overline M_y(x,x')=(1-\rho)M_y(x,x')+\rho q_y(x')\) be the mixture distribution of the accelerated and RGO branches.  Since \(M_y\) preserves \(q_y\),
\begin{align*}
  \tint q_y(x)M_y(x,x')\,\d x = q_y(x')
  \implies
  \tint[\tint f(x',y)\overline M_y(x,x')\,\d x']q_y(x)\,\d x
 =\tint f(x',y)q_y(x')\,\d x'.
\end{align*}
Thus the second update preserves the conditional law \(q_y\) for every \(y\).  Composing it with the reflection preserves \(\Pi_\eta\), whose position marginal is \(\pi\).
\end{proof}

\subsection{Why Reflection and Refreshment?}

The standard proximal sampler redraws \(Y\mid X\) independently at every iteration.  This update preserves the target, but it also replaces the current displacement \(Y-X\) by fresh Gaussian noise and therefore discards the direction of the previous move.  To retain this information, reflected proximal dynamics instead sets \(Y^+=2X-Y\).  The reflection reverses the displacement without changing its magnitude and, by \cref{prop:framework-exactness}, preserves the augmented target \(\Pi_\eta\).

Reflection alone leaves \(X\) unchanged, so the subsequent conditional update must turn the retained displacement into motion.  The accelerated branch does this by starting from the current position and applying a \(q_{Y^+}\)-invariant kernel.  Unlike an independent RGO draw, its output can remain correlated with its input; this correlation allows the direction retained by reflection to persist across iterations. Successive position increments tend to point in the same direction, so they add rather than cancel. The two concrete kernels in \cref{sec:flux,sec:persistent} are designed so that this persistence produces transport over the target scale.

Preserving \(q_y\), however, does not by itself imply contraction: the identity kernel is invariant, and a nearly deterministic kernel can retain periodic or weakly mixed components.  We therefore take an independent RGO step with a small \textit{refresh probability} \(\rho\).  The accelerated branch preserves direction on most iterations, while the refresh branch removes correlations that the persistent motion alone may fail to dissipate.  The one-dimensional example below illustrates how this balance changes diffusive motion into approximately ballistic transport.

\subsection{A One-Dimensional Picture of Acceleration}

Consider the quadratic potential
\(V(x)=\alpha x^2/2\).  The target has standard deviation
\(1/\sqrt\alpha\), whereas one local proximal update moves on the scale
\(\sqrt\eta\).  In the standard proximal sampler, each Gaussian redraw
forgets the previous direction, so the signed
increments are centered and approximately uncorrelated. As a result, the typical net displacement after \(n\) steps is only
\(\sqrt{n\eta}\).  Reaching the target scale requires roughly
\(1/(\alpha\eta)\) iterations.
Reflection keeps the previous displacement available to the accelerated
branch.  If this branch makes several successive moves point in a correlated
direction, their signed increments add \textit{approximately linearly} (\(n\sqrt\eta\)) rather than
cancelling as a random walk (\(\sqrt{n\eta}\)).  Achieving net displacement
\(1/\sqrt\alpha\) then takes roughly
\(1/\sqrt{\alpha\eta}\) iterations.  The occasional RGO branch prevents the
chain from remaining on a periodic trajectory without erasing its direction at every step.
For both of our proposed realizations, one iteration uses
\(\widetilde O(1+\beta\eta\sqrt d)\) oracle queries.  Multiplying this cost
by \(\widetilde O(1/\sqrt{\alpha\eta})\) iterations
and choosing an optimal \(\eta=\widetilde\Theta(1/(\beta\sqrt d))\) gives
\(\widetilde O(\sqrt\kappa d^{1/4})\) oracle queries; \cref{sec:flux,sec:persistent} make this contraction and query bound rigorous for the two resulting chains.

\begin{figure}[t]
\vspace{-0.5em}
\centering
\includegraphics[width=0.8\textwidth]{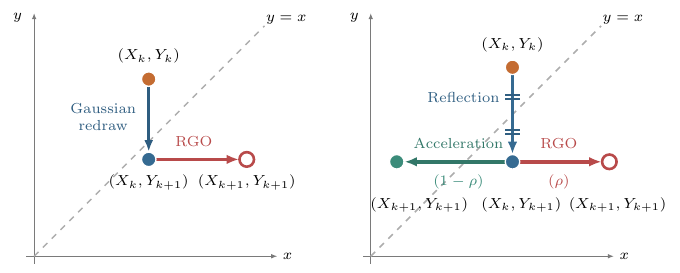}
\vspace{-1em}
\caption{
One outer iteration in the \((x,y)\)-plane.  \textbf{Left:} \textbf{Standard proximal sampler} (\cref{sec:preliminaries-ps}) independently redraws \(Y_{k+1}\sim\cN(X_k,\eta I_d)\), discarding the previous displacement, and then draws \(X_{k+1}\sim q_{Y_{k+1}}\) by RGO.  \textbf{Right:} \textbf{Reflected proximal sampler} (\cref{alg:reflected-proximal-overview}) sets \(Y_{k+1}=2X_k-Y_k\), so \(Y_k\) and \(Y_{k+1}\) are equidistant from \(X_k\), and then takes the accelerated conditional transition with probability \(1-\rho\) or an independent RGO refresh with probability \(\rho\). In the illustrated accelerated branch, \(X_{k+1}-Y_{k+1}\) retains the same sign as \(X_k-Y_k\), preserving directional information across the iteration; the independent RGO draw need not preserve this sign.
}
\label{fig:reflected-overview}
\end{figure}

\section{Flux-Resampled Underdamped Half-Turn}
\label{sec:flux}

We first instantiate the accelerated branch of \cref{alg:reflected-proximal-overview} without retaining momentum between outer iterations.  At the reflected value \(Y_{k+1}=y\), the update starts from \(X_k=x\), draws a Gaussian momentum, runs an event-driven underdamped half-turn, and returns its terminal position. At a high level, it freezes the nonlinear force at an independent conditional reference, follows the resulting exactly solvable affine motion, and uses random events to correct the residual force.

\subsection{The Flux-Resampled Update}

We present the ideal flux-resampled accelerated branch, \(X_{k+1}\sim F_{Y_{k+1}}(X_k,\cdot)\), in \cref{alg:flux-half-turn}. We choose a fixed numerical constant \(\gamma_0\in(0,1]\), and set
\(\gamma=\gamma_0\sqrt{\frac{\alpha\beta\eta^2}{1+\alpha\eta}}\),
\(\omega=\sqrt{1-\frac{\gamma^2}4}\), and \(T_\gamma=\frac{\pi}{\omega}\).  The condition \(\beta\eta\le1/2\) used below (\cref{prop:flux-main-reversal}) ensures \(\gamma<2\).
\((Z_t,P_t)\) denotes the within-update position--momentum trajectory.  At an \textit{event} time \(t\), we apply an instantaneous momentum update: \(Z_t\) remains continuous, while the momentum jumps from \(P_{t-}\) to \(P_t\).\footnote{By default, the process is always assumed to be c\`adl\`ag (right-continuous with left limits).}
All driving random sources used by one call are fresh: conditional on \((x,y)\), the reference draw, initial momentum, Brownian motion, driving Poisson clock, and event marks are mutually independent.
Thus, following the general framework (\cref{alg:reflected-proximal-overview}), the ideal flux chain makes one exact RGO draw \(q_y\) in every outer iteration.

\begin{algorithm}[htbp]
\refstepcounter{algorithm}
\vspace{0.35em}
\label{alg:flux-half-turn}
{\raggedright\textbf{Algorithm \thealgorithm: Ideal flux-resampled accelerated branch \(F_y\)}\par}
\vspace{0.3em}
\hrule
\vspace{0.35em}
\begin{algorithmic}[1]
\Require \(X_k=x\), \(Y_{k+1}=y\), and step size \(\eta\)
\State Draw \(\widetilde X\sim q_y\) independently of \(x\), draw
\(P_0\sim\gamma_d\), and set \(Z_0\gets x\)
\State Until \(t=T_\gamma\), evolve between events according to
\(\d Z_t=\sqrt\eta P_t\,\d t\) and
\(\d P_t=-\{(Z_t-y)/\sqrt\eta+\sqrt\eta\nabla V(\widetilde X)\}\,\d t
-\gamma P_t\,\d t+\sqrt{2\gamma}\,\d B_t\)
\State Set residual gradient \(h_t\gets\nabla V(Z_t)-\nabla V(\widetilde X)\) and generate
events at rate \(\lambda_t\gets\sqrt\eta[P_{t-}^\top h_t]_+\)
\State At each event (where $\lambda_t>0$ ensures $h_t\ne 0$), set $n_t\gets h_t/\norm{h_t}$, draw $R_t$ from the Rayleigh density $r\mapsto r\e^{-r^2/2}1_{r>0}$, and update
\(P_t\gets P_{t-}-(P_{t-}^\top n_t+R_t)n_t\)
\State \Return \(X_{k+1}\gets Z_{T_\gamma} \sim F_y(x,\cdot)\)
\end{algorithmic}
\vspace{0.1em}
\end{algorithm}
\subsection{Why a Half-Turn and Why Events?}
We now provide intuition for the design of \cref{alg:flux-half-turn}.
At fixed \(y\), we freeze the nonlinear force at the independent reference \(\widetilde X\).  The conditional force then splits as
\[
-\nabla_x\log q_y(x) =
 \tfrac{x-y}{\eta}+\nabla V(x)
 =\big\{\tfrac{x-y}{\eta}+\nabla V(\widetilde X)\big\}
  +\big\{\nabla V(x)-\nabla V(\widetilde X)\big\}.
\]
The frozen term is a harmonic force centered at \(x_{\mathrm{ctr}}=y-\eta\nabla V(\widetilde X)\), while the second term is the residual force.  Freezing makes the between-event SDE affine and exactly solvable.  Drawing \(\widetilde X\sim q_y\) keeps the residual small on average because the stationary position and the reference follow the same localized law.

First ignore the residual events. Then \(Z_t-x_{\mathrm{ctr}}\) follows a damped harmonic oscillator.  At its half-period \(T_\gamma=\pi/\omega\), the coefficient of the initial momentum \(P_0\) vanishes, while the centered input \(Z_0-x_{\mathrm{ctr}}\) is multiplied by \(-\e^{-\gamma T_\gamma/2}\).  Thus the affine motion approximately reflects the input across \(x_{\mathrm{ctr}}\), producing a coherent move rather than a local diffusive step.  The choice of \(\gamma\) balances the damping error with the residual-force error; \cref{flux:app:halfturn} makes this tradeoff precise.

The affine half-turn alone is not exact for \(q_y\), because freezing omits the residual force.  To correct it, define the residual potential \(V_{\mathrm{res}}(z):=V(z)-V(\widetilde X)-\nabla V(\widetilde X)^\top(z-\widetilde X)\).  Along the between-event trajectory,
\(\frac{\d}{\d t}V_{\mathrm{res}}(Z_t)=\sqrt\eta P_t^\top h_t\),
and therefore \(\lambda_t=[\frac{\d}{\d t}V_{\mathrm{res}}(Z_t)]_+\).
Thus an event can occur only while the affine trajectory climbs the omitted potential.
When the process is stationary, conditional on position, the component \(P_t^\top n_t\) is standard Gaussian.  The event rate weights each positive value in proportion to its size, so the incoming value \(P_{t-}^\top n_t>0\) observed at an event has density \(r\mapsto r\e^{-r^2/2}1_{r>0}\).  The update draws \(R_t\) from this Rayleigh law, sets the outgoing component to \(P_t^\top n_t=-R_t\), and leaves the orthogonal momentum unchanged. This rule balances the probability flux of forward and momentum-reversed paths \citep{vanetti2017piecewise,michel2020forward}.

Together, the event rate and Rayleigh update pair each fixed-reference path with a momentum-reversed path of equal stationary probability.  Averaging over the independent reference and discarding momentum therefore yields a reversible position kernel \(F_y\) for \(q_y\), as formalized next.

\begin{proposition}[Conditional invariance and reversibility]
\label{prop:flux-main-reversal}
Suppose \(\beta\eta\le1/2\).
Conditional on the reflected auxiliary value
\(Y_{k+1}=y\), if the input position satisfies \(X_k\mid Y_{k+1}=y\sim q_y\), then the output satisfies \(X_{k+1}\mid Y_{k+1}=y\sim q_y\)  as well.
Moreover, \(F_y\) satisfies detailed balance with respect to \(q_y\), i.e.,
\(q_y(x)F_y(x,x')=q_y(x')F_y(x',x)\).
Consequently, \cref{alg:reflected-proximal-overview} with accelerated branch \(F_y\) (\cref{alg:flux-half-turn}) preserves \(\Pi_\eta\), whose position marginal is \(\pi\).
\end{proposition}

The proof in \cref{flux:app:validity} establishes non-explosion and sums the preceding Rayleigh flux balance over all possible event times, yielding the complete path-reversal argument.

\subsection{From the Ideal Kernel to an Implementable Update}
\label{sec:flux-implementation-main}

\cref{alg:flux-half-turn} is ideal in two places: it draws \(\widetilde X\) exactly from \(q_y\), and it uses the state-dependent event rate \(\lambda_t\).  The implementable update replaces these operations by a certified approximate RGO and capped Poisson thinning, with a finite work limit for every auxiliary optimization.

\paragraph{Certified approximate RGO.}
Gradient descent on the conditional objective \(z\mapsto V(z)+\norm{z-y}^2/(2\eta)\) starts from the incoming position and stops when it finds a point \(\widehat x\) satisfying \(\norm{y-\eta\nabla V(\widehat x)-\widehat x}\le\sqrt{d\eta}\).  This is an observable approximate-proximal certificate, not yet a conditional draw; it supplies the center required by the first-order rejection sampler that produces an approximate sample from \(q_y\).  The solver is stopped after a deterministic query budget; if it has not certified the residual, the algorithm terminates and returns the current position.  Otherwise, the approximate draw is used either as the RGO-branch output or as the accelerated branch's reference \(\widetilde X\).

\paragraph{Capped event clock.}
Once \(\nabla V(\widetilde X)\) is stored, the affine diffusion between candidate events can be sampled exactly without further oracle queries.  Candidate times are proposed at rate \(\bar\lambda\); at a candidate time \(t\), one query evaluates \(\lambda_t=\sqrt\eta[P_t^\top\{\nabla V(Z_t)-\nabla V(\widetilde X)\}]_+\), and the proposal is accepted with probability \(\min\{1,\lambda_t/\bar\lambda\}\).  This exactly simulates the truncated rate \(\min\{\lambda_t,\bar\lambda\}\), and the resulting bias is the probability that the truncated and ideal clocks disagree (\cref{sec:preliminaries-poisson-process}).

\paragraph{Coupling to the ideal chain.}
The implementable and ideal iterations use common randomness until one of three disagreements occurs: the center solver reaches its cap, the approximate and exact conditional draws differ, or the ideal event rate exceeds \(\bar\lambda\).  The appendix assigns each source a per-iteration error proportional to the target accuracy divided by the total number of iterations, couples the two chains until their first disagreement, and applies a union bound over the run.  On the agreement event, the implementable chain follows the ideal chain exactly; \cref{flux:app:implementation} gives the caps, tolerances, and complete query-complexity analysis.

\subsection{Convergence and Oracle Complexity}

We now establish the convergence of the ideal and implementable chains.
Let \(K_{\eta,\rho}\) be the transition kernel of one complete reflected iteration with the ideal accelerated branch \(F_y\).  In other words, if \((X_0,Y_0)\sim\nu\), then \((X_m,Y_m)\sim\nu K_{\eta,\rho}^m\).
Let \(\rho_\star\in(0,1)\) be the
fixed universal numerical constant furnished by the theorem below, and set
\begin{equation}
 \rho=\rho_\star
 \sqrt{\alpha\eta\log\big(\tfrac{\e}{\alpha\eta}\big)}\le\rho_\star<1.
 \label{flux:eq:rho}
\end{equation}

\begin{theorem}[Ideal flux contraction]
\label{flux:thm:ideal}
There exist universal constants \(c_0,C>0\) and a fixed
\(\rho_\star\in(0,1)\), independent of all problem parameters, such that, if
\(\beta\eta\le c_0\), then every integer
\(n_{\mathrm{blk}}\ge \frac{C}{\sqrt{\alpha\eta}}\sqrt{\log\frac{\e}{\alpha\eta}}\) satisfies, for every
\(\nu\ll\Pi_\eta\) and every nonnegative integer \(J\),
\begin{equation}
 \chi^2(\nu K_{\eta,\rho}^{J n_{\mathrm{blk}}}\Vert\Pi_\eta)
 \le \e^{-J}\chi^2(\nu\Vert\Pi_\eta).
 \label{flux:eq:ideal-contraction}
\end{equation}
\end{theorem}

The proof first establishes in \cref{flux:app:halfturn} that one conditional half-turn nearly reverses centered affine position coordinates, with an error weighted by the same conditional covariance deficit through which reflection transfers information.  At the outer-chain level, reflection moves variation in \(Y\) into the conditional position, the half-turn reverses that component, and the RGO branch damps the remainder.  A modified \(L^2\) energy then contracts over blocks of \(n_{\mathrm{blk}}=\widetilde O(1/\sqrt{\alpha\eta})\) iterations; the complete argument is in \cref{flux:app:contraction}.

\begin{theorem}[Implementable flux half-turn]
\label{flux:thm:complexity}
Under \cref{ass:target}, let \(0<\eps<1/4\), \(\Delta\ge1\), and let
\(\nu_0\) be an initial position law satisfying
\(D_2(\nu_0\Vert\pi)\le\Delta\).  There exist universal numerical constants
\(c,C>0\), independent of \(d,\kappa,\Delta,\eps\), such that the following
choice is valid:
\begin{equation}
 L=\Delta+\log\tfrac{Cd\kappa}{\eps},\qquad
 \eta=\tfrac{c}{\beta(\sqrt{dL}+L)}.
 \label{flux:eq:eta-final}
\end{equation}
Initialize \(X_0\sim\nu_0\) and
\(Y_0\mid X_0\sim\cN(X_0,\eta I_d)\).  Then run
\cref{alg:reflected-proximal-overview} with the accelerated branch in
\cref{alg:flux-half-turn}, using the approximate RGO and capped Poisson clock
specified in \cref{flux:app:implementation}.  Its output distribution
\(\widehat\nu\) satisfies \(\TV(\widehat\nu,\pi)\le\eps\).  Writing
\(N_{\mathrm{qry}}\) for its total number of oracle queries,
\begin{equation}
 \E N_{\mathrm{qry}}
 \le C\sqrt\kappa\,(d+L)^{1/4}L^{3/4}
       \left(\Delta+\log\tfrac4\eps\right)\log(\e\kappa L).
 \label{flux:eq:complexity}
\end{equation}
In particular, a constant-\(D_2\) warm start requires
\(\widetilde O(\sqrt\kappa d^{1/4})\) expected oracle queries.
\end{theorem}

The step size in \cref{flux:eq:eta-final} simultaneously controls the approximate RGO and the capped event clock.  Summing the center-solver, conditional-draw, and event-proposal costs over the ideal mixing horizon gives \cref{flux:eq:complexity}; the complete error allocation and query-complexity analysis are in \cref{flux:app:implementation}.

\section{Persistent-Momentum Quarter-Turn}
\label{sec:persistent}

The flux method uses momentum only inside one accelerated update and then discards it.  The persistent construction instead augments the outer state to \((X,Y,P)\), carries the terminal momentum into the next iteration, and requires its accelerated branch to preserve \(q_y\otimes\gamma_d\). As in the flux method, we first analyze this branch with an exact conditional reference, then replace that reference by an implementable Gaussian construction that reserves approximate RGO calls for the refresh branch.

\subsection{The Persistent Update}

At fixed \(Y=y\), let \(\zeta_y\) be a reference law depending only on \(y\).  For an incoming position--momentum pair \((x,p)\), \cref{alg:persistent-quarter-turn} defines the phase-space kernel \(G_y^\zeta((x,p),\cdot)\).  The ideal contraction analysis takes \(\zeta_y=q_y\); the implementable reference is constructed in \cref{sec:persistent-rgo-light-main}.
As in \cref{alg:flux-half-turn}, an \textit{event} is an instantaneous momentum update \(P_{t-}\mapsto P_t\).  Conditional on \((x,p,y)\), the reference draw, Poisson clock, and event marks are fresh and mutually independent.
\begin{algorithm}[htbp]
\refstepcounter{algorithm}
\vspace{0.2em}
\label{alg:persistent-quarter-turn}
{\raggedright\textbf{Algorithm \thealgorithm: Ideal persistent-momentum
accelerated branch \(G_y^\zeta\)}\par}
\hrule
\vspace{0.1em}
\begin{algorithmic}[1]
\Require Incoming pair \((x,p)\), fixed auxiliary value \(y\), step size
\(\eta\), and reference law \(\zeta_y\)
  \State Draw \(\widetilde X\sim\zeta_y\) independently of \((x,p)\), and set
  \((Z_0,P_0)\gets(x,p)\)
  \State Until \(t=\pi/2\), evolve between events according to
  \(\dot Z_t=\sqrt\eta P_t\) and
  \(\dot P_t=-(Z_t-y)/\sqrt\eta
                 -\sqrt\eta\nabla V(\widetilde X)\)
  \State Set residual gradient \(h_t\gets\nabla V(Z_t)-\nabla V(\widetilde X)\) and generate events at rate \(\lambda_t\gets\sqrt\eta[P_{t-}^{\top}h_t]_+\)
  \State At each event (where $\lambda_t>0$ ensures $h_t\ne 0$), set $n_t\gets h_t/\norm{h_t}$ and update
  \(P_t\gets P_{t-}-2(P_{t-}^{\top}n_t)n_t\)
\State \Return \((X^+,P^+)\gets (Z_{\pi/2},P_{\pi/2})\sim G_y^\zeta((x,p),\cdot)\)
\end{algorithmic}
\vspace{0.1em}
\end{algorithm}

A complete ideal persistent iteration first reflects \(Y_{k+1}=2X_k-Y_k\).  It then redraws both position and momentum on the RGO branch, or applies the persistent phase-space kernel on the accelerated branch:
\[
 (X_{k+1},P_{k+1})\sim
 \begin{cases}
  q_{Y_{k+1}}\otimes\gamma_d,&\text{with probability }\rho,\\
  G_{Y_{k+1}}^\zeta((X_k,P_k),\cdot),&\text{otherwise}.
 \end{cases}
\]

\subsection{Why a Quarter-Turn?}
The persistent method must transfer position information into momentum rather than discard it.  Fix the reference \(\widetilde X\), set \(x_{\mathrm{ctr}}=y-\eta\nabla V(\widetilde X)\), and let \(W_t=(Z_t-x_{\mathrm{ctr}})/\sqrt\eta\).  Without events, \(\dot W_t=P_t\) and \(\dot P_t=-W_t\), so the motion is a rotation.
After one quarter-period, \((W_0,P_0)\mapsto(P_0,-W_0)\), so centered position information is stored in the terminal momentum and can return to position in later iterations.  A half-turn instead maps \((W_0,P_0)\) to \((-W_0,-P_0)\), producing a sign flip without exchanging the two coordinates.  Events correct the residual force through specular momentum reflection, which preserves momentum magnitude and is therefore compatible with persistence.  The RGO branch supplies the occasional refreshment described in \cref{sec:framework}.

\begin{proposition}[Conditional invariance]
\label{prop:persistent-main-exactness}
For every fixed \(y\) and reference value \(\widetilde x\), the
accelerated-branch motion preserves \(q_y\otimes\gamma_d\).  Consequently,
\(G_y^\zeta\) preserves \(q_y\otimes\gamma_d\) whenever
\(\widetilde X\sim\zeta_y\) is independent of the incoming position and
momentum.  The complete persistent iteration above therefore preserves
\(\widehat\Pi_\eta:=\Pi_\eta\otimes\gamma_d\).
\end{proposition}

The proof in \cref{persistent:sec:exactness} combines harmonic energy conservation with the residual event rate and specular reflection.  It establishes non-explosion and pairs each stationary path with its momentum-reversed path run backward, which yields the conditional invariance stated above.

\subsection{From the Ideal Kernel to an Implementable Update}
\label{sec:persistent-rgo-light-main}

The ideal contraction analysis draws the accelerated-branch reference exactly from $q_y$, but the implementable update avoids this per-iteration RGO.  Initialized at $y$ (rather than the incoming position, so that the result depends only on $y$), the conditional mode solve in \cref{sec:flux-implementation-main} returns a point $\widehat x_y$ satisfying $\norm{y-\eta\nabla V(\widehat x_y)-\widehat x_y}\le\sqrt{d\eta}$.  The accelerated branch then draws $G\sim\gamma_d$ and uses the shifted Gaussian reference $\widetilde X=y-\eta\nabla V(\widehat x_y)+\sqrt\eta G$.  Because this reference depends only on $y$ and fresh randomness, \cref{prop:persistent-main-exactness} still gives exact conditional invariance; its approximation to $q_y$ is needed only for the contraction estimate.

The mode solve and event clock use the same capping principle as the flux method, while an approximate RGO is invoked only on a refresh iteration, which occurs with probability \(\rho\).  The appendix couples this implementable chain to an uncapped shifted-reference comparison chain until a mode cap, clock cap, or approximate refresh causes disagreement, and gives the complete error and query-complexity analysis in \cref{persistent:sec:implementation}.

\subsection{Convergence and Oracle Complexity}

We now prove the convergence of the ideal and implementable chains.
Let \(K_{\mathrm{per}}\) be the transition kernel of one complete ideal persistent iteration with \(\zeta_y=q_y\).  In other words, if \((X_0,Y_0,P_0)\sim\nu\), then \((X_m,Y_m,P_m)\sim\nu K_{\mathrm{per}}^m\).  Set \(g=\sqrt{\alpha\eta}\) and \(\ell_g=1+\log(1/g)\).

\begin{theorem}[Ideal persistent contraction]
\label{thm:persistent-contraction}
There exist universal constants \(c_0,C>0\) and a fixed \(c_\rho>0\), independent of all problem parameters, such that the following holds.  If \(0<\eta\le c_0/\beta\), set \(\rho=c_\rho g\sqrt{\ell_g}\) and suppose \(\rho\le1/4\).  Then every integer \(n_{\mathrm{blk}}\ge C\sqrt{\ell_g}/g\) satisfies, for every \(\nu\ll\widehat\Pi_\eta=\Pi_\eta\otimes\gamma_d\) and every nonnegative integer \(J\),
\begin{equation}
 \chi^2(\nu K_{\mathrm{per}}^{J n_{\mathrm{blk}}}
 \|\widehat\Pi_\eta)
 \le\e^{-J}\chi^2(\nu\|\widehat\Pi_\eta).
 \label{eq:persistent-ideal-contraction}
\end{equation}
Consequently, initialize \(X_0\sim\nu_0\), \(Y_0\mid X_0\sim\cN(X_0,\eta I_d)\), and independently \(P_0\sim\gamma_d\). Suppose \(D_2(\nu_0\Vert\pi)\le\Delta\). For \(\eps\le1/2\), one obtains \(\TV(\Law(X_n),\pi)\le\eps\) if
\begin{equation}
 n\ge\tfrac{C{\ell_g}^{1/2}}g
       \left(\Delta+\log\tfrac4\eps\right).
 \label{eq:persistent-iterations}
\end{equation}
\end{theorem}

Since \(g=\sqrt{\alpha\eta}\), one constant-factor reduction of \(\chi^2\) divergence takes \(\widetilde O(1/\sqrt{\alpha\eta})\) iterations, and repeating this contraction from the warm start gives \cref{eq:persistent-iterations}.  The proof tracks how an observable moves among the auxiliary variable \(Y\), conditional position, and momentum: reflection creates conditional position dependence, the quarter-turn rotates it into momentum and back, and the RGO branch removes the remaining conditional dependence.  The appendix first proves the covariance-weighted approximate quarter-rotation in \cref{persistent:lem:joint-quarter}, then converts this three-coordinate transfer mechanism into the \(L^2\) operator estimate in \cref{persistent:prop:ideal-contraction}; applying that estimate to density ratios yields \cref{eq:persistent-ideal-contraction}.

\begin{theorem}[Implementable persistent sampler]
\label{thm:persistent-finite}
Under \cref{ass:target}, let \(\nu_0\) satisfy
\(D_2(\nu_0\Vert\pi)\le\Delta\), and let
\(\eps\le1/2\).  There exist fixed universal constants \(c,C>0\) so that the following holds.
Define
\(\Lambda_\eps=\Delta+\log(Cd\kappa/\eps)\) and
\(\eta=c/\{\beta(\sqrt{d\Lambda_\eps}+\Lambda_\eps)\}\).
Initialize \(X_0\sim\nu_0\),
\(Y_0\mid X_0\sim\cN(X_0,\eta I_d)\), and independently \(P_0\sim\gamma_d\).  Run the implementable algorithm specified in
\cref{persistent:sec:implementation}; its output law \(\widehat\nu\) satisfies
\(\TV(\widehat\nu,\pi)\le\eps\).  If
\(N_{\mathrm{qry}}\) counts all oracle queries and \(N_{\mathrm{rgo}}\)
counts approximate-RGO calls, then
\begin{align}
 \E N_{\mathrm{qry}}
 &\le\widetilde O(
 \sqrt\kappa\,(\Delta+\log(8/\eps))
 (\sqrt{d\Lambda_\eps}+\Lambda_\eps)^{1/2}),
 \label{eq:persistent-general-complexity}\\
 \E N_{\mathrm{rgo}}
 &\le\widetilde O(\Delta+\log(8/\eps)).
 \nonumber
\end{align}
In particular, a constant-\(D_2\) warm start gives
\(\widetilde O(\sqrt\kappa d^{1/4})\) expected oracle queries and only
\(\operatorname{polylog}(d,\kappa,1/\eps)\) expected RGO calls.
\end{theorem}

The ideal horizon is \(\widetilde O(g^{-1}\{\Delta+\log(1/\eps)\})\), while refreshes occur with probability \(\rho=\widetilde O(g)\); their product gives the polylogarithmic RGO count.  Every iteration pays for the capped mode solve; accelerated iterations then use the shifted Gaussian reference and capped event clock, whereas refresh iterations invoke the approximate RGO.  At the step size in the theorem these costs are \(\widetilde O(1)\) per iteration, and summing them over the ideal horizon gives the stated total query bound.  The full error and query-complexity analysis is in \cref{persistent:sec:implementation}; an optional quadratic-baseline refinement is stated and proved in \cref{cor:persistent-near-gaussian}.

\section{Removing the Supplied Warm Start}
\label{sec:warm-start-removal}

The rates in \cref{sec:flux,sec:persistent} assume a constant-\(D_2\) warm start. Our cold-start pipeline has three stages: localization constructs a law with finite \(D_2\) divergence, a coarse proximal chain reduces this divergence to a constant, and either reflected sampler performs the final high-accuracy sampling.  When a finite bound \(D_2(\nu_0\Vert\pi)\le B\) is available, we skip localization and choose between direct reflected sampling and the coarse warm-up route. The resulting bounds are end-to-end: all stages used are included in the query count.
Warm-start removal preserves the high-accuracy character of the reflected final stage: the overall dependence on \(1/\varepsilon\) remains polylogarithmic.

Write the coarse step size as \(u=r/\beta\), where \(r\in[1,\kappa]\) interpolates between the step sizes \(1/\beta\) and \(1/\alpha\), and let the \emph{ideal coarse proximal chain} be the Gibbs sampler in \cref{eq:rgo} with step size \(u\).
The R\'enyi contraction of \citet[Thm. 3]{chen2022improved} makes the ideal coarse chain constant-\(D_2\)-warm after \(O((\kappa/r)\log(\e B))\) iterations.
The RGO has condition number \((\beta+\nicefrac{1}{u})/(\alpha+\nicefrac{1}{u})\le1+r\).  Increasing \(r\) therefore shortens the ideal coarse chain but makes each conditional draw harder.  The algorithm replaces those exact draws by the cheaper of two implementable posterior samplers: a regular interface based on smoothed Picard HMC \citep{chen2026smoothed} followed by reflected finalization, or a harmonic-base interface based on exact diffusion \citep{chen2026exact,chen2026high}.
For fixed \(r\), multiplying the number of coarse iterations by the
cost of the regular conditional interface gives the cost of the regular-interface route:
\[
 \widetilde O(
 \kappa\{r^{-1}+r^{1/6}d^{1/6}+r^{-1/2}d^{1/4}\}
 \log(\e B)).
\]
Optimizing over \(r\in[1,\kappa]\) and considering the relative sizes of \(\kappa\) and \(d\) yields the best polynomial cost factor below, which we denote by \(\Psi(\kappa,d)\):
\begin{equation*}
 \Psi(\kappa,d)\asymp
 \begin{cases}
  \kappa^{2/3}d^{1/3},&\text{when $d\lesssim\kappa^{16/7}$, with $r\asymp\kappa$},\\
  \kappa d^{3/16},&\text{when $\kappa^{16/7}\lesssim d\lesssim\kappa^8$, with $r\asymp d^{1/8}$},\\
  \sqrt\kappa d^{1/4},&\text{when $d\gtrsim\kappa^8$, with $r\asymp\kappa$}.
 \end{cases}
\end{equation*}
Thus the optimized coarse route costs \(\widetilde O(\Psi(\kappa,d)\log(\e B))\).  This bound includes the implementable warm-up and the final
reflected sampler, whose lower-order cost is absorbed in every regime.  It
does not include the one-time localization cost needed when the input is only
a point.

\begin{theorem}[Warm-start removal]
\label{thm:warm-removal-distributional}
\label{thm:warm-removal-cold}
\emph{(i) Finite-divergence input.}
Suppose a sample from an initial law \(\nu_0\) is available and the known
number \(B\ge1\) satisfies \(D_2(\nu_0\Vert\pi)\le B\).  For every
\(0<\eps<1/4\), the wrapper using either reflected sampler as its final stage
returns an output law \(\widehat\nu\) with
\(\TV(\widehat\nu,\pi)\le\eps\).  If \(N_{\mathrm{qry}}\) is the total number
of oracle queries, then
\begin{equation}
 \E N_{\mathrm{qry}}
 \le\widetilde O(
 \min\{
  \sqrt\kappa B(\sqrt{dB}+B)^{1/2},
  \Psi(\kappa,d)\log(\e B)
 \}).
 \label{eq:warm-removal-distributional-cost}
\end{equation}
If $B$ is polynomially bounded in $d,\kappa,1/\eps$, this is at most
$\widetilde O(\kappa^{6/7}d^{1/4})$, since $\Psi(\kappa,d)\lesssim\kappa^{6/7}d^{1/4}$ across all three regimes (attained at the crossover $d\asymp\kappa^{16/7}$).

\smallskip
\emph{(ii) Supplied point.}
Given \(x_0\in\R^d\) and \(0<\eps<1/4\), let \(x_*\) be the minimizer of \(V\), define
\(\log_+(t)=\max\{1,\log t\}\), and set
\(Q_{\mathrm{loc}}=\sqrt\kappa\log_+[\kappa\{2+\sqrt\beta
\norm{x_0-x_*}\}]\).  An implementable algorithm returns an output law
\(\widehat\nu\) with \(\TV(\widehat\nu,\pi)\le\eps\) and
\begin{equation}
 \E N_{\mathrm{qry}}
 \le\widetilde O(Q_{\mathrm{loc}}+\Psi(\kappa,d))
 \le\widetilde O(
 Q_{\mathrm{loc}}+\kappa^{6/7}d^{1/4}).
 \label{eq:warm-removal-cold-cost}
\end{equation}
\end{theorem}

The first term in \cref{eq:warm-removal-distributional-cost} skips the
coarse stage and applies a reflected sampler directly with the full divergence
budget \(B\).  The second term uses the coarse route just described: warm the
input using step size \(u=r/\beta\), then apply a constant-warm reflected finalizer.
The ideal coarse chain is used only for analysis and is not executed.  Its endpoint is constant-\(D_2\)-warm, whereas the implementable warm-up replaces each exact conditional draw by an implementable posterior sampler.  The two endpoints are close in total variation.  Applying the same capped reflected finalizer to both cannot increase this discrepancy, so the convergence guarantee for the ideal coarse endpoint transfers to the algorithmic output. The wrapper is modular: it requires of the final stage only a TV guarantee from a constant-\(D_2\) warm start at cost \(\widetilde O(\sqrt\kappa d^{1/4})\), so the improvement comes from optimizing the coarse warm-up scale rather than from the final dynamics. Compared with the bound \(\widetilde O(\kappa^{7/6}d^{1/6}+\sqrt\kappa d^{1/4})\) for smoothed Picard HMC warm-up followed by PBPS \citep{chen2026smoothed,chen2026accelerated}, \(\Psi(\kappa,d)\) is never larger and is polynomially smaller whenever \(d\lesssim\kappa^8\).

For a supplied point, accelerated optimization stops at the observable
certificate \(\norm{\nabla V(x_{\mathrm{ref}})}\le\alpha/\sqrt\beta\).
One draw from \(\cN(x_{\mathrm{ref}},(2\beta)^{-1}I_d)\) then has
\(D_2\le1+\tfrac d2\log(2\kappa)\), so the finite-divergence result applies.
The algorithm does not use \(x_*\); it appears only in the localization
bound.  The regime optimization, error handoff, and localization proof are
given in \cref{app:warm-start-removal}.

\section{Extension under a Poincar\'e Inequality}
\label{sec:pi-extension}

We now replace \cref{ass:target} by a Poincar\'e inequality (PI), which is a weaker assumption than strong log-concavity.

\begin{assumption}[Poincar\'e inequality and semiconvexity]
\label{ass:pi-extension}
The probability density \(\pi\propto \e^{-V}\) has \(V\in C^2(\R^d)\), and there are \(m>0\) and \(0\le K\le\beta\), such that, for every smooth \(f\),
\begin{equation}\label{eq:pi}
 \Var_\pi(f)\le \tfrac1m\E_\pi\norm{\nabla f}^2,
 \qquad -KI\preceq\nabla^2V(x)\preceq\beta I
 \quad(\text{for all }x\in\R^d).
\end{equation}
\end{assumption}

Since $\alpha$-strongly convex potentials satisfy PI with gap \(\alpha\) \citep{bakry2014analysis}, \(m\) plays the role of \(\alpha\) and \(\beta/m\) plays the role of the condition number $\kappa$. \cref{ass:target} is the special case \(K=0\) and \(m=\alpha\). We remark that \cref{eq:pi} implies \(m\le\beta\), which is proved in \cref{lem:pi-gap-smoothness}.

\begin{theorem}[Flux and persistent sampling under a Poincar\'e inequality]
\label{thm:pi-extension}
Under \cref{ass:pi-extension}, suppose \(D_2(\nu_0\Vert\pi)\le\Delta\) for \(\Delta\ge0\), and let \(\eps\le1/2\). For suitable universal constants \(c,C>0\), set
\begin{equation}
 \Lambda_{\mathrm{PI}}=1+\Delta+
 \log\tfrac{Cd(1+\beta/m)(1+K/m)}{\eps},
 \qquad
 \eta=\tfrac{c}{\beta(\sqrt{d\Lambda_{\mathrm{PI}}}+\Lambda_{\mathrm{PI}})}.
 \label{eq:pi-step-size}
\end{equation}
Initialize \(X_0\sim\nu_0\) and \(Y_0\mid X_0\sim\cN(X_0,\eta I_d)\). For each sampler below, the output law \(\widehat\nu\) satisfies \(\TV(\widehat\nu,\pi)\le\eps\), and the expected oracle-query count satisfies
\begin{equation}
 \E N_{\mathrm{qry}}
 \le\widetilde O\big(
 \sqrt{\tfrac{\beta}{m}}\sqrt{1+\tfrac Km}\,
 (\sqrt{d\Lambda_{\mathrm{PI}}}+\Lambda_{\mathrm{PI}})^{1/2}
 \big(1+\Delta+\log\tfrac1\eps\big)\big).
 \label{eq:pi-query-bound}
\end{equation}
For a constant-\(D_2\) warm start, both samplers therefore use \(\widetilde O(\sqrt{\beta/m}\sqrt{1+K/m}\,d^{1/4})\) expected oracle queries.

\emph{(i) Flux-resampled half-turn.} The implementation in \cref{app:flux-pi} uses one approximate RGO per completed iteration. Its RGO count is bounded by the full iteration horizon, whose order is given by the right-hand side of \cref{eq:pi-query-bound}.

\emph{(ii) Persistent-momentum quarter-turn.} Draw \(P_0\sim\gamma_d\) independently of \((X_0,Y_0)\). The implementation in \cref{app:pi-extension} uses an approximate RGO only on refresh steps, with
\begin{equation}
 \E N_{\mathrm{rgo}}
 \le\widetilde O\big(
 \big(1+\tfrac Km\big)
 \big(1+\Delta+\log\tfrac1\eps\big)\big).
 \label{eq:pi-rgo-bound}
\end{equation}
Thus a constant-\(D_2\) warm start requires \(\widetilde O(1+\frac Km)\) expected RGO calls for the persistent sampler.
\end{theorem}

At \(K=0\) and \(m=\alpha\), \cref{thm:pi-extension} reduces to \cref{flux:thm:complexity,thm:persistent-finite}, with the same complexity bounds. Acceleration survives for \(K\lesssim m\), where the query bound is \(\widetilde O(\sqrt{\beta/m}\,d^{1/4})\), and degrades to \(\widetilde O(\sqrt{\beta K}m^{-1}d^{1/4})\) when \(K\gg m\). The latter is the price of negative curvature: at \(K=\beta\) the bound is \(\widetilde O((\beta/m)d^{1/4})\), while \citet{zhou2026towards} show that \(\Omega(\beta/m)\) queries are necessary from a warm start under the stronger log-Sobolev inequality for \(d=\Omega(\log^2(\beta/m))\) (see \cref{sec:related-work} for further discussion). We remark that both guarantees need a supplied warm start and do not extend the initialization result of \cref{sec:warm-start-removal}, which we leave for future work.

The proofs keep the invariance arguments and change three ingredients of the macroscopic estimates.  Because the conditional covariance may now exceed the Gaussian one, Brascamp--Lieb gives way to a signed bound on $D_y=I-\Cov(W\mid y)$ for the standardized conditional displacement $W=(X-y)/\sqrt\eta$ under $q_y$ (\cref{lem:pi-signed-readable}); the reflection estimate is driven by the spectral gap instead of strong convexity, giving $g^2\asymp m\eta$ (\cref{lem:pi-outer-readable}); and negative curvature inflates the factorization constant to $A\asymp\ell_g(1+K/m)$ with $\ell_g=1+\log(1/g)$, which is the source of the $K/m$ factors in \cref{eq:pi-query-bound,eq:pi-rgo-bound}. \cref{app:pi-extension,app:flux-pi} give the details.

\section{Conclusion}
\label{sec:conclusion}

Reflected proximal dynamics accelerate the outer chain by combining exact Gaussian reflection, conditional kinetic transport, and controlled refreshment.  From a constant-\(\chi^2\) warm start, both realizations use \(\widetilde O(\sqrt\kappa d^{1/4})\) expected oracle queries, while the persistent construction makes only a polylogarithmic expected number of approximate-RGO calls.  Including localization and warm-up yields an end-to-end expected oracle complexity of \(\widetilde O(\kappa^{6/7}d^{1/4})\). Both the warm-start and end-to-end guarantees have only polylogarithmic dependence on \(1/\varepsilon\), arising from an ideal mixing horizon of order \(\log(1/\varepsilon)\). Strong log-concavity is not essential to the mechanism: under a PI with gap \(m\) and a two-sided curvature bound, the same two samplers retain the accelerated warm-start rate \(\widetilde O(\sqrt{\beta/m}\,d^{1/4})\) whenever the negative curvature is \(O(m)\).
Two directions remain open. Our PI guarantees assume a supplied warm start, and removing it would require a coarse stage that does not rest on R\'enyi contraction under strong convexity. The end-to-end exponent is also set by initialization rather than by the sampler it initializes: the warm-up costs \(\widetilde O(\kappa^{6/7}d^{1/4})\) while the reflected stage alone costs \(\widetilde O(\sqrt\kappa d^{1/4})\), and whether the latter is attainable from a cold start remains open.

\subsection*{AI use statement}

The idea of accelerating the proximal sampler using underdamped dynamics originated with the authors.  We used generative AI tools to brainstorm implementation details and algorithmic variants, and assist in drafting and revising mathematical proofs. AI-generated suggestions were treated as unverified drafts rather than as authoritative sources.  The authors reviewed and revised all AI-assisted material, checked the final theorem statements and proofs, and verified the citations used in the paper against the referenced works.
We take responsibility for the final content of this work, including all text, claims, and artifacts produced with the aid of generative AI.

\bibliography{iclr2027_conference}
\bibliographystyle{iclr2027_conference}

\clearpage
\appendix

\crefalias{section}{appendix}
\crefalias{subsection}{appendix}
\crefalias{subsubsection}{appendix}
\clearpage
\appendixcontents
\clearpage

\section{Related Work and Detailed Comparison}
\label{sec:related-work}

\paragraph{Proximal and kinetic samplers.}
Earlier proximal Langevin methods use Moreau regularization or proximal
mappings to discretize overdamped Langevin dynamics
\citep{durmus2018efficient,salim2020primal}.  The Gaussian proximal sampler
instead separates global contraction from sampling a localized conditional
\citep{lee2021structured,chen2022improved}.
Approximate-RGO algorithms improve the local implementation
\citep{fan2023improved,chen2026high}; our methods instead
modify the outer dynamics to retain information across conditional
updates.  Kinetic Langevin samplers obtain acceleration by retaining
velocity, with specialized discretizations needed to control the resulting
finite-step error
\citep{cheng2018underdamped,dalalyan2020sampling,ma2021is,shen2019randomized,zhang2023improved}.
We use the same principle inside each localized conditional problem, while
an exact Gaussian reflection transports memory between successive
conditionals.

\paragraph{Event-driven kinetic updates.}
The bouncy particle sampler corrects deterministic motion through velocity
events triggered by the positive directional derivative of the potential
\citep{bouchardcote2018bouncy}.  Event-driven Hamiltonian samplers and the
Boomerang sampler instead integrate an exactly solvable reference motion and
use such events to correct the residual force
\citep{vanetti2017piecewise,bierkens2020boomerang}.  Gaussian
flux-resampling of the force-aligned velocity component appears in the
framework of \citet{vanetti2017piecewise} and Forward Event-Chain Monte Carlo
\citep{michel2020forward}; related samplers combine Langevin motion with
randomized velocity events
\citep{monmarche2022exact,gouraud2025velocityjump}.  Our flux half-turn places
these ingredients inside a reflected proximal iteration and selects the
duration and damping to obtain the weighted reversal estimate used by the
outer-chain analysis.

\paragraph{End-to-end complexity.}
\cref{tab:proximal-comparison} compares representative high-accuracy
results under smooth strong log-concavity.  We report the error metric in
which each paper states its guarantee and the initialization that its bound
requires.  The exact-ULD/FORS row is not a proximal method; we include it as
the closest global benchmark with a different tradeoff between \(d\) and
\(\kappa\).

\begin{table}[t]
\centering
\caption{Representative high-accuracy bounds for smooth, strongly
log-concave sampling.  The table reports oracle-query counts in expectation
unless marked otherwise.  The notation \(\widetilde O\) suppresses
polylogarithmic factors in \(d,\kappa,1/\eps\).}
\label{tab:proximal-comparison}

\setlength{\tabcolsep}{3pt}
\renewcommand{\arraystretch}{1.18}
\begin{tabularx}{\textwidth}{@{}
  >{\raggedright\arraybackslash}p{0.17\textwidth}
  >{\raggedright\arraybackslash}p{0.25\textwidth}
  >{\raggedright\arraybackslash}p{0.18\textwidth}
  >{\raggedright\arraybackslash}p{0.13\textwidth}
  >{\raggedleft\arraybackslash}X@{}}
\toprule
Paper & Method & Initialization & Metric & Complexity \\
\midrule
\citet{chen2022improved}
& Gaussian proximal sampler with rejection RGO
& polynomial KL/R\'enyi divergence
& KL, \(\chi^2\), R\'enyi
& \(\widetilde O(\kappa d)\) \\
\midrule

\citet{fan2023improved}
& approximate-rejection proximal sampler
& Gaussian cold start
& TV, \(W_2\), \(\chi^2\)
& \(\widetilde O(\kappa d^{1/2})\) \\
\midrule

\citet{altschuler2024faster}
& ULMC warm-up with MALA/proximal reduction
& known or located minimizer; no warm law
& TV, KL, \(\chi^2\), \(W_2\)
& \(\widetilde O(\kappa d^{1/2})\) w.h.p. \\
\midrule

\citet{mitra2025fast}
& proximal sampler with a basic rejection compiler
& polynomial \(\Phi\)-divergence
& \(\Phi\)-divergence
& \(\widetilde O(\kappa d)\) \\
\midrule

\citet{chen2026high}
& proximal sampler with a FORS conditional sampler
& finite-R\'enyi Gaussian start
& fixed-order R\'enyi
& \(\widetilde O(\kappa d^{1/2})\) \\
\midrule

\citet{chen2026exact}
& exact ULD/FORS
& Gaussian with \(D_{2q}=O(d\log\kappa)\), fixed \(q\)
& fixed-order R\'enyi
& \(\widetilde O(\kappa^{2/3}d^{1/3})\) \\
\midrule

\citet{chen2026smoothed,chen2026accelerated}
& SPHMC warm-up with a PBPS finalizer
& no supplied warm distribution
& TV
& \(\widetilde O(\kappa^{7/6}d^{1/6}
                 +\sqrt\kappa d^{1/4})\) \\
\midrule

This work
& reflected proximal dynamics with optimized initialization
& no supplied warm distribution
& TV
& \(\widetilde O(\kappa^{6/7}d^{1/4})\) \\
\bottomrule
\end{tabularx}
\end{table}

\paragraph{Lower bounds.}
Recent query lower bounds suggest that the \(\kappa\)-dependence of our warm-start rates is close to optimal.  For Gaussian targets with condition number \(\kappa\), \citet{zhou2026complexity} prove a lower bound of \(\Omega(\min\{d,\sqrt\kappa\log(1+\frac{d}{\sqrt\kappa\eps})\})\).  Such targets satisfy \cref{ass:target}, so once \(d\gtrsim\sqrt\kappa\) the factor \(\sqrt\kappa\) in \cref{flux:eq:complexity} cannot be improved by more than logarithmic terms.  Under a log-Sobolev inequality, \citet{zhou2026towards} prove that \(\Omega(\kappa)\) queries are necessary \emph{from a warm start} for \(\eps=1/10\), provided \(d\gtrsim\log^2\kappa\).  Log-Sobolev targets satisfy PI \citep{bakry2014analysis}, and a \(\beta\)-smooth potential admits the curvature bound with \(K=\beta\), for which \cref{eq:pi-query-bound} reads \(\widetilde O((\beta/m)d^{1/4})\).  A linear dependence on the smoothness-to-gap ratio is therefore not an artifact of our argument, and the penalty \(\sqrt{1+K/m}\) interpolates between that worst case and the rate \(\widetilde O(\sqrt{\beta/m}\,d^{1/4})\) attained when \(K\lesssim m\). What these bounds leave open is the factor \(d^{1/4}\) and, without a supplied warm start, the exponent \(\kappa^{6/7}\) of \cref{eq:warm-removal-cold-cost}.

\paragraph{Extensions of proximal sampling.}
The two-variable Gaussian augmentation has also been extended to bipartite
networks, yielding distributed Gibbs samplers with non-asymptotic guarantees
for structured log-concave targets \citep{yuan2023gibbs,yuan2025clientonly}.
The proximal viewpoint has also been developed beyond smooth,
strongly log-concave Euclidean targets.  Alternating-sampling and
proximal-oracle implementations address nonsmooth, weakly smooth, and
composite potentials
\citep{liang2022proximal,liang2023a,liang2026proximal}.  Adaptive step-size selection for the proximal sampler is explored in \citet{yuan2025proximal}. Non-Euclidean variants replace the
Gaussian geometry by log-Laplace regularization
\citep{gopi2023algorithmic}, while related algorithmic-diffusion
constructions treat uniform sampling from convex bodies
\citep{kook2024in}.  For heavy-tailed targets, Gaussian proximal oracles
can face an intrinsic high-accuracy barrier, motivating stable-oracle
analogues \citep{he2024separation}.  Recent strong-data-processing analyses
also establish convergence of the proximal sampler in relative Fisher
information \citep{wibisono2026mixing}.  These works broaden the target,
geometry, oracle, or convergence criterion; our methods instead modify the
outer dynamics to carry directional memory, within the smooth strongly
log-concave setting and its relaxation to a Poincar\'e inequality in
\cref{sec:pi-extension}.

\paragraph{Detailed comparison with concurrent PBPS.}
Concurrent work on the proximal bouncy particle sampler (PBPS, \citet[Sec.~3]{chen2026accelerated}) shares the Gaussian augmentation, outer reflection, occasional conditional refresh, and $\widetilde O(\sqrt\kappa d^{1/4})$ warm-start query bound with our methods (\cref{tab:reflected-proximal-comparison}). For a fixed auxiliary value $y$, all three accelerated branches share a common template: draw a reference $\widetilde X$, initialize $(Z_0,P_0)\gets(x,p_0)$, and evolve until time $T$ according to
$$
 \d Z_t=\sqrt\eta P_t\,\d t,\qquad
 \d P_t=-\Big\{\tfrac{Z_t-y}{\sqrt\eta}+\sqrt\eta\nabla V(\widetilde X)\Big\}\d t-\gamma P_t\,\d t+\sqrt{2\gamma}\,\d B_t,
$$
with events at rate $\lambda_t=\sqrt\eta[P_{t-}^\top h_t]_+$, where $h_t=\nabla V(Z_t)-\nabla V(\widetilde X)$ and $n_t=h_t/\norm{h_t}$. The three realizations specialize this template as follows:
\begin{itemize}[leftmargin=*,itemsep=0.1em,topsep=0.1em]
\item PBPS: $\widetilde X\sim q_y$, $p_0\sim\gamma_d$, $\gamma=0$, $T=\pi$, and the specular update $P_t\gets P_{t-}-2(P_{t-}^\top n_t)n_t$; return $Z_T$ and discard momentum.
\item Flux half-turn: $\widetilde X\sim q_y$, $p_0\sim\gamma_d$, $\gamma>0$, $T=T_\gamma=\pi/\omega$, and the Rayleigh flux update $P_t\gets P_{t-}-(P_{t-}^\top n_t+R_t)n_t$; return $Z_T$ and discard momentum.
\item Persistent quarter-turn: $\widetilde X=y-\eta\nabla V(\widehat x_y)+\sqrt\eta G$ with $G\sim\gamma_d$, $p_0=p$ carried from the previous iteration, $\gamma=0$, $T=\pi/2$, and the specular update; return $(Z_T,P_T)$.
\end{itemize}
The two half-turns (PBPS and flux) share the same outer structure: without events they map the normalized centered coordinates $(W,P)$ to $(-W,-P)$, discard momentum at the end of each step, and require one independent RGO draw per iteration (either as the refresh output or as the reference $\widetilde X$). Our flux variant adds damping, Brownian forcing, and Rayleigh resampling, demonstrating that reflected-proximal acceleration is not tied to a single deterministic collision rule. The persistent quarter-turn departs from this half-turn template in two ways. First, without events it maps $(W,P)$ to $(P,-W)$, storing centered position information in the momentum carried across outer iterations; its contraction proof therefore tracks a three-way transfer among $Y$, position, and momentum through a resolvent estimate rather than the half-turn energy argument. Second, its accelerated steps replace the exact reference draw $\widetilde X\sim q_y$ with a shifted Gaussian reference, so an approximate RGO is invoked only on rare refresh steps ($\operatorname{polylog}(d,\kappa,1/\eps)$ times in expectation). While conditional invariance under a shifted reference holds for any kernel that preserves $q_y$ at each fixed reference value, our contraction proof establishes that the persistent quarter-turn also retains accelerated mixing under this replacement. Finally, \cref{sec:pi-extension} extends the warm-start guarantees for both of our realizations to smooth semiconvex targets under a Poincar\'e inequality, whereas the analysis of \citet{chen2026accelerated} assumes strong log-concavity.

\begin{table}[t]
\centering
\caption{Comparison of reflected proximal samplers from a
constant-\(D_2\) warm start.}
\label{tab:reflected-proximal-comparison}
\renewcommand{\arraystretch}{1.12}
\begin{tabularx}{\textwidth}{@{}>{\raggedright\arraybackslash}p{0.20\textwidth}YYY@{}}
\toprule
& PBPS & Flux half-turn & Persistent quarter-turn \\
\midrule
Reference \(\widetilde X\)
& \(q_y\) & \(q_y\) & shifted Gaussian \\
\midrule
Duration \(T\)
& \(\pi\) & \(\pi/\omega\) & \(\pi/2\) \\
\midrule
Damping and noise
& none & \(\gamma>0\) & none \\
\midrule
Event update
& specular & Rayleigh resampling & specular \\
\midrule
Momentum across outer steps
& discarded & discarded & retained \\
\midrule
Oracle queries
& \(\widetilde O(\sqrt\kappa d^{1/4})\)
& \(\widetilde O(\sqrt\kappa d^{1/4})\)
& \(\widetilde O(\sqrt\kappa d^{1/4})\) \\
\midrule
Approximate-RGO calls
& \(\widetilde O(\sqrt\kappa d^{1/4})\)
& \(\widetilde O(\sqrt\kappa d^{1/4})\)
& \(\operatorname{polylog}(d,\kappa,1/\eps)\) \\
\bottomrule
\end{tabularx}
\end{table}

\section{Proofs for the Flux-Resampled Half-Turn}
\label{app:flux}

This appendix proves
\cref{prop:flux-main-reversal,prop:flux-main-antipode,flux:thm:ideal,flux:thm:complexity}.
The proof follows the algorithm:

\begin{enumerate}[leftmargin=*,label=\arabic*.]
\item establish the geometry of the localized conditional law and show
      that the event-driven half-turn preserves it;
\item prove that one conditional half-turn approximately reverses every
      affine position coordinate that is relevant to the outer chain;
\item combine this estimate with the outer reflection and occasional
      conditional refreshment to obtain accelerated \(L^2\) contraction;
\item replace exact conditional draws and unbounded event clocks by finite
      implementations, then account for total-variation error and oracle
      queries.
\end{enumerate}

Standard ingredients are cited rather than reproved: Brascamp--Lieb and
Poincar\'e inequalities for strongly log-concave measures, skew balance for
velocity-jump processes, and the existing approximate-RGO implementation.
The new parts are the covariance-weighted half-turn estimate and the
modified-energy argument that turns it into contraction of the complete
reflected proximal chain.

Throughout this appendix, set
\[
 a=\alpha\eta,\qquad s=\beta\eta,\qquad
 g_y(w)=\sqrt\eta\,\nabla V(y+\sqrt\eta w).
\]
Thus \(0<a\le s\).  Under the change of variables
\(w=(x-y)/\sqrt\eta\), the conditional density \(q_y\) becomes
\(q_y^w(w)\propto\exp\{-\norm w^2/2-V(y+\sqrt\eta w)\}\).
We write \(I\) for the identity on the space at hand, use \(I_d\) when its
matrix dimension is worth displaying, and set \(\gamma_d=\cN(0,I_d)\).
Unless a measure is displayed explicitly,
inner products and norms below are taken in the stationary \(L^2\) space
of the objects under discussion; gradients of functions of \(y\) are with
respect to \(y\).
Operators act on observables; probability densities therefore evolve by
their adjoints.  The few operators used in the contraction proof are:

\begin{center}
\begin{tabular}{@{}cl@{}}
\toprule
symbol & meaning\\
\midrule
\(\mathsf P\) & conditional expectation given \(Y\)\\
\(\mathsf Q=I-\mathsf P\) & remaining conditional variation\\
\(\mathsf U\) & outer reflection \((x,y)\mapsto(x,2x-y)\)\\
\(\mathsf F\) & conditional flux half-turn\\
\(\mathsf K\) & one complete ideal iteration\\
\(\mathsf T=\mathsf P\mathsf U\mathsf P\) & reflection seen on \(Y\)-only observables\\
\(\Gamma=(I-\mathsf T^2)^{1/2}\) & amount transferred out of the \(Y\)-only part\\
\bottomrule
\end{tabular}
\end{center}

\subsection{Conditional Geometry and Validity of the Half-Turn}
\label{flux:app:validity}

We first collect the conditional identities used throughout the proof and
then verify that the event-driven motion is an exact conditional update.
Neither step requires a classical Hessian: the identities hold in the weak
sense for the strongly convex, smooth potential in \cref{ass:target}.

Let \(W=(X-y)/\sqrt\eta\) under \(q_y\), and denote its mean and
covariance by \(\bar w_y\) and \(\Sigma_y\).  The following lemma says
that the departure of this conditional law from a standard Gaussian is
measured by the same matrix \(I-\Sigma_y\) that controls the curvature of
the \(Y\)-marginal.

\begin{lemma}[Localized conditional geometry]
\label{flux:lem:conditional-geometry}
For every \(y\),
\[
 \frac1{1+s}I\preceq\Sigma_y\preceq\frac1{1+a}I,
 \qquad
 \frac a{1+a}I\preceq I-\Sigma_y\preceq\frac s{1+s}I.
\]
If \(\Pi_\eta^Y(y)\propto \e^{-U_\eta(y)}\), then
\[
 \nabla U_\eta(y)=\eta^{-1}\{y-\E[X\mid Y=y]\},\qquad
 \nabla^2U_\eta(y)=\eta^{-1}(I-\Sigma_y),
\]
and hence
\(\alpha(1+a)^{-1}I\preceq\nabla^2U_\eta
\preceq\beta(1+s)^{-1}I\).
Moreover, when \(s<1\),
\begin{equation}
 \E g_y(W)=-\bar w_y,\qquad
 \Cov(W,g_y(W))=I-\Sigma_y,\qquad
 \Cov(g_y(W))\preceq\frac{s}{1-s}(I-\Sigma_y).
 \label{flux:eq:conditional-identities}
\end{equation}
\end{lemma}

\begin{proof}
The normalized conditional potential has curvature between
\((1+a)I\) and \((1+s)I\).  Brascamp--Lieb gives the upper covariance
bound.  The information inequality gives
\(\Cov(W)\succeq\{\E\nabla^2[-\log q_y^w(W)]\}^{-1}
\succeq(1+s)^{-1}I\), proving the remaining covariance bounds.
Differentiating the conditional normalizer gives the two identities for
\(U_\eta\).

For \cref{flux:eq:conditional-identities}, use
\(\nabla\log q_y^w(w)=-w-g_y(w)\).  Integration by parts first gives
\(\E(W+g_y(W))=0\) and
\(\Sigma_y+\Cov(W,g_y(W))=I\).  Applying it once more to \(g_y\) gives
\[
 \E\nabla g_y(W)=I-\Sigma_y+\Cov(g_y(W)).
\]
The conditional law has Poincar\'e constant at most one, while
\(0\preceq\nabla g_y\preceq sI\).  Therefore
\(\Cov(g_y(W))\preceq s\E\nabla g_y(W)\); rearrangement proves the last
claim.  These standard Brascamp--Lieb and integration-by-parts facts are
reviewed, for example, by \citet{chewi2023log}.  The extension from smooth
potentials and test functions to the regularity in \cref{ass:target} is
collected in \cref{flux:app:regularity}.
\end{proof}

We next verify the conditional process.  Fix \(y\) and a reference
\(\widetilde w\).  To lighten notation write
\(g=g_y\), \(b_{\rm ref}=g(\widetilde w)\), and
\(h(w)=g(w)-b_{\rm ref}\).  In normalized coordinates, the generator of
the diffusion between events is
\[
 \mathcal L_0
 =p^\top\nabla_w-(w+b_{\rm ref})^\top\nabla_p
 +\gamma(\Delta_p-p^\top\nabla_p),
\]
and events occur at rate \([p^\top h(w)]_+\).  At an event, write
\(n=h/\norm h\), decompose \(p=v+zn\), draw \(R>0\) with density
\(r \e^{-r^2/2}\), and replace \(p\) by \(v-Rn\).  When \(h=0\), the
event rate is zero and the direction \(n\) is never evaluated.  More
explicitly, if \(\mathcal J\) denotes the event part of the generator, then
\[
 (\mathcal Jf)(w,p)
 =[p^\top h(w)]_+
 \int_0^\infty\left[
 f\left(w,p-\{p^\top n(w)+r\}n(w)\right)-f(w,p)
 \right]r\e^{-r^2/2}\,\d r,
\]
with \(\mathcal Jf=0\) when \(h(w)=0\).  The full generator used below is
\(\mathcal L=\mathcal L_0+\mathcal J\).

\begin{proposition}[Fixed-reference validity]
\label{flux:prop:fixed-validity}
For every fixed \((y,\widetilde w)\), the preceding process is
non-explosive and preserves \(q_y^w\otimes\gamma_d\).  If
\(\mathsf S_t^{y,\widetilde w}\) is its transition operator and
\((Jf)(w,p)=f(w,-p)\), then
\[
 (\mathsf S_t^{y,\widetilde w})^*=J\mathsf S_t^{y,\widetilde w}J.
\]
It is a strongly continuous contraction semigroup on
\(L^2(q_y^w\otimes\gamma_d)\).
Consequently, inserting an independent Gaussian momentum, running for a
fixed time, and discarding momentum gives a reversible position kernel for
\(q_y^w\).  The conclusion remains true after averaging over any reference
law independent of the incoming \((W,P)\) conditional on \(Y=y\).
\end{proposition}

\begin{proof}
The calculation follows the usual skew-balance argument for event-driven
kinetic samplers
\citep{vanetti2017piecewise,bierkens2020boomerang,michel2020forward,
monmarche2022exact}; we give the semigroup step because our reference
motion also contains Langevin noise.

The proof has three steps.  First, a trajectory with no event can be
reversed after flipping momentum; its no-event probability supplies exactly
the likelihood correction caused by the residual force.  Second, the
Rayleigh update balances the stationary probability flow through a single
collision.  Finally, we sum these identities over the number and times of
the collisions.  This proves reversal of the complete process as well as
invariance.

\emph{Non-explosion.}
With
\(\mathcal V=1+\{\norm{w+b_{\rm ref}}^2+\norm p^2\}/2\), the full
generator obeys
\(\mathcal L\mathcal V\le C_{\widetilde w}(1+\mathcal V)\), and the event
rate is at most \(C_{\widetilde w}(1+\mathcal V)\), where
\(C_{\widetilde w}\) is a finite constant that may depend on the fixed
reference.  Indeed, a jump changes \(\mathcal V\) by
\((R^2-z^2)/2\), \(\E R^2=2\), and
\(\norm{h(w)}\le s\norm{w-\widetilde w}\).
A stopped Dynkin formula and Gronwall bound both
\(\E\mathcal V(W_t,P_t)\) and the
expected event count on every finite interval.  The recursive
integrated-hazard construction is therefore well defined and
non-explosive.

\emph{Reversal between and at events.}
Let \(Z_t^0=(W_t^0,P_t^0)\) be the affine Langevin diffusion without
events, let \(\lambda(w,p)=[p^\top h(w)]_+\), and define its killed
semigroup---that is, its transition weighted by the probability of seeing
no event---by
\[
 \mathsf D_tf(z)=\E_z\left[
 \e^{-\int_0^t\lambda(Z_u^0)\,\d u}f(Z_t^0)\right].
\]
The affine diffusion is invariant and momentum-reversible for the Gaussian
density
\(\mu_{\rm ref}\propto
\exp\{-\norm{w+b_{\rm ref}}^2/2-\norm p^2/2\}\).
For the desired density \(\mu=q_y^w\otimes\gamma_d\), the ratio
\(\varrho=\d\mu/\d\mu_{\rm ref}\) satisfies
\(\nabla\log\varrho=-h\).  Since the position
path is absolutely continuous,
\[
 \frac{\varrho(W_t^0)}{\varrho(W_0^0)}
 =\exp\left\{-\int_0^tP_u^{0\top}h(W_u^0)\,\d u\right\}
 =\exp\left\{-\int_0^t(\lambda-\lambda\circ J)(Z_u^0)\,\d u\right\}.
\]
Stationary path reversal for the affine diffusion therefore gives
\(\mathsf D_t^*=J\mathsf D_tJ\) first for bounded nonnegative functions.
Taking one test function equal to one shows \(\mu\mathsf D_t\le\mu\), and Jensen's inequality
extends the identity to \(L^2(\mu)\).

Let \(\mathcal G\) be the gain kernel: it applies one collision, weighted by
the rate at which that collision occurs.  At fixed \(w\), write the incoming momentum as
\(p=v+zn\), where \(v\perp n\) and \(z>0\).  Apart from a common constant,
the stationary marked flux is
\[
 q_y^w(w)\norm{h(w)}\varphi_{d-1}(v)
 z \e^{-z^2/2}r \e^{-r^2/2}\,
 \d w\,\d v\,\d z\,\d r,\qquad z,r>0.
\]
Here \(\varphi_{d-1}\) is the standard Gaussian density on the subspace
orthogonal to \(n\).
It is unchanged by \((v,z,r)\mapsto(-v,r,z)\), which turns a marked
transition into the momentum reversal of its time-reversed transition.
Thus \(\mathcal G^*=J\mathcal GJ\) in the corresponding kernel sense.

The contribution of paths with exactly \(m\) events is the time-simplex
integral of
\(\mathsf D_{t_0}\mathcal G\mathsf D_{t_1}\cdots
\mathcal G\mathsf D_{t_m}\).
Reversing the simplex coordinates and using the preceding two identities
shows that its adjoint is obtained by conjugating with \(J\).  Tonelli
justifies this calculation for nonnegative tests, and non-explosion makes
the sum over \(m\) conservative, meaning that it retains total probability
one.  Summing proves
\((\mathsf S_t^{y,\widetilde w})^*=J\mathsf S_t^{y,\widetilde w}J\);
taking one test
function equal to one proves invariance.  Jensen gives \(L^2\)
contractivity, while stochastic continuity and density of bounded
continuous compactly supported functions give strong continuity.

Finally, if \(\iota f(w,p)=f(w)\) and
\(\mathsf M u(w)=\int u(w,p)\,\gamma_d(\d p)\), then
\(\mathsf M=\iota^*\), \(J\iota=\iota\), and \(\mathsf M J=\mathsf M\).
Thus inserting a Gaussian momentum and then discarding it are adjoint
operations.  Suppressing the fixed superscripts \((y,\widetilde w)\),
\((\mathsf M\mathsf S_t\iota)^*=\mathsf M\mathsf S_t\iota\).
Conditioning on an
independent reference and applying Fubini proves the last assertion.
\end{proof}

\begin{proof}[Proof of \cref{prop:flux-main-reversal}]
Use \cref{flux:prop:fixed-validity} with
\(\widetilde W\sim q_y^w\), independently of the incoming position.  It
shows that the conditional position kernel in
\cref{alg:flux-half-turn} is reversible and invariant for \(q_y\).  The
outer map \((x,y)\mapsto(x,2x-y)\) preserves \(\Pi_\eta\), because it
only changes the sign of \(y-x\) and has unit absolute Jacobian.
\cref{prop:framework-exactness} now proves invariance of the complete row
and of its position marginal \(\pi\).
\end{proof}

\subsection{The Covariance-Weighted Half-Turn Estimate}
\label{flux:app:halfturn}

The next proposition is the local estimate used by the outer-chain proof.
It does not ask the conditional update to approximate a reflection uniformly
over all functions.  It asks only that centered affine position coordinates
be nearly negated, with a smaller error in directions where the localized
conditional is closer to the Gaussian reference.

\begin{proposition}[Conditional affine half-turn]
\label{flux:prop:weighted-affine}
\label{prop:flux-main-antipode}
There is a universal \(C\) such that, whenever \(s=\beta\eta\le1/2\),
the normalized-coordinate position kernel \(F_y^w\) of
\cref{alg:flux-half-turn} satisfies,
uniformly in \(y,u\in\R^d\),
\begin{equation}
 \norm{(I+F_y^w)\ell_{u,y}}_{L^2(q_y^w)}^2
 \le Cs\,u^\top(I-\Sigma_y)u,
 \qquad \ell_{u,y}(w)=u^\top(w-\bar w_y).
 \label{flux:eq:weighted-halfturn}
\end{equation}
\end{proposition}

The proof has two ingredients.  First, the damped harmonic flight can be
solved exactly on affine observables.  At the chosen half-period, dependence
on the fresh initial momentum vanishes and the incoming position receives a
negative coefficient.  Second, the event correction appears as a forcing
term whose second moment is controlled by
\(s(I-\Sigma_y)\) through \cref{flux:eq:conditional-identities}.

Fix \(y\) and a reference \(\widetilde w\), and retain the notation
\(g,b_{\rm ref},h\) from \cref{flux:app:validity}.  Let
\(\bar r=\E R=\sqrt{\pi/2}\),
\(\psi(z)=z_+(z+\bar r)\), and
\(n(w)=h(w)/\norm{h(w)}\) when \(h(w)\ne0\).  For affine position and
momentum observables
\(\ell_{u,y}(w)=u^\top(w-\bar w_y)\) and \(m_u(p)=u^\top p\), define
\[
 k_{u,\widetilde w}(w,p)
 =-u^\top\{\bar w_y+b_{\rm ref}\}
  -u^\top h(w)\psi(p^\top n(w)),
\]
where the second term is zero when \(h(w)=0\).

\begin{lemma}[Affine evolution]
\label{flux:lem:affine-evolution}
The two observables belong to the \(L^2(q_y^w\otimes\gamma_d)\) generator
domain and satisfy
\[
 \mathcal L\ell_{u,y}=m_u,\qquad
 \mathcal Lm_u=-\ell_{u,y}-\gamma m_u+k_{u,\widetilde w}.
\]
Writing \(\omega=\sqrt{1-\gamma^2/4}\), set
\[
 A_\gamma(t)=\e^{-\gamma t/2}
 \left\{\cos(\omega t)+\frac\gamma{2\omega}\sin(\omega t)\right\},
 \qquad
 B_\gamma(t)=\e^{-\gamma t/2}\frac{\sin(\omega t)}\omega.
\]
Then, for the fixed-reference semigroup,
\[
 \mathsf S_t^{y,\widetilde w}\ell_{u,y}
 =A_\gamma(t)\ell_{u,y}+B_\gamma(t)m_u
  +\int_0^tB_\gamma(t-r)
       \mathsf S_r^{y,\widetilde w}k_{u,\widetilde w}\,\d r.
\]
At \(T_\gamma=\pi/\omega\), this becomes
\begin{equation}
 (I+\mathsf S_{T_\gamma}^{y,\widetilde w})\ell_{u,y}
 =(1-\e^{-\gamma T_\gamma/2})\ell_{u,y}
  +\int_0^{T_\gamma}B_\gamma(T_\gamma-r)
       \mathsf S_r^{y,\widetilde w}k_{u,\widetilde w}\,\d r.
 \label{flux:eq:halfturn-duhamel}
\end{equation}
\end{lemma}

\begin{proof}
The deterministic part of the generator gives
\(\mathcal L\ell_{u,y}=m_u\) and
\(-\ell_{u,y}-\gamma m_u-u^\top(\bar w_y+b_{\rm ref})\) in the momentum
coordinate.
At an event, the momentum change is
\(-(Z+R)u^\top n\), where \(Z=p^\top n>0\); multiplying its conditional
mean by the rate \(\norm h Z\) gives the remaining term
\(-u^\top h\,\psi(Z)\).

These quantities are square integrable because \(q_y^w\) has Gaussian
tails, \(h\) grows at most linearly, and \(p^\top n\) is standard normal
conditional on the position.  Applying the stopped Dynkin formula and then
removing the stopping time using the Lyapunov bound in
\cref{flux:prop:fixed-validity} gives, after suppressing the fixed
superscripts \((y,\widetilde w)\),
\[
 \mathsf S_t\ell_{u,y}-\ell_{u,y}
 =\int_0^t\mathsf S_rm_u\,\d r,\qquad
 \mathsf S_tm_u-m_u
 =\int_0^t\mathsf S_r(-\ell_{u,y}-\gamma m_u
                       +k_{u,\widetilde w})\,\d r.
\]
Strong continuity and division by \(t\) prove the generator-domain
statements.  Solving the resulting two-dimensional linear equation by
variation of constants gives the displayed formula.  At \(T_\gamma\),
\(B_\gamma(T_\gamma)=0\) and
\(A_\gamma(T_\gamma)=-\e^{-\gamma T_\gamma/2}\), which proves
\cref{flux:eq:halfturn-duhamel}.
\end{proof}

\begin{proof}[Proof of \cref{flux:prop:weighted-affine}]
Let \(W,\widetilde W\) be independent with law \(q_y^w\), and average
first over the Gaussian momentum; we denote this last expectation by
\(\E_{P}\).  Direct one-dimensional Gaussian
integration gives
\(\E\psi(Z)=1\) and
\(\Var\psi(Z)=5/2+\pi/4\) for \(Z\sim\cN(0,1)\).  Using
\(\E g_y(W)=-\bar w_y\), we obtain
\[
 \E_P k_{u,\widetilde W}
 =-u^\top\{g_y(W)-\E g_y(W)\},
\]
while the conditional momentum variance is
\((5/2+\pi/4)[u^\top\{g_y(W)-g_y(\widetilde W)\}]^2\).
Independence of \(W\) and \(\widetilde W\) therefore gives
\begin{equation}
 \E k_{u,\widetilde W}^2
 =\left(6+\frac\pi2\right)
   \Var(u^\top g_y(W))
 \le Cs\,u^\top(I-\Sigma_y)u,
 \label{flux:eq:forcing-moment}
\end{equation}
where the last step is \cref{flux:eq:conditional-identities}.

Insert fresh Gaussian momentum in \cref{flux:eq:halfturn-duhamel}, average
the independent reference, and discard the terminal momentum.  Jensen's
inequality, semigroup contractivity, and
\(\int_0^{T_\gamma}|B_\gamma(r)|\,\d r\le C\) yield
\[
 \norm{(I+F_y^w)\ell_{u,y}}_2^2
 \le C\left\{\gamma^2u^\top\Sigma_yu
                   +\E k_{u,\widetilde W}^2\right\}.
\]
Here the norm is in \(L^2(q_y^w)\), and the remaining expectation includes
the independent position, reference, and Gaussian momentum.
The parameter choice in \cref{sec:flux} gives
\(\gamma^2\le as/(1+a)\).  Since
\(I-\Sigma_y\succeq a(1+a)^{-1}I\) and \(\Sigma_y\preceq I\), the
damping term is at most
\(s u^\top(I-\Sigma_y)u\).  Combining this with
\cref{flux:eq:forcing-moment} proves
\cref{flux:eq:weighted-halfturn}.
\end{proof}

\subsection{From the Conditional Half-Turn to Accelerated Contraction}
\label{flux:app:contraction}

We now prove the ideal-chain theorem.  The argument follows two components
of an observable: its average at fixed \(Y\), and its remaining conditional
variation.  Outer reflection transfers part of the first component into
the second.  The flux half-turn approximately reverses precisely this
transferred part, while an occasional independent conditional refresh
damps all conditional variation.

Let \(\mathfrak H=L^2(\Pi_\eta)\), let
\(\mathsf P f=\E[f\mid Y]\), and set
\(\mathsf Q=I-\mathsf P\).  Let \(\mathsf U\) be the pullback of the
outer reflection and \(\mathsf F\) the joint conditional half-turn.  By
\cref{flux:prop:fixed-validity},
\(\mathsf F=\mathsf P+\mathsf F_Q\), where \(\mathsf F_Q\) is a
self-adjoint contraction on \(\operatorname{range}(\mathsf Q)\).  The first
term records observables depending only on \(Y\); the second acts on their
remaining conditional variation.  Thus one
ideal iteration has observable operator
\begin{equation}
 \mathsf K=\mathsf U\{\mathsf P+(1-\rho)\mathsf F_Q\}.
 \label{flux:eq:ideal-operator}
\end{equation}

\begin{proposition}[Ideal \(L^2\) contraction]
\label{flux:prop:ideal-L2}
There are universal constants \(c,C>0\) such that, when \(s\le c\) and
\(\rho\asymp\sqrt{a\log(\e/a)}\), the ideal chain satisfies
\[
 \norm{\mathsf K^n}_{L^2_0(\Pi_\eta)\to L^2_0(\Pi_\eta)}
 \le \sqrt3\exp\left\{-\frac{cn\sqrt a}
                              {\sqrt{\log(\e/a)}}\right\}.
\]
\end{proposition}

The rest of this subsection proves the proposition.  We first quantify
what reflection transfers, then lift the conditional estimate in
\cref{flux:prop:weighted-affine} to the complete state space, and finally
construct an energy that decreases in one step.

\subsubsection{What Outer Reflection Transfers}

On the centered \(Y\)-only subspace
\(\mathcal H_Y=\operatorname{range}(\mathsf P)\cap L^2_0(\Pi_\eta)\),
define
\[
 \mathsf T=\mathsf P\mathsf U\mathsf P,
 \qquad \Gamma=(I-\mathsf T^2)^{1/2},
 \qquad \mathsf C=\mathsf Q\mathsf U\mathsf P.
\]
Under stationarity, let \(X\sim\pi\) and \(Z\sim\cN(0,I_d)\) be
independent, and write
\(Y_+=X+\sqrt\eta Z\), \(Y_-=X-\sqrt\eta Z\).  These are not two
independent chain iterates: they share \((X,Z)\) and are the auxiliary
values before and after the same reflection, so \(Y_-=2X-Y_+\).  Then
\((\mathsf Tf)(y)=\E[f(Y_-)\mid Y_+=y]\), and
\(\norm{\mathsf Cf}^2=\E\Var(f(Y_-)\mid Y_+)\).
Thus \(\mathsf T\) is the part retained as a function of \(Y_+\), while
\(\mathsf C\) is the part transferred into conditional variation.  The
operator \(\Gamma\) measures the size of this transferred part.

\begin{lemma}[Outer-reflection geometry]
\label{flux:lem:outer-geometry}
If \(s\le1/3\), then \(\mathsf T\) is self-adjoint and, on
\(\mathcal H_Y\),
\begin{equation}
 -\tfrac12I\preceq\mathsf T
 \preceq\frac{1-a}{1+a}I,
 \qquad
 \Gamma\succeq\gamma_a I,
 \qquad \gamma_a=\frac{2\sqrt a}{1+a}.
 \label{flux:eq:outer-spectrum}
\end{equation}
Moreover,
\(4\eta\norm{\nabla\mathsf T f}^2\le\norm{\Gamma f}^2\), and
\(\mathsf C=\mathsf V\Gamma\) for an isometry
\(\mathsf V:\mathcal H_Y\to\operatorname{range}(\mathsf Q)\).
\end{lemma}

\begin{proof}
Exchangeability of \((Y_+,Y_-)\) makes \(\mathsf T\) self-adjoint.
The conditional density of \(Y_-\) given \(Y_+=y\) is proportional to
\[
 \exp\left\{-V\left(\frac{y+u}{2}\right)
              -\frac{\norm{y-u}^2}{8\eta}\right\}.
\]
Its curvature in \(u\) is at least
\((\alpha+\eta^{-1})I/4\), while the derivative with respect to \(u\)
of its \(y\)-score has norm at most \((\eta^{-1}-\alpha)/4\).
Conditional Brascamp--Lieb therefore gives
\[
 \eta\norm{\nabla\mathsf T f}^2
 \le\frac{(1-a)^2}{4(1+a)}\norm{\mathsf C f}^2
 \le\frac14\norm{\Gamma f}^2.
\]
The \(Y\)-marginal has Poincar\'e constant at most
\(\alpha^{-1}+\eta\) by \cref{flux:lem:conditional-geometry}.  Applying
this inequality to the centered function \(\mathsf Tf\) and using
\(\norm f^2=\norm{\mathsf Tf}^2+\norm{\mathsf Cf}^2\) gives
\(\norm{\mathsf T}\le(1-a)/(1+a)\).

For the lower spectral edge, exchangeability and Gaussian Poincar\'e give
\(\langle h,(I-\mathsf T)h\rangle\le2\eta\norm{\nabla h}^2\).
Apply this with \(h=\mathsf Tf\) and use the preceding derivative bound.
Spectral calculus yields
\((I-\mathsf T)^2(I+2\mathsf T)\succeq0\), hence
\(\mathsf T\succeq-I/2\).  Because
\((1-a)/(1+a)\ge1/2\), these two edges imply the stated lower bound on
\(\Gamma\).  Finally, unitarity of \(\mathsf U\) gives
\(\mathsf C^*\mathsf C=I-\mathsf T^2=\Gamma^2\), so
\(\mathsf V=\mathsf C\Gamma^{-1}\) is an isometry.

 The weak-gradient and spectral-domain justification for this calculation
 is given in \cref{flux:app:regularity}.
\end{proof}

Reflection creates a conditional-position signal that is almost affine.
For weakly differentiable \(f\in\mathcal H_Y\), set
\[
 b_y=\E[\nabla f(Y_-)\mid Y_+=y],\qquad
 L_f=2\sqrt\eta\,b_{Y_+}^\top(W-\bar w_{Y_+}),\qquad
 W=\frac{X-Y_+}{\sqrt\eta}.
\]
Here \(b_y\) is the average slope of \(f(Y_-)\) conditional on the
reflected value \(Y_+=y\), and \(L_f\) is the corresponding centered
linear signal in the conditional position \(W\).  The next lemma separates
this leading signal from a nonlinear remainder.  Its quantities
\(\mathcal R_f\) and \(\mathcal D_f\) measure, respectively, the gradient
variation not retained by conditional averaging and the
curvature-weighted size of the affine signal.

\begin{lemma}[Affine approximation after reflection]
\label{flux:lem:affine-after-reflection}
There is an error \(e_f\), centered conditionally on \(Y_+\), such that
\(\mathsf C f=L_f+e_f\).  If
\[
 \mathcal R_f=\norm{\nabla f}^2-\norm{\mathsf T\nabla f}^2,
 \qquad
 \mathcal D_f=\E[b_{Y_+}^\top(I-\Sigma_{Y_+})b_{Y_+}],
\]
then, for sufficiently small universal \(s\),
\begin{equation}
 \norm{e_f}^2\le4\eta\mathcal R_f,\qquad
 \mathcal R_f+\mathcal D_f
 \le C\norm{\nabla(I-\mathsf T)f}\norm{\nabla f}.
 \label{flux:eq:affine-reflection-control}
\end{equation}
\end{lemma}

\begin{proof}
Here \(\mathsf T\nabla f\) means that \(\mathsf T\) is applied to every
coordinate of the vector-valued gradient; in particular,
\((\mathsf T\nabla f)(y)=b_y\).
Conditionally on \(Y_+=y\),
\(Y_-=y+2\sqrt\eta W\).  Subtracting its best first-order conditional
approximation gives
\[
 e_f=f(Y_-)-\mathsf Tf(Y_+)
       -2\sqrt\eta\,b_{Y_+}^\top(W-\bar w_{Y_+}).
\]
Conditional Poincar\'e proves
\(\norm{e_f}^2\le4\eta\mathcal R_f\).

For the second estimate, differentiate the conditional expectation
weakly.  Integration by parts and
\cref{flux:eq:conditional-identities} give
\[
 \nabla(I-\mathsf T)f
 =\nabla f-b_y+2(I-\Sigma_y)b_y
 +\eta^{-1/2}\Cov(e_f,g_y(W)\mid Y_+=y).
\]
In this identity, \(\nabla f\) on the right is evaluated at \(y\), and the
conditional covariance is over \(W\) given \(Y_+=y\).  Write
\(b(y)=b_y\).
Take its inner product with \(\nabla f+b_y\) and average over \(y\).
The first term gives \(\mathcal R_f\), and the leading covariance term
gives \(4\mathcal D_f\).  Also, the spectral lower bound
\(\mathsf T\succeq-I/2\) implies
\[
 \norm{\nabla f-b}^2
 \le3\{\norm{\nabla f}^2-\norm b^2\}=3\mathcal R_f,
\]
The two remaining cross terms are bounded using this inequality,
\cref{flux:eq:conditional-identities}, the first inequality in
\cref{flux:eq:affine-reflection-control}, and
\(I-\Sigma_y\preceq sI\).  They are at most
\[
 C\sqrt{s\mathcal R_f\mathcal D_f}
 +C\sqrt{s\mathcal R_f(s\mathcal R_f+\mathcal D_f)}.
\]
The leading inner product is at most
\(2\norm{\nabla(I-\mathsf T)f}\norm{\nabla f}\).
Young's inequality absorbs the displayed cross terms when \(s\) is
sufficiently small.  The graph-norm extension to the stated weak-gradient
domain is recorded in \cref{flux:app:regularity}.
\end{proof}

\subsubsection{Lifting the Local Half-Turn Estimate}

The next result is the bridge between the conditional dynamics and the
outer Markov chain.

We first record the operator comparison that accounts for the logarithm in
that bridge.  It converts an interpolation-type estimate into a quadratic
form estimate without assuming that the two operators commute.

\begin{lemma}[Dyadic operator comparison]
\label{flux:lem:dyadic-comparison}
Suppose \(mI\preceq B\preceq MI\) and a bounded operator \(A\) satisfies
\(\norm{Av}^2\le\norm v\norm{Bv}\).  Then
\[
 A^*A\preceq C\{1+\log(M/m)\}B.
\]
Here \(m\) and \(M\) are local to this lemma: they are spectral bounds for
\(B\), not the Poincar\'e gap \(m\) of \cref{ass:pi-extension}.
\end{lemma}

\begin{proof}
Partition the spectrum of \(B\) into dyadic intervals and write
\(v=\sum_{j=1}^Jv_j\), where
\(J\le C\{1+\log(M/m)\}\).  If \(\lambda_j\) is the lower endpoint of
the \(j\)-th interval, then the hypothesis gives
\(\norm{Av_j}\le\sqrt{2\lambda_j}\norm{v_j}\).  Consequently,
\[
 \norm{Av}^2
 \le J\sum_j2\lambda_j\norm{v_j}^2
 \le2J\langle v,Bv\rangle,
\]
which proves the claim.
\end{proof}

\begin{proposition}[Global half-turn estimate]
\label{flux:prop:global-halfturn}
For sufficiently small universal \(s\) and every
\(v\in\mathcal H_Y\),
\begin{equation}
 \norm{(I+\mathsf F_Q)\mathsf Vv}^2
 \le C\log(\e/\gamma_a)\norm{\Gamma v}^2.
 \label{flux:eq:global-halfturn}
\end{equation}
\end{proposition}

\begin{proof}
First restrict \(f\) to the spectral subspace of \(\mathsf T\)
corresponding to \([1/2,1]\).  Apply
\cref{flux:prop:weighted-affine} conditionally to the affine part \(L_f\)
of \(\mathsf Cf\), use
\(\norm{I+\mathsf F_Q}\le2\) on \(e_f\), and then apply
\cref{flux:lem:affine-after-reflection}.  This gives
\[
 \norm{(I+\mathsf F_Q)\mathsf C f}^2
 \le C\eta\norm{\nabla(I-\mathsf T)f}\norm{\nabla f}.
\]
On this spectral subspace, applying the derivative estimate in
\cref{flux:lem:outer-geometry} to \(\mathsf T^{-1}f\) and
\(\mathsf T^{-1}(I-\mathsf T)f\) gives
\[
 \sqrt\eta\norm{\nabla f}\le C\norm{\Gamma f},
 \qquad
 \sqrt\eta\norm{\nabla(I-\mathsf T)f}
 \le C\norm{\Gamma^3f}.
\]
Since \(\mathsf C=\mathsf V\Gamma\), setting
\(f=\Gamma^{-1}v\), which is allowed by
\(\Gamma\succeq\gamma_aI\), shows that
\(\mathsf A_{\rm ht}=(I+\mathsf F_Q)\mathsf V\) obeys
\(\norm{\mathsf A_{\rm ht}v}^2
\le C\norm v\norm{\Gamma^2v}\).
Applying \cref{flux:lem:dyadic-comparison} to
\(A=\mathsf A_{\rm ht}\), \(B=\Gamma^2\), and
\(m=\gamma_a^2\) gives the required bound without treating the generally
non-self-adjoint half-turn error \(\mathsf A_{\rm ht}\) as a positive
operator.

On the complementary spectral subspace,
\(\Gamma^2\succeq3I/4\), while \(\norm{\mathsf A_{\rm ht}}\le2\), so the result
is immediate.  The two output vectors need not be orthogonal; the triangle
inequality and
\(\norm{r_1+r_2}^2\le2\norm{r_1}^2+2\norm{r_2}^2\) prove
\cref{flux:eq:global-halfturn} for an arbitrary \(v\).
\end{proof}

\subsubsection{A Decreasing Energy}

For centered \(f\in\mathfrak H\), write
\(f_Y=\mathsf Pf\), \(f_Q=\mathsf Qf\), and
\(f_V=\mathsf V^*f_Q\).  Since \(\mathsf V\) is an isometry, the
conditional component has the orthogonal decomposition
\[
 f_Q=\mathsf V f_V+f_\perp,\qquad
 f_\perp=(I-\mathsf V\mathsf V^*)f_Q.
\]
The three pieces have distinct roles.  The term \(f_Y\) is variation across
values of \(Y\); \(\mathsf V f_V\) is the conditional signal created from
it by reflection; and \(f_\perp\) is the remaining conditional variation.
The cross term below tracks the exchange between the first two pieces,
whereas the RGO branch controls the full \(f_Q\), including \(f_\perp\).
Define
\[
 \mathcal E_\vartheta(f)=\norm f^2+\vartheta\mathcal C(f_Y,f_V),
 \qquad
 \mathcal C(f_Y,f_V)=\frac12(\norm{f_Y}^2-\norm{f_V}^2)
 -\langle\mathsf T\Gamma^{-1}f_Y,f_V\rangle.
\]
The form \(\mathcal C\) is chosen so that, under an exact conditional sign
reversal, the pair \((f_Y,f_V)\) undergoes the rotation displayed below and
\(\mathcal C\) decreases by \(\norm{f_Y}^2-\norm{f_V}^2\).

\begin{lemma}[One-step energy decay]
\label{flux:lem:one-step-energy}
Let
\[
 \Lambda_\rho=\log(\e/\gamma_a)+(\rho/\gamma_a)^2,
 \qquad \vartheta=c_\star\rho/\Lambda_\rho
\]
for a sufficiently small universal \(c_\star\).  If
\(0<\rho\le1/2\), then
\[
 \tfrac12\norm f^2\le\mathcal E_\vartheta(f)
 \le\tfrac32\norm f^2,
 \qquad
 \mathcal E_\vartheta(\mathsf Kf)
 \le\left(1-c\frac\rho{\Lambda_\rho}\right)
       \mathcal E_\vartheta(f).
\]
\end{lemma}

\begin{proof}
Because \(\mathsf T\) commutes with \(\Gamma\), the block operator
associated with \(2\mathcal C\) has square
\(\operatorname{diag}(\Gamma^{-2},\Gamma^{-2})\).  Hence
\(\abs{\mathcal C}\le(2\gamma_a)^{-1}
(\norm{f_Y}^2+\norm{f_V}^2)\).  The displayed choice of \(\vartheta\)
satisfies \(\vartheta\le\gamma_a\), proving energy equivalence.

From \cref{flux:eq:ideal-operator} and contractivity of
\(\mathsf F_Q\),
\begin{equation}
 \norm{\mathsf Kf}^2-\norm f^2\le-\rho\norm{f_Q}^2.
 \label{flux:eq:refresh-dissipation}
\end{equation}
To see how the correction changes, first replace the conditional move by
an exact sign reversal and set
\(g=\mathsf U(\mathsf P-\mathsf Q)f\).  Unitarity of \(\mathsf U\) and
\(\mathsf C=\mathsf V\Gamma\) give
\[
 \binom{g_Y}{g_V}
 =\begin{pmatrix}\mathsf T&-\Gamma\\
                  \Gamma&\mathsf T\end{pmatrix}
  \binom{f_Y}{f_V},
 \qquad
 \mathcal C(g_Y,g_V)-\mathcal C(f_Y,f_V)
 =-\norm{f_Y}^2+\norm{f_V}^2.
\]

The actual conditional move differs from this sign reversal by
\[
 e_{\rm ht}=\mathsf V^*(I+\mathsf F_Q)f_Q,
 \qquad e_\rho=\rho f_V+(1-\rho)e_{\rm ht}.
\]
The two tracked components of the actual update are
\((\mathsf Kf)_Y=g_Y+\Gamma e_\rho\) and
\((\mathsf Kf)_V=g_V-\mathsf T e_\rho\).  Direct substitution gives
\[
 \mathcal C((\mathsf Kf)_Y,(\mathsf Kf)_V)-\mathcal C(f_Y,f_V)
 =-\norm{f_Y}^2+\norm{f_V}^2
  +\langle g_Y,\Gamma^{-1}e_\rho\rangle
  +\tfrac12\norm{e_\rho}^2.
\]
By \cref{flux:prop:global-halfturn},
\(\mathsf A_{\rm ht}^*\mathsf A_{\rm ht}
\preceq C\log(\e/\gamma_a)\Gamma^2\) for
\(\mathsf A_{\rm ht}=(I+\mathsf F_Q)\mathsf V\).  The corresponding
operator factorization gives a bounded \(B\), with
\(\norm{B}^2\le C\log(\e/\gamma_a)\), such that
\(\mathsf A_{\rm ht}=B\Gamma\).  Since
\(e_{\rm ht}=\mathsf A_{\rm ht}^*f_Q\), this gives
\(\norm{\Gamma^{-1}e_{\rm ht}}^2
\le C\log(\e/\gamma_a)\norm{f_Q}^2\).  Therefore
\(\norm{\Gamma^{-1}e_\rho}^2
\le C\Lambda_\rho\norm{f_Q}^2\), while
\(\norm{e_\rho}\le2\norm{f_Q}\).  Young's inequality now yields
\[
 \mathcal C((\mathsf Kf)_Y,(\mathsf Kf)_V)-\mathcal C(f_Y,f_V)
 \le-\tfrac12\norm{f_Y}^2+C\Lambda_\rho\norm{f_Q}^2.
\]
Multiply by \(\vartheta\), add
\cref{flux:eq:refresh-dissipation}, and choose \(c_\star\) small enough.
Energy equivalence then proves the claimed one-step decay.
\end{proof}

\begin{proof}[Proof of \cref{flux:prop:ideal-L2}]
By \cref{flux:eq:outer-spectrum},
\(\gamma_a\asymp\sqrt a\).  With
\(\rho\asymp\sqrt{a\log(\e/a)}\), we have
\(\Lambda_\rho\lesssim\log(\e/a)\) and
\(\rho/\Lambda_\rho\gtrsim\sqrt a/\sqrt{\log(\e/a)}\).
Iterating \cref{flux:lem:one-step-energy} and using energy equivalence
proves the proposition.
\end{proof}

\begin{proof}[Proof of \cref{flux:thm:ideal}]
Choose \(n_{\mathrm{blk}}\) as in the theorem, with a sufficiently large
universal constant.  \cref{flux:prop:ideal-L2} gives
\(\norm{\mathsf K^{n_{\mathrm{blk}}}}_{L^2_0\to L^2_0}^2\le \e^{-1}\).
If \(\varrho_\nu=\d\nu/\d\Pi_\eta\), then the density evolves by
\(\mathsf K^*\) and
\(\chi^2(\nu\Vert\Pi_\eta)=\norm{\varrho_\nu-1}^2\).  Since an operator and its
adjoint have the same norm, applying the block estimate \(J\) times proves
\cref{flux:eq:ideal-contraction}.
\end{proof}

\subsection{Implementable Algorithm and Query-Complexity Analysis}
\label{flux:app:implementation}

The ideal chain uses an exact conditional reference and an uncapped event
clock.  Its finite counterpart makes the following three replacements:
\begin{center}
\begin{tabular}{@{}lll@{}}
\toprule
source & ideal operation & finite operation\\
\midrule
reference & exact draw from \(q_y\) & certified approximate RGO\\
event clock & full state-dependent rate & capped Poisson thinning\\
center preparation & unrestricted solve & capped solver with early return\\
\bottomrule
\end{tabular}
\end{center}
The proposition is stated first; the rest of the subsection controls the
three corresponding disagreement events and counts their oracle queries.

\begin{proposition}[Finite compiler at a fixed step size]
\label{flux:prop:finite-compiler}
Let \(D_2(\nu_0\Vert\pi)\le\Delta\), where \(\Delta\ge1\), and let
\(0<\eps<1/4\).  For \(0<\eta\le1/\beta\), define
\[
 \ell=K\left\{\Delta+
 \log\left(\frac{Kd\kappa}{\eps\alpha\eta}\right)\right\}.
\]
There are universal constants \(c,K>0\) such that, if
\begin{equation}
 s\{\sqrt{d\ell}+\ell\}\le c,
 \label{flux:eq:finite-gate}
\end{equation}
the finite flux algorithm returns \(\widehat X\) with
\(\TV(\Law(\widehat X),\pi)\le\eps\) and
\begin{equation}
 \E N_{\mathrm{qry}}
 \le \frac K{\sqrt a}\sqrt{\log(\e/a)}
       \left(\Delta+\log\frac4\eps\right)
       \left\{\log(\e\kappa\ell)
                    +s(\sqrt{d\ell}+\ell)\right\}.
 \label{flux:eq:fixed-scale-cost}
\end{equation}
\end{proposition}

\subsubsection{Capped Poisson thinning}
\label{flux:app:poisson-thinning}

For fixed \((y,\widetilde x)\), put
\(x_{\rm ctr}=y-\eta\nabla V(\widetilde x)\) and
\(U=(X-x_{\rm ctr})/\sqrt\eta\).  Between events, \((U,P)\) follows a
linear Gaussian diffusion.  With
\(\omega=\sqrt{1-\gamma^2/4}\), its conditional mean after time \(t\) is
\[
 \e^{-\gamma t/2}
 \begin{pmatrix}
 \cos(\omega t)+\frac{\gamma}{2\omega}\sin(\omega t)
   &\frac{\sin(\omega t)}\omega\\
 -\frac{\sin(\omega t)}\omega
   &\cos(\omega t)-\frac{\gamma}{2\omega}\sin(\omega t)
 \end{pmatrix}
 \binom{U}{P},
\]
coordinatewise.  Its covariance is obtained by integrating the transported
Gaussian noise over this same linear motion.  Both mean and covariance are
explicit finite-dimensional operations.  Thus,
once \(\nabla V(\widetilde x)\) has been stored, no oracle query is needed
between candidate event times.

Given a cap \(\bar\lambda\), candidate times form a homogeneous Poisson
process of rate \(\bar\lambda\).  At a candidate time, one oracle query
computes
\[
 \lambda_t=\sqrt\eta
 [P_t^\top\{\nabla V(X_t)-\nabla V(\widetilde x)\}]_+,
\]
and the candidate is accepted with probability
\(\min\{1,\lambda_t/\bar\lambda\}\).  This is exact thinning for the
truncated rate \(\min\{\lambda_t,\bar\lambda\}\).

\begin{lemma}[Stationary rate tail and clock coupling]
\label{flux:lem:clock-coupling}
Conditionally on \(Y=y\), let \(X,\widetilde X\) be independent draws from
\(q_y\), and let \(P\sim\gamma_d\) be independent.  Then, for universal
\(c,C>0\) and every \(u\ge0\),
\[
 \Pr\left\{\sqrt\eta
 [P^\top\{\nabla V(X)-\nabla V(\widetilde X)\}]_+
 >Cs(\sqrt{du}+u)\right\}\le C\e^{-cu}.
\]
If, in addition, \(s\sqrt d\) is below a universal constant, then an
augmented initial law \(\nu\) satisfying
\(D_2(\nu\Vert\Pi_\eta)\le\Delta\) admits the following consequence: for
every \(\xi\in(0,1/2)\), the exact and capped half-turns can be coupled to
agree with probability at least \(1-\xi\), using
\[
 \bar\lambda_{\xi,\Delta}
 =Cs\{\sqrt{d\ell_{\xi,\Delta}}+\ell_{\xi,\Delta}\},
 \qquad
 \ell_{\xi,\Delta}=C\{\Delta+\log(1/\xi)\}.
\]
The capped half-turn uses exactly
\(T_\gamma\bar\lambda_{\xi,\Delta}\) candidate events in expectation.
\end{lemma}

\begin{proof}
The conditional law is
\((\alpha+\eta^{-1})\)-strongly log-concave.  Norm concentration and its
second-moment bound show that, outside an event of probability
\(2\e^{-u}\), both
\(\norm{X-\E X}\) and \(\norm{\widetilde X-\E X}\) are at most
\(C\sqrt{\eta(d+u)}\).
Smoothness then bounds the conditional standard deviation of the scalar
product with \(P\) by \(Cs(\sqrt d+\sqrt u)\).  A Gaussian tail bound
proves the displayed estimate.  Integrating it above
\(Cs(\sqrt{d\ell}+\ell)\) gives
\[
 \E(\lambda-\bar\lambda)_+
 \le Cs(\sqrt d+\ell+1)\e^{-c\ell}.
\]

Under \(\Pi_\eta\), invariance from
\cref{flux:prop:fixed-validity} makes this bound valid at every time.
Couple the exact and capped clocks with the same Poisson random measure;
they separate only if an excess-rate event occurs.  Taking
\(\ell=C\{\Delta+\log(1/\xi)\}\) makes the stationary separation
probability at most \(\e^{-\Delta}\xi^2\).  Under a \(D_2\)-warm initial
law, the path likelihood ratio is only its initial density ratio, since all
subsequent references, Gaussian noise, clocks, and marks use common
kernels.  Cauchy--Schwarz therefore raises the stationary failure
probability by at most \(\e^{\Delta/2}\), giving \(\xi\).  The final
expectation is that of a homogeneous Poisson clock on an interval of length
\(T_\gamma=O(1)\).
\end{proof}

\paragraph{Approximate conditional draws.}
We use the following existing RGO result.

\begin{proposition}[Approximate RGO]
\label{flux:prop:rgo}
There is a universal \(K_{\rm rgo}\) such that the following holds.  Given
\(y\), \(0<\delta<1/2\), and a point \(\widehat x\) satisfying
\[
 \norm{y-\eta\nabla V(\widehat x)-\widehat x}\le\sqrt{d\eta},
\]
if
\[
 \eta^{-1}\ge K_{\rm rgo}\beta
 \{\sqrt{d\log(1/\delta)}+\log(1/\delta)\},
\]
the rejection sampler of
\citet[Thms. 3.1 and D.1]{chen2026high} returns a law
\(\widehat q_y\) satisfying
\[
 \chi^2(\widehat q_y\Vert q_y)\le\delta^2,\qquad
 \chi^2(q_y\Vert\widehat q_y)\le\delta^2,
\]
using at most \(K_{\rm rgo}\) expected oracle queries after construction
of the displayed certificate.
\end{proposition}

The certificate itself uses only ordinary oracle queries.  The conditional
potential
\(V_y(x)=V(x)+\norm{x-y}^2/(2\eta)\) has condition number at most two.
Gradient descent with step
\(2/(\alpha+\beta+2\eta^{-1})\), initialized at the incoming position
\(x\), reaches and certifies the required residual in at most
\begin{equation}
 C\left[1+\log\left\{1+
 \frac{\norm{y-\eta\nabla V(x)-x}}{\sqrt{d\eta}}\right\}\right]
 \label{flux:eq:center-cost}
\end{equation}
oracle queries, by standard linear convergence for strongly convex smooth
optimization.

\paragraph{The implementable algorithm and error allocation.}
Use \(\rho\) from \cref{flux:eq:rho}, and set
\[
 n_{\mathrm{blk}}=\left\lceil\frac K{\sqrt a}
        \sqrt{\log(\e/a)}\right\rceil,\quad
 M=\left\lceil K\left(\Delta+\log\frac4\eps\right)\right\rceil
        n_{\mathrm{blk}},
 \quad \delta_{\rm row}=\frac{\eps}{6M}.
\]
Use the cap
\(\bar\lambda=Ks(\sqrt{d\ell}+\ell)\) and allow at most
\(Q_{\rm ctr}=\lceil K\log(\e\kappa\ell)\rceil\) center-solver queries per
row.  Initialize \(X_0\sim\nu_0\),
\(Y_0\mid X_0\sim\cN(X_0,\eta I_d)\).  Each row reflects \(Y\), runs the
capped center solver, draws an approximate reference from
\cref{flux:prop:rgo} with tolerance \(\delta_{\rm row}\), and then uses
that draw either as the independent refresh or as the reference of the
capped flux half-turn.  If the center solver does not certify its residual
within its cap, the algorithm stops and returns the current position.  This
failure rule, together with the almost-sure finiteness of the rejection
sampler and Poisson clock, makes the procedure almost surely finite from
every initial state.

\begin{proof}[Proof of \cref{flux:prop:finite-compiler}]
Augmentation preserves warmness:
\[
 \frac{\d\Law(X_0,Y_0)}{\d\Pi_\eta}(x,y)
 =\frac{\d\nu_0}{\d\pi}(x),\qquad
 D_2(\Law(X_0,Y_0)\Vert\Pi_\eta)\le\Delta.
\]
\cref{flux:thm:ideal} and the choice of \(M\) put the ideal output within
\(\eps/2\) of \(\pi\).  Because every ideal row preserves \(\Pi_\eta\),
data processing keeps its intermediate \(D_2\) divergence at most
\(\Delta\).

It remains to couple every finite row to its ideal counterpart.  There are
three possible disagreements.

\emph{Center cap.}
Under stationarity, write \(Y=X+\sqrt\eta Z\).  After reflection, the
normalized residual at the incoming position is
\[
 \frac{\norm{\sqrt\eta Z+\eta\nabla V(X)}}{\sqrt{d\eta}}.
\]
Strong log-concavity, Gaussian concentration, and
\(\norm{\nabla V(X)}\le\beta\norm{X-x^\star}\), where \(x^\star\) minimizes
\(V\), show that it exceeds
\(2\sqrt\kappa(1+\sqrt{2u/d})\) with probability at most \(2\e^{-u}\).
Taking
\(u=\Delta+2\log(1/\delta_{\rm row})+\log2\) and using
\cref{flux:eq:center-cost} shows that the stationary cap-failure
probability is at most \(\e^{-\Delta}\delta_{\rm row}^2\).  The warm-path
Cauchy--Schwarz argument used in \cref{flux:lem:clock-coupling} makes it at
most \(\delta_{\rm row}\) under the current ideal law.

\emph{Approximate reference.}
The definitions imply
\(\log(1/\delta_{\rm row})\le C\ell\).  The gate
\cref{flux:eq:finite-gate} therefore implies the small-step-size condition in
\cref{flux:prop:rgo}.  Either directed \(\chi^2\) bound implies total
variation at most \(\delta_{\rm row}/2\), so maximal coupling matches the
approximate and exact references except with probability at most
\(\delta_{\rm row}\).  Because all laws have Borel densities on
\(\R^d\), these maximal couplings may be selected jointly measurably in
\(y\).  On agreement, the same draw is used in both chains.

\emph{Capped clock.}
The same parameter bound gives
\(\ell_{\delta_{\rm row},\Delta}\le C\ell\).  After increasing \(K\),
\cref{flux:lem:clock-coupling}, applied conditionally on the common
pre-row history, couples the exact and capped half-turns with failure
probability at most \(\delta_{\rm row}\).

A union bound over three events and \(M\) rows is at most \(\eps/2\).
Together with the ideal mixing error, this proves the total-variation
claim.

Early termination only reduces work.  Each row uses at most
\(Q_{\rm ctr}\) center queries, a universal expected number of RGO
queries, one query to cache \(\nabla V(\widetilde X)\) on the half-turn branch, and
\(T_\gamma\bar\lambda\) expected event-test queries.  Hence
\[
 \E N_{\mathrm{qry}}
 \le CM\{Q_{\rm ctr}+s(\sqrt{d\ell}+\ell)\},
\]
which is \cref{flux:eq:fixed-scale-cost}.
\end{proof}

\begin{proof}[Proof of \cref{flux:thm:complexity}]
Take \(\eta\) and \(L\) as in \cref{flux:eq:eta-final}, with the universal
constant \(c\) sufficiently small.  Then
\[
 a=\frac{c}{\kappa(\sqrt{dL}+L)},\qquad
 s=\frac{c}{\sqrt{dL}+L},\qquad \ell\le CL,
\]
so \cref{flux:eq:finite-gate} holds.  Moreover,
\[
 \frac1{\sqrt a}\sqrt{\log(\e/a)}
 \le C\sqrt\kappa\,(d+L)^{1/4}L^{3/4},
\]
and the braces in \cref{flux:eq:fixed-scale-cost} are at most
\(C\log(\e\kappa L)\).  Substitution proves
\cref{flux:eq:complexity} and the total-variation conclusion.
\end{proof}

For the warm-start removal argument we also record the dependence on a
nonconstant divergence budget.

\begin{corollary}[General warmness budget]
\label{flux:cor:general-warmness}
Let \(B\ge1\), \(D_2(\nu_0\Vert\pi)\le B\), and set
\[
 L_B=B+\log\frac{Cd\kappa}{\eps},\qquad
 A_B=\sqrt{dL_B}+L_B,\qquad
 \eta=\frac{c}{\beta A_B}.
\]
Then the flux algorithm returns an \(\eps\)-TV sample using
\[
 \E N_{\mathrm{qry}}
 \le C\sqrt\kappa A_B^{1/2}\sqrt{\log(\e\kappa A_B)}
 \left(B+\log\frac4\eps\right)\log(\e\kappa L_B)
\]
expected oracle queries.  Ignoring logarithmic factors, this is
\(\widetilde O(\sqrt\kappa B\{\sqrt{dB}+B\}^{1/2})\), and hence
\(\widetilde O(\sqrt\kappa d^{1/4}B^{5/4})\) when \(1\le B\le d\).
\end{corollary}

\begin{proof}
Apply \cref{flux:prop:finite-compiler} with \(\Delta=B\).  The displayed
choice gives \(a=c/(\kappa A_B)\), \(s=c/A_B\), and
\(\ell\le CL_B\), so the finite gate holds.  Substitution into
\cref{flux:eq:fixed-scale-cost} gives the stated bound.
\end{proof}

\subsubsection{Regularity and Domain Extensions}
\label{flux:app:regularity}

For completeness, we collect here the regularity steps used above.  The
weak derivatives and integrations by parts are first justified for smooth
potentials and test functions.  Mollifying \(V\) preserves strong
convexity and smoothness; the normalizers and conditional moments have a
common Gaussian envelope, so the identities pass to the limit.

The derivative estimate in \cref{flux:lem:outer-geometry} also begins on
smooth functions.  Closedness of the weak gradient extends
\(\nabla\mathsf T\) by continuity.  On the spectral subspace \([1/2,1]\),
\(\mathsf T^{-1}\) is bounded.  Applying the extended estimate to
\(\mathsf T^{-1}f\) and \(\mathsf T^{-1}(I-\mathsf T)f\) therefore gives
approximations converging both in \(L^2\) and in weak-gradient norm; in
particular, no nonsmooth spectral projection is applied directly to a
classical derivative.  The same graph-norm approximation extends
\cref{flux:lem:affine-after-reflection} from smooth functions to its stated
weak-gradient domain.

The event-driven process itself needs no smoothing: its residual force is
globally Lipschitz, its integrated hazard is continuous, and the direction
\(h/\norm h\) is evaluated only at positive event rate.

The proof is complete relative to the earlier RGO result in
\cref{flux:prop:rgo} and standard strongly log-concave concentration and
optimization facts.

\section{Proofs for the Persistent Quarter-Turn}
\label{app:persistent-full}

This appendix proves
\cref{prop:persistent-main-exactness,prop:persistent-main-rotation,thm:persistent-contraction,thm:persistent-finite,cor:persistent-near-gaussian}.
The argument is organized in the dependency order of the proof:

\begin{enumerate}[leftmargin=*,label=\arabic*.]
\item show that every fixed-reference conditional motion is exact;
\item quantify how outer reflection creates a conditional-position signal;
\item prove that a quarter-turn transfers this signal to and from momentum;
\item combine this transfer with occasional refreshment to obtain
      \(L^2\) contraction;
\item replace the ideal reference, event clock, and refresh by finite
      counterparts and account for their errors;
\item absorb the known strong-convexity quadratic into the conditional
      motion to obtain the near-Gaussian refinement.
\end{enumerate}

Standard ingredients are cited rather than reproved.  In particular, we
use established skew-balance results for velocity-jump processes, standard
Poincar\'e and log-Sobolev consequences of strong log-concavity, and the
existing first-order RGO implementation.  The new arguments are the
position--momentum transfer estimate, its unit-circle contraction analysis,
and the shifted-reference implementation.

Throughout, \(\pi\) denotes the target.  Uppercase \(X,Y,P,E\) denote
position, auxiliary position, momentum, and reference random variables,
with \(E\) corresponding to \(\widetilde X\) in the main text;
lowercase letters denote their values.  Operators act on observables, so
laws evolve through their adjoints.  Concretely, for a Markov transition
\(\mathsf A\),
\((\mathsf Af)(z)=\E[f(Z_{\rm next})\mid Z_{\rm current}=z]\); thus operator
products act from right to left on functions, even though the corresponding
state updates occur in chronological order.  The small set of operators used
throughout is:

\begin{center}
\begin{tabular}{@{}cl@{}}
\toprule
symbol & meaning\\
\midrule
\(\mathsf P\) & conditional expectation given \(Y\)\\
\(\mathsf Q=I-\mathsf P\) & remaining conditional variation\\
\(\mathsf U\) & outer reflection\\
\(\mathsf R\) & momentum reversal, \((x,y,p)\mapsto(x,y,-p)\)\\
\(\mathsf S\) & conditional quarter-turn\\
\(\mathsf K\) & one complete ideal iteration\\
\(\mathsf T=\mathsf P\mathsf U\mathsf P\) & reflection seen on \(Y\)-only observables\\
\(\Gamma=(I-\mathsf T^2)^{1/2}\) & amount transferred out of the \(Y\)-only part\\
\bottomrule
\end{tabular}
\end{center}

\subsection{Exactness of the Persistent Conditional Move}
\label{persistent:sec:exactness}

We first isolate the only facts about the event-driven dynamics that are
used later: it is well defined, it preserves the desired conditional law,
and reversing momentum gives its stationary time reversal.  These are
standard skew-balance properties of bouncy and Boomerang processes
\citep{bouchardcote2018bouncy,vanetti2017piecewise,bierkens2020boomerang,
monmarche2022exact}.  We verify that our conditional process satisfies the
required identities.

Fix \(Y=y\), use the normalized position \(w=(x-y)/\sqrt\eta\), and write
\[
 q_y^w(\d w)\propto
 \exp\{-V(y+\sqrt\eta w)-\norm w^2/2\}\,\d w,
 \qquad
 b_y(w)=\sqrt\eta\,\nabla V(y+\sqrt\eta w).
\]
For a nonzero vector \(h\), let
\(\operatorname{Ref}_h p=p-2(p^\top h)h/\norm h^2\), and set
\(\operatorname{Ref}_0p=p\).  A physical reference \(E\) corresponds to
\(e=(E-y)/\sqrt\eta\).  Conditional on \((y,e)\), the generator is
\begin{align}
 (\mathcal L_{y,e}F)(w,p)
 ={}&p^\top\nabla_wF-\{w+b_y(e)\}^\top\nabla_pF \notag\\
 &+[p^\top\{b_y(w)-b_y(e)\}]_+
 \{F(w,\operatorname{Ref}_{b_y(w)-b_y(e)}p)-F(w,p)\}.
 \label{eq:persistent-generator}
\end{align}
The first line is a deterministic harmonic flight centered at
\(-b_y(e)\).  The last line says that a state-dependent Poisson clock rings
at rate \([p^\top\{b_y(w)-b_y(e)\}]_+\); when it rings, only the momentum is
reflected across the residual-gradient direction
\(b_y(w)-b_y(e)\).  Fixing \(e\) is a proof device.  The accelerated branch
later draws a fresh reference independently given \(y\), and averages these
fixed-reference kernels.

\begin{lemma}[Fixed-reference validity]
\label{persistent:lem:fixed-invariance}
For every fixed \(y,e\), \cref{eq:persistent-generator} defines a
non-explosive Markov process.  If \(S_t^{y,e}\) denotes its transition
operator and \(\mathsf R(w,p)=(w,-p)\), then
\[
 q_y^w\otimes\gamma_d\text{ is invariant},
 \qquad (S_t^{y,e})^*=\mathsf R S_t^{y,e}\mathsf R .
\]
Both conclusions remain valid after averaging \(S_t^{y,e}\) over any
measurable reference law depending on \(y\), provided the reference is
conditionally independent of the incoming \((W,P)\).
\end{lemma}

\begin{proof}
Put \(a=b_y(e)\), \(r(w)=b_y(w)-a\), and
\[
 \mathsf J(w,p)=(w,\operatorname{Ref}_{r(w)}p),
 \qquad s(w,p)=p^\top r(w).
\]
Between events the process follows the harmonic flow
\[
 \phi_u(w,p)=\bigl(-a+(w+a)\cos u+p\sin u,
                    -(w+a)\sin u+p\cos u\bigr).
\]
Both this flow and \(\mathsf J\) preserve
\(H(w,p)=\{\norm{w+a}^2+\norm p^2\}/2\).  Thus every set
\(K_R=\{H\le R\}\) is invariant.  Lipschitz continuity of \(b_y\) bounds
the event rate \(\lambda=[s]_+\) on \(K_R\) by
\[
 \norm p\,\norm{b_y(w)-b_y(e)}
 \le\beta\eta\sqrt{2R}\{\sqrt{2R}+\norm{e+a}\}
 =:\Lambda_R.
\]
Thinning a homogeneous Poisson process of rate \(\Lambda_R\) therefore
constructs the process from every point of \(K_R\).  Since every initial
state lies in some \(K_R\), this proves well-posedness and non-explosion.

We prove skew balance directly on paths.  The idea is to pair every path with
its time reversal after flipping momentum: collision factors agree one by
one, while the change in target density exactly cancels the change in the
survival probability.  Write \(z=(w,p)\), set
\[
 \Psi(w)=V(y+\sqrt\eta w)-a^\top w,
 \qquad \nu=q_y^w\otimes\gamma_d,
\]
so that the Lebesgue density of \(\nu\) is, up to normalization,
\(\varrho(z)=\exp\{-H(z)-\Psi(w)\}\).  Along a flight,
\(\frac{\d}{\d u}\Psi(w_u)=s(z_u)\), while a collision changes neither
\(H\) nor \(\Psi\).  Moreover,
\begin{equation}
 \mathsf R\phi_u\mathsf R=\phi_{-u},\qquad
 \mathsf J\mathsf R=\mathsf R\mathsf J,\qquad
 s\circ\mathsf J=s\circ\mathsf R=-s.
 \label{persistent:eq:path-reversal-identities}
\end{equation}

Fix \(t>0\).  For
\(\boldsymbol t=(t_1,\ldots,t_n)\) in the ordered simplex
\(\Delta_n(t)=\{0<t_1<\cdots<t_n<t\}\), let
\(z_u=z_u(z_0,\boldsymbol t)\) be the deterministic skeleton obtained by
following \(\phi\) and inserting a collision \(\mathsf J\) at every
\(t_i\).  On \(K_R\), the iterated-hazard expansion of the transition
operator is
\begin{equation}
 S_t^{y,e}F(z_0)
 =\sum_{n\ge0}\int_{\Delta_n(t)}
   \exp\left\{-\int_0^t\lambda(z_u)\,\d u\right\}
   \prod_{i=1}^n\lambda(z_{t_i-})F(z_t)\,\d\boldsymbol t .
 \label{persistent:eq:ordered-event-expansion}
\end{equation}
The bound by \(\Lambda_R\) makes this series absolutely convergent.

Pair this skeleton with the one starting at \(\mathsf Rz_t\) and having
ordered collision times
\((t-t_n,\ldots,t-t_1)\).  By
\cref{persistent:eq:path-reversal-identities}, it is the momentum-reversed
time reversal of the original skeleton and ends at \(\mathsf Rz_0\).  The
map
\[
 (z_0,t_1,\ldots,t_n)
 \longmapsto(\mathsf Rz_t,t-t_n,\ldots,t-t_1)
\]
preserves Lebesgue measure.  Indeed, \(\phi_u\) and \(\mathsf R\) preserve
volume, while \(\mathsf J\) does so by applying, for every fixed \(w\), an
orthogonal change of variables in \(p\).  This argument does not require
continuity of the reflection direction at \(r(w)=0\), where the event rate
vanishes.

If \(z^-\) is a forward pre-collision state, the corresponding reverse
pre-collision state is \(\mathsf R\mathsf Jz^-\); hence
\[
 \lambda(\mathsf R\mathsf Jz^-)=\lambda(z^-).
\]
Thus the collision-rate products in
\cref{persistent:eq:ordered-event-expansion} agree.  The reverse survival
hazard is \(\int_0^t[-s(z_u)]_+\,\d u\), and
\[
 \Psi(w_t)-\Psi(w_0)
 =\int_0^t s(z_u)\,\d u
 =\int_0^t\bigl\{[s(z_u)]_+-[-s(z_u)]_+\bigr\}\,\d u.
\]
Since \(H(z_t)=H(z_0)\) and \(\varrho(\mathsf Rz)=\varrho(z)\), this is
exactly the survival-weight identity
\begin{equation}
 \varrho(z_0)\exp\left\{-\int_0^t[s(z_u)]_+\,\d u\right\}
 =\varrho(\mathsf Rz_t)
  \exp\left\{-\int_0^t[-s(z_u)]_+\,\d u\right\}.
 \label{persistent:eq:path-weight-balance}
\end{equation}

Let \(\nu_R\) be \(\nu\) conditioned on \(K_R\).  Pairing paths term by
term in \cref{persistent:eq:ordered-event-expansion} using
\cref{persistent:eq:path-weight-balance} gives, for bounded measurable
\(F,G\),
\[
 \int F S_t^{y,e}G\,\d\nu_R
 =\int (\mathsf R G)S_t^{y,e}(\mathsf R F)\,\d\nu_R
 =\int (\mathsf R S_t^{y,e}\mathsf R F)G\,\d\nu_R.
\]
Letting \(R\to\infty\) gives the same identity under \(\nu\).  Taking
\(F=1\) proves invariance.  Jensen's inequality then makes
\(S_t^{y,e}\) a contraction on \(L^2(\nu)\), and density of bounded
functions extends the identity to \(L^2(\nu)\), proving
\((S_t^{y,e})^*=\mathsf R S_t^{y,e}\mathsf R\).

Finally, the recursive hazard construction is jointly Borel in
\((y,e,z_0)\): the flow, rate, integrated hazard, and reflection are Borel,
and successive event times are obtained by taking level-hitting times of
the integrated hazard.  Non-explosion gives the measurable pointwise
limit.  Consequently, conditioning on an incoming-state-independent
reference and applying Fubini passes both invariance and the adjoint
identity to every measurable reference mixture.
\end{proof}

We now pass from the conditional move to one complete proximal iteration.
All operators below act on \(L^2(\widehat\Pi_\eta)\), where
\(\widehat\Pi_\eta=\Pi_\eta\otimes\gamma_d\).  Let \(\mathsf U\) be the
pullback of the outer reflection
\((x,y,p)\mapsto(x,2x-y,p)\), let
\(\mathsf P f=\E[f\mid Y]\), and set \(\mathsf Q=I-\mathsf P\).  Finally,
for the ideal chain let \(\mathsf S\) be the quarter-period conditional
kernel after drawing \(E\sim q_y\), independently of the incoming state.
Then \(\mathsf S\mathsf P=\mathsf P\mathsf S=\mathsf P\).  The
conditional mixture and the observable operator of one ideal iteration are
\begin{equation}
 \mathsf B_\rho
 =\rho\mathsf P+(1-\rho)\mathsf S
 =\mathsf P+(1-\rho)\mathsf S\mathsf Q,
 \qquad
 \mathsf K=\mathsf U\mathsf B_\rho.
 \label{persistent:eq:K-def}
\end{equation}
The first term is the independent conditional refresh: as an operator on
observables, an exact draw from \(q_y\otimes\gamma_d\) is precisely
\(\mathsf P\).  The second term is the persistent quarter-turn.

\begin{lemma}[Invariance and refresh dissipation]
\label{persistent:lem:row-identities}
The operator \(\mathsf K\) preserves \(\widehat\Pi_\eta\) and is a
contraction.  Moreover, for every \(f\in L^2(\widehat\Pi_\eta)\),
\begin{equation}
 \norm f^2-\norm{\mathsf Kf}^2
 \ge \rho\norm{\mathsf Qf}^2.
 \label{persistent:eq:dissipation}
\end{equation}
\end{lemma}

\begin{proof}
The outer reflection is a measure-preserving involution because it sends
\(Y-X\) to its negative.  \cref{persistent:lem:fixed-invariance} shows that
the conditional mixture preserves the conditional position--momentum law;
hence their composition preserves \(\widehat\Pi_\eta\).  Orthogonality of
\(\mathsf P f\) and \(\mathsf Q f\), together with contractivity of
\(\mathsf S\), gives
\[
 \norm f^2-\norm{\mathsf Kf}^2
 =\norm{\mathsf Qf}^2-(1-\rho)^2\norm{\mathsf S\mathsf Qf}^2
 \ge\{1-(1-\rho)^2\}\norm{\mathsf Qf}^2,
\]
which is stronger than \cref{persistent:eq:dissipation}.
\end{proof}

\subsection{How Reflection Feeds the Quarter-Turn}
\label{persistent:sec:graph}

This subsection proves the sampling-specific estimate used by the
contraction argument.  Reflection first turns a function of \(Y\) into an
almost affine function of the conditional position.  The quarter-turn then
sends that affine position component into momentum and sends the
corresponding momentum component back with the opposite sign.

At a semantic level, the signal flow is
\[
 \begin{aligned}
 f(Y_+)\xrightarrow{\ \mathsf U\ }f(Y_-)
   &=\mathsf T f+\mathsf C f,
     &&\mathsf C:=\mathsf Q\mathsf U\mathsf P,\\
 \mathsf C f&\approx L_f,
     &\mathsf S\mathsf C f&\approx M_f,\\
 \mathsf S M_f&\approx-\mathsf C f,
     &\mathsf U M_f&\approx M_f.
 \end{aligned}
\]
The proposition below controls the last three errors simultaneously; the
intermediate affine observable \(L_f\) makes that control possible.
Refreshment then damps whatever is not represented by these signals.

\paragraph{Notation and the result to be proved.}
Let
\(\mathsf P f=\E[f\mid Y]\), \(\mathsf Q=I-\mathsf P\), and let
\(\mathsf U\) denote outer reflection.  On the centered \(Y\)-only
subspace \(\mathcal H_Y=\operatorname{range}(\mathsf P)\cap L^2_0\), define
\[
 \mathsf T=\mathsf P\mathsf U\mathsf P,\qquad
 \Gamma=(I-\mathsf T^2)^{1/2},\qquad
 g=\sqrt{\alpha\eta}.
\]
Under the stationary reflection coupling, write
\(Y_+=Y\), \(Y_-=2X-Y\), and \(W=(X-Y_+)/\sqrt\eta\).  Conditional on
\(Y_+=y\), denote the mean and covariance of \(W\) by
\(\bar w_y\) and \(\Sigma_y\), and set \(D_y=I-\Sigma_y\).
For a differentiable \(f\in\mathcal H_Y\), define
\[
 a_f(y)=\E[\nabla f(Y_-)\mid Y_+=y],\quad
 L_f=2\sqrt\eta\,a_f(Y_+)^\top(W-\bar w_{Y_+}),\quad
 M_f=2\sqrt\eta\,a_f(Y_+)^\top P.
\]
The letter \(L_f\) denotes the affine position observable; \(M_f\) is its
momentum counterpart.  Componentwise, the definition also says
\(a_f=\mathsf T\nabla f\), where \(\mathsf T\) acts coordinatewise on
vector-valued functions.  The roles of the main quantities are summarized
as follows.

\begin{center}
\begin{tabular}{@{}lp{0.67\linewidth}@{}}
\toprule
quantity & role in the transfer argument\\
\midrule
\(\mathsf T f\) & the part of the reflected observable predictable from
\(Y_+\) alone\\
\(\Gamma f\) & the scale of the conditional-position signal created by
reflection\\
\(\mathsf C f\) & the conditional-position signal itself\\
\(L_f\) & the leading affine approximation to that position signal\\
\(M_f\) & the momentum copy of \(L_f\) produced by a quarter-turn\\
\bottomrule
\end{tabular}
\end{center}

These formulas are initially classical for differentiable \(f\).  The
proposition below is an all--\(L^2\) statement: the closed weak-gradient
lemma that follows proves that every \(f\in\mathcal H_+\) has the required
weak gradient and defines \(a_f,M_f\) by graph closure.  The auxiliary
observable \(L_f\) is needed only in the smooth-core calculation.  In
particular, the proposition does not assume that a spectral projection
preserves smoothness.

\begin{proposition}[Position--momentum transfer]
\label{persistent:prop:three-side}
Let \(\mathcal H_+\) be the spectral subspace of \(\mathsf T\)
corresponding to \([1/2,1]\).  If \(\beta\eta\) is smaller than a universal
constant, then every \(f\in\mathcal H_+\) satisfies
\begin{align}
 &\norm{\mathsf S(\mathsf Q\mathsf U\mathsf P f)-M_f}^2
 +\norm{\mathsf S M_f+\mathsf Q\mathsf U\mathsf P f}^2
 +\norm{\mathsf U M_f-M_f}^2 \notag\\
 &\hspace{35mm}\le
 C\norm{\Gamma f}\norm{\Gamma^3f}.
 \label{persistent:eq:three-side}
\end{align}
\end{proposition}

The proof has a technical prelude followed by three substantive steps.  The
prelude extends estimates from smooth observables to the closed weak-gradient
domain.  Step~1 analyzes the signal created by reflection, Step~2 analyzes
its quarter-rotation, and Step~3 combines the two estimates and performs the
single closure passage.

\paragraph{Technical prelude: extending smooth estimates to all \(L^2\).}
We first isolate the only closure argument needed below.  A graph core here
means that every function with a weak gradient can be approximated together
with its gradient by smooth compactly supported functions.  Let
\(\nu=\Pi_\eta^Y\), let
\(D:\operatorname{Dom}(D)\subset L^2_0(\nu)
\to L^2(\nu;\R^d)\) be the closed weak gradient, and put
\[
 \mathcal C_0=C_c^\infty(\R^d)\cap L^2_0(\nu).
\]

\begin{lemma}[Closed weak-gradient smoothing]
\label{persistent:lem:closed-gradient-smoothing}
The set \(\mathcal C_0\) is a graph core for \(D\).  Suppose that
\(\mathsf T\) is a self-adjoint contraction,
\(\mathsf T\mathcal C_0\subset\operatorname{Dom}(D)\), and, on
\(\mathcal C_0\),
\begin{equation}
 \sqrt\eta\norm{D\mathsf T u}\le c_D\norm{\Gamma u}.
 \label{persistent:eq:core-smoothing}
\end{equation}
Then \(D\mathsf T\) has an everywhere-defined bounded extension
\(\mathsf R_D:L^2_0(\nu)\to L^2(\nu;\R^d)\), and in fact
\[
 \mathsf T L^2_0(\nu)\subset\operatorname{Dom}(D),\qquad
 D\mathsf T=\mathsf R_D,\qquad
 \sqrt\eta\norm{\mathsf R_Du}\le c_D\norm{\Gamma u}.
\]
Consequently, if
\(\Pi_+=\mathbf1_{[1/2,1]}(\mathsf T)\), then
\(\operatorname{range}(\Pi_+)\subset\operatorname{Dom}(D)\) and
\begin{align*}
 \sqrt\eta\norm{Df}&\le2c_D\norm{\Gamma f},
 \\
 \sqrt\eta\norm{D(I-\mathsf T)f}
 &\le2c_D\norm{\Gamma^3f},
 \qquad f\in\operatorname{range}(\Pi_+).
\end{align*}
Moreover, for every \(f\in\operatorname{Dom}(D)\),
\begin{equation}
 D(I-\mathsf T)f=Df-\mathsf R_Df.
 \label{persistent:eq:graph-difference}
\end{equation}
Thus convergence in the graph norm of \(D\) automatically implies
convergence of \(D(I-\mathsf T)f\); no additional graph norm is needed.
\end{lemma}

\begin{proof}
The density of \(\nu\) is positive and continuous, hence is bounded above
and below on every compact set.  Cutoff followed by ordinary convolution
therefore approximates every element of \(\operatorname{Dom}(D)\) in the
norm \(\norm f^2+\norm{Df}^2\).  To preserve centering, subtract the mean
of each approximant times a fixed \(\psi\in C_c^\infty\) satisfying
\(\int\psi\,\d\nu=1\).  This proves the graph-core assertion.

If \(u_n\in\mathcal C_0\) converges in \(L^2(\nu)\),
\cref{persistent:eq:core-smoothing} makes \(D\mathsf T u_n\) Cauchy,
while \(\mathsf T u_n\) converges.  Closedness of \(D\) shows that the
limit belongs to \(\operatorname{Dom}(D)\), identifies its gradient, and
gives the bounded extension \(\mathsf R_D=D\mathsf T\).

On \(\operatorname{range}(\Pi_+)\), \(\mathsf T^{-1}\) is bounded with
norm at most two.  For \(f\) in that subspace,
\[
 Df=\mathsf R_D\mathsf T^{-1}f,
 \qquad
 D(I-\mathsf T)f
 =\mathsf R_D\mathsf T^{-1}(I-\mathsf T)f.
\]
The first estimate follows from
\(\gamma_t/t\le2\gamma_t\) on \(t\in[1/2,1]\), where
\(\gamma_t=(1-t^2)^{1/2}\).  The second follows from
\[
 \frac{1-t}{t}=\frac{1-t^2}{t(1+t)}\le2(1-t^2),
 \qquad t\in[1/2,1].
\]
Finally, \(\mathsf T f\in\operatorname{Dom}(D)\) for every
\(f\in L^2_0(\nu)\),
so \cref{persistent:eq:graph-difference} follows by linearity and proves
the last assertion.
\end{proof}

\subsubsection{Step 1: What the Outer Reflection Creates}

\begin{lemma}[Outer-reflection geometry]
\label{persistent:lem:outer-geometry}
Assume \(\beta\eta\le1/3\).  On \(\mathcal H_Y\),
\[
 -\tfrac12 I\preceq\mathsf T
 \preceq\frac{1-g^2}{1+g^2}I,\qquad
 \Gamma\succeq\frac{2g}{1+g^2}I.
\]
Moreover, \(\mathsf C:=\mathsf Q\mathsf U\mathsf P\) can be written as
\(\mathsf C=\mathsf V\Gamma\), where \(\mathsf V\) is an isometry, and
\[
 \mathsf U\mathsf P f=\mathsf P\mathsf T f+\mathsf V\Gamma f,\qquad
 \mathsf U\mathsf V f=\mathsf P\Gamma f-\mathsf V\mathsf T f.
\]
On \(\mathcal H_+\),
\[
 \sqrt\eta\norm{\nabla f}\le C\norm{\Gamma f},\qquad
 \sqrt\eta\norm{\nabla(I-\mathsf T)f}
 \le C\norm{\Gamma^3f}.
\]
\end{lemma}

\begin{proof}
The pair \((Y_+,Y_-)\) is exchangeable and
\((\mathsf Tf)(y)=\E[f(Y_-)\mid Y_+=y]\), so \(\mathsf T\) is a
self-adjoint contraction.  The conditional density of \(Y_-\) given
\(Y_+=y\) is proportional to
\[
 \exp\left\{-V\left(\frac{y+u}{2}\right)
             -\frac{\norm{y-u}^2}{8\eta}\right\}.
\]
Its score with respect to \(y\) is
\(s_y(u)=-\nabla V((y+u)/2)/2-(y-u)/(4\eta)\).  For
\(f\in\mathcal C_0\), weak differentiation under the integral gives
\[
 D\mathsf Tf(y)
 =\Cov(f(Y_-),s_y(Y_-)\mid Y_+=y).
\]
Indeed, on compact sets of \(y\), the score times the conditional density
is dominated by an integrable Gaussian tail.  Thus Fubini and weak
integration by parts justify differentiation of both the numerator and
normalizing constant.  In particular,
\(\mathsf T\mathcal C_0\subset\operatorname{Dom}(D)\).
Applying conditional Poincar\'e to this covariance gives
\[
 \eta\norm{D\mathsf Tf}^2
 \le\frac{(1-\alpha\eta)^2}{4(1+\alpha\eta)}
       \{\norm f^2-\norm{\mathsf Tf}^2\}
 \le\frac14\norm{\Gamma f}^2.
\]
The \(Y\)-marginal has Poincar\'e constant at most
\(\alpha^{-1}+\eta\): this follows from the conditional covariance bound
\(\Cov(X\mid Y)\preceq\eta(1+\alpha\eta)^{-1}I\) and
Brascamp--Lieb \citep{chewi2023log}.  Equivalently, if \(W_\eta\) is the
negative log-density of the \(Y\)-marginal, then
\[
 \nabla^2W_\eta(y)
 =\eta^{-1}I-\eta^{-2}\Cov(X\mid Y=y)
 \succeq\frac{\alpha}{1+\alpha\eta}I.
\]
Applying the resulting Poincar\'e inequality to the
centered function \(\mathsf Tf\) yields
\(\norm{\mathsf T}_{L^2_0}\le(1-g^2)/(1+g^2)\).
This inequality is first obtained on \(\mathcal C_0\) and extends to
\(L^2_0(\nu)\) by density.

For the lower edge, exchangeability and Gaussian Poincar\'e imply
\[
 \ip h{(I-\mathsf T)h}
 =\tfrac12\E\{h(Y_+)-h(Y_-)\}^2
 \le2\eta\norm{\nabla h}^2.
\]
Take \(h=\mathsf Tf\) for \(f\in\mathcal C_0\) and use the preceding
derivative bound.  The resulting inequality involves only bounded
polynomials in \(\mathsf T\), so density extends it to all of
\(L^2_0(\nu)\).  Spectral calculus gives
\(2\mathsf T^2(I-\mathsf T)\preceq I-\mathsf T^2\), equivalently
\((I-\mathsf T)^2(I+2\mathsf T)\succeq0\), and hence
\(\mathsf T\succeq-I/2\).  The two spectral edges imply the lower bound
on \(\Gamma\).  Since \(\mathsf T1=1\), the lower bound
\(\mathsf T\succeq-I/2\) also holds on the full scalar \(L^2(\nu)\), and
hence coordinatewise on vector fields such as \(Df\), whose components
need not be centered.

Finally,
\(\mathsf C^*\mathsf C=I-\mathsf T^2=\Gamma^2\).  Thus
\(\mathsf V=\mathsf C\Gamma^{-1}\) is an isometry, and
\(\mathsf U^2=I\) gives the two block identities.

The core derivative estimate above is
\cref{persistent:eq:core-smoothing} with \(c_D=1/2\).  Applying
\cref{persistent:lem:closed-gradient-smoothing} proves that all of
\(\mathcal H_+\), not merely a smooth subset of it, belongs to the weak
gradient domain and gives the two final inequalities.  Notice that this
argument never applies the discontinuous spectral projection \(\Pi_+\) to
a smooth function.
\end{proof}

\begin{lemma}[Affine approximation after reflection]
\label{persistent:lem:affine-control}
For every \(f\in\mathcal C_0\),
\(\mathsf Q\mathsf U\mathsf P f=L_f+e_f\), where
\(\E[e_f\mid Y_+]=0\).  With
\[
 \mathcal R_f=\norm{\nabla f}^2-\norm{\mathsf T\nabla f}^2,\qquad
 \mathcal D_f=\E[a_f(Y_+)^\top D_{Y_+}a_f(Y_+)],
\]
one has
\begin{equation}
 \norm{e_f}^2\le4\eta\mathcal R_f,\qquad
 \mathcal R_f+\mathcal D_f
 \le C\norm{\nabla(I-\mathsf T)f}\norm{\nabla f}.
 \label{persistent:eq:affine-control}
\end{equation}
\end{lemma}

\begin{proof}
Conditionally on \(Y_+=y\), we have
\(Y_-=y+2\sqrt\eta W\).  Set
\(e_f=\mathsf Q\mathsf U\mathsf Pf-L_f\).  The choice of \(a_f(y)\)
makes both \(\E[e_f\mid Y_+=y]\) and
\(\E[\nabla_We_f\mid Y_+=y]\) zero.
The conditional position law is one-strongly log-concave in \(W\), so
conditional Poincar\'e gives the first inequality.

For the second, let
\(b_y(w)=\sqrt\eta\,\nabla V(y+\sqrt\eta w)\).  Weak integration by parts
under the conditional density gives
\[
 \Cov(W,b_y(W)\mid y)=D_y,\qquad
 \Cov(b_y(W)\mid y)\preceq
 \frac{\beta\eta}{1-\beta\eta}D_y.
\]
Conditional Poincar\'e applied to linear functions gives
\(\Sigma_y\preceq I\), while the Stein identity
\(\E[\nabla b_y]=D_y+\Cov(b_y)\) and
\(0\preceq\nabla b_y\preceq\beta\eta I\) give
\(0\preceq D_y\preceq\beta\eta I\).
The second relation follows by applying conditional Poincar\'e to \(b_y\)
and using
\((\nabla b_y)^2\preceq(\beta\eta)\nabla b_y\).
These are standard Brascamp--Lieb/Stein calculations
\citep{chewi2023log}; for a \(C^{1,1}\) potential they follow by
mollification and cutoff.

Differentiating \(\mathsf Tf\) weakly and inserting
\(\mathsf Q\mathsf U\mathsf Pf=L_f+e_f\) yields
\[
 \nabla(I-\mathsf T)f
 =\nabla f-a_f+2D_ya_f+\eta^{-1/2}
   \Cov(e_f,b_y(W)\mid Y_+=y).
\]
Take its inner product with \(\nabla f+a_f\) and average over \(y\).
Using \(a_f=\mathsf T\nabla f\), exact rearrangement gives
\begin{align*}
 \mathcal R_f+4\mathcal D_f
 ={}&\ip{\nabla(I-\mathsf T)f}{\nabla f+a_f}
 -2\E\left[a_f^\top D_{Y_+}(\nabla f-a_f)\right]\\
 &-\eta^{-1/2}\E\left[
   \Cov(e_f,b_{Y_+}(W)\mid Y_+)^\top(\nabla f+a_f)\right].
\end{align*}
Cauchy--Schwarz, the two conditional covariance bounds above, and the first
inequality control the right-hand side by
\[
 2\norm{\nabla(I-\mathsf T)f}\norm{\nabla f}
 +C\sqrt{\beta\eta\,\mathcal R_f\mathcal D_f}
 +C\sqrt{\beta\eta\,\mathcal R_f
                 (\beta\eta\,\mathcal R_f+\mathcal D_f)}.
\]
For sufficiently small \(\beta\eta\), Young's inequality absorbs the last
two terms and proves \cref{persistent:eq:affine-control}.
\end{proof}

\subsubsection{Step 2: What the Quarter-Turn does to Affine Information}

At fixed \(y\), write
\(\ell_{a,y}(w)=a^\top(w-\bar w_y)\) and \(m_a(p)=a^\top p\).
Let \(\mathsf S\) denote the quarter-period kernel after drawing an
independent exact reference from the conditional position law.
All norms and expectations in the next lemma are under
\(q_y^w\otimes\gamma_d\), including the independent reference randomness.

\begin{lemma}[Approximate quarter rotation]
\label{persistent:lem:joint-quarter}
\label{prop:persistent-main-rotation}
If \(L_\eta:=\beta\eta\le1/2\), then, uniformly in \(y\),
\[
 \norm{\mathsf S\ell_{a,y}-m_a}^2
 +\norm{\mathsf S m_a+\ell_{a,y}}^2
 \le \frac{22L_\eta}{1-L_\eta}
       a^\top D_ya.
\]
\end{lemma}

\begin{proof}
For fixed reference \(e\), the generator in
\cref{eq:persistent-generator} maps
\(\ell_{a,y}\) to \(m_a\) and maps \(m_a\) to
\(-\ell_{a,y}+k_{a,e}\), where, with
\(r=b_y(w)-b_y(e)\) and
\(Z=p^\top r/\norm r\),
\[
 k_{a,e}
 =\E[a^\top b_y(W)\mid y]-a^\top b_y(e)
  -2a^\top r\,Z^2\mathbf1_{\{Z>0\}}.
\]
When \(r=0\), only the final term is defined to be zero; the first two
terms are unchanged.
The localized Dynkin formula therefore gives, at \(T=\pi/2\),
\[
 S_T^{y,e}\ell_{a,y}-m_a
 =\int_0^TS_s^{y,e}k_{a,e}\cos s\,\d s,\qquad
 S_T^{y,e}m_a+\ell_{a,y}
 =\int_0^TS_s^{y,e}k_{a,e}\sin s\,\d s.
\]
This is the variation-of-constants identity for the two observable
equations (equivalently, a localized Dynkin formula).  Localization is
legitimate because the harmonic energy in the proof of
\cref{persistent:lem:fixed-invariance} is conserved; dominated convergence
then removes the cutoff.

When \(e\) is an independent conditional draw, elementary Gaussian moments
give
\(\E k_{a,e}^2=11\Var(a^\top b_y(W)\mid y)\).
The covariance calculation in the proof of
\cref{persistent:lem:affine-control} bounds this variance by
\(L_\eta(1-L_\eta)^{-1}a^\top D_ya\).  Contractivity of \(S_s^{y,e}\)
and Cauchy--Schwarz in time prove the claimed bound.
\end{proof}

\subsubsection{Step 3: Assemble the Global Transfer Estimate}

\paragraph{The single graph closure.}
For every centered \(f\in\operatorname{Dom}(D)\), define
\[
 a_f=\mathsf T Df,\qquad
 M_f=2\sqrt\eta\,a_f(Y)^{\top}P,\qquad
 \mathsf C=\mathsf Q\mathsf U\mathsf P.
\]
This agrees with the preceding definition on smooth functions.  Since
\(P\) is conditionally standard Gaussian and \(\mathsf T\) is a
contraction on vector-valued \(L^2\), for
\(f,h\in\operatorname{Dom}(D)\),
\begin{equation}
 \norm{M_f-M_h}^2
 =4\eta\norm{\mathsf T D(f-h)}^2
 \le4\eta\norm{D(f-h)}^2.
 \label{persistent:eq:momentum-graph-continuity}
\end{equation}
Consider the stacked error map from \(\operatorname{Dom}(D)\), equipped
with the graph norm \(\norm{f}_D^2=\norm f^2+\eta\norm{Df}^2\), into three
copies of the stationary \(L^2\) space:
\begin{equation*}
 \boldsymbol{\mathsf E}f
 =\left(
 \mathsf S\mathsf C f-M_f,
 \mathsf S M_f+\mathsf C f,
 \mathsf U M_f-M_f
 \right).
\end{equation*}
The operators \(\mathsf C,\mathsf S,\mathsf U\) are contractions, so
\cref{persistent:eq:momentum-graph-continuity} makes
\(\boldsymbol{\mathsf E}\) graph-norm continuous.  Moreover,
\cref{persistent:eq:graph-difference} shows that
\(f\mapsto D(I-\mathsf T)f\) is graph-norm continuous.  These two facts
will pass the final three-error inequality through the core in one step;
none of its intermediate conditional error quantities needs a separate
closure argument.

\begin{proof}[Proof of \cref{persistent:prop:three-side}]
First let \(f\in\mathcal C_0\).  Apply
\cref{persistent:lem:joint-quarter} conditionally with coefficient
\(2\sqrt\eta\,a_f(y)\), then add the remainder \(e_f\) from
\cref{persistent:lem:affine-control}.  The first two terms on the left of
\cref{persistent:eq:three-side} are bounded by
\(C\{\beta\eta^2\mathcal D_f+\norm{e_f}^2\}\).

For the third term, reflection exchanges \(Y_+\) and \(Y_-\) without
changing momentum.  Since \(\mathsf T\succeq-I/2\),
\(\norm{\mathsf U M_f-M_f}^2\le4\eta\mathcal R_f\).
\cref{persistent:eq:affine-control} now bounds the sum of all three terms
by
\[
 C\eta\norm{\nabla(I-\mathsf T)f}\norm{\nabla f}.
\]
Equivalently, on \(\mathcal C_0\) we have the pre-spectral estimate
\begin{equation}
 \norm{\boldsymbol{\mathsf E}f}^2
 \le C\eta\norm{D(I-\mathsf T)f}\norm{Df}.
 \label{persistent:eq:pre-spectral-transfer}
\end{equation}

Now take an arbitrary \(f\in\mathcal H_+\).  By
\cref{persistent:lem:closed-gradient-smoothing},
\(f\in\operatorname{Dom}(D)\).  Because \(\mathcal C_0\) is a graph
core, there are \(f_n\in\mathcal C_0\) with \(f_n\to f\) and
\(Df_n\to Df\).  The approximants are not required to lie in
\(\mathcal H_+\).  The graph continuity proved above gives
\[
 \boldsymbol{\mathsf E}f_n\to\boldsymbol{\mathsf E}f,
 \qquad
 D(I-\mathsf T)f_n\to D(I-\mathsf T)f.
\]
Passing to the limit in
\cref{persistent:eq:pre-spectral-transfer} and only then applying the two
sector derivative estimates in
\cref{persistent:lem:outer-geometry} proves
\cref{persistent:eq:three-side}.  Thus no density assertion about
\(\mathcal C_0\cap\mathcal H_+\), and no smoothness preservation by the
spectral projection, is used.
\end{proof}

\paragraph{A bounded form used below.}
The transfer proposition is now complete.  We next normalize its position
and momentum signals to the same scale so that the contraction proof can
treat them as two coordinates of one approximate rotation.  The isometry
\(\mathsf V\) is the normalized position channel, \(\mathsf A\) will be the
normalized momentum channel, and \(F_1,F_2,F_3\) record the three departures
from an exact quarter-rotation.  The later factorization
\(F_j=B_j\Gamma\) says that these errors shrink with the amount of signal
being transferred.

On \(\mathcal H_+\), the definition above and
\cref{persistent:lem:outer-geometry} make
\(f\mapsto M_f\) a bounded operator \(\mathsf M\): indeed
\(\norm{M_f}^2\le4\eta\norm{\nabla f}^2
\le C\norm{\Gamma f}^2\).  Put
\(\mathsf A=\mathsf M\Gamma^{-1}\), and recall
\(\mathsf Q\mathsf U\mathsf P=\mathsf V\Gamma\).
\(\mathsf A\) is therefore bounded on \(\mathcal H_+\).  Define
\(F_1,F_2,F_3\) there by
\begin{equation}
 \mathsf S\mathsf V=\mathsf A+F_1,\qquad
 \mathsf S\mathsf A=-\mathsf V+F_2,\qquad
 \mathsf U\mathsf A=\mathsf A+F_3.
 \label{persistent:eq:normalized-transfer}
\end{equation}
Position observables are even in momentum and momentum observables are odd,
so the ranges of \(\mathsf V\) and \(\mathsf A\) are orthogonal.
Momentum reversal in \cref{persistent:lem:fixed-invariance} gives the
adjoint versions of these relations.  When an adjoint is taken in the next
subsection, \(\mathsf V,\mathsf A,F_j\) mean their restrictions to
\(\mathcal H_+\), extended by zero on its orthogonal complement.  Thus all
subsequent adjoints are bounded, everywhere-defined Hilbert-space
operators.

\begin{corollary}[Controlled transfer errors]
\label{persistent:cor:factorized-errors}
Let \(\ell_g=1+\log(1/g)\).  There are bounded operators \(B_j\) such that
\begin{equation}
 F_j=B_j\Gamma,\qquad F_j^*=\Gamma B_j^*,\qquad
 \norm{B_j}\le C\sqrt{\ell_g},\qquad j=1,2,3.
 \label{persistent:eq:factorized-errors}
\end{equation}
\end{corollary}

\begin{proof}
\cref{persistent:eq:three-side}, with \(f\) replaced by
\(\Gamma^{-1}f\), gives
\(\norm{F_jf}^2\le C\norm f\norm{\Gamma^2f}\).
Partition the spectrum of \(\Gamma^2\) into dyadic intervals between
\(cg^2\) and \(1\).  On each spectral component \(f_r\),
\(\norm{F_jf_r}^2\le C\ip{f_r}{\Gamma^2f_r}\).
There are \(O(\ell_g)\) components, so the triangle inequality followed by
Cauchy--Schwarz gives
\[
 F_j^*F_j\preceq C\ell_g\Gamma^2.
\]
Since \(\Gamma\) is invertible on the centered subspace, setting
\(B_j=F_j\Gamma^{-1}\) proves
\cref{persistent:eq:factorized-errors}.  The factor order is important;
no commutation between \(B_j\) and \(\Gamma\) is used.
\end{proof}

\subsection{From Position--Momentum Transfer to Contraction}
\label{persistent:sec:resolvent}

We now prove the ideal-chain bound.  Reflection alone preserves norm, and
the quarter-turn mainly rotates information between position and momentum,
so a one-step Poincar\'e estimate cannot reveal the contraction.  Instead,
we rule out oscillations on the unit circle and then convert that bound into
geometric decay.

The resolvent strategy asks whether \((z-\mathsf K)f=b\) can have a large
solution when \(\abs z=1\).  For each slowly varying \(Y\)-mode, the maps
\(\mathsf V\) and \(\mathsf A\) expose its position and momentum companions;
the resulting three coordinates form an approximate finite-dimensional
rotation.  The refresh dissipation estimate controls the orthogonal
conditional variation that these coordinates do not see.  A uniform bound
on the unit-circle resolvent then yields geometric decay through
\cref{persistent:lem:resolvent-decay}.

\begin{proposition}[Ideal persistent contraction]
\label{persistent:prop:ideal-contraction}
Under the assumptions of \cref{thm:persistent-contraction}, choosing
\(\rho\asymp g\sqrt{\ell_g}\), where
\(g=\sqrt{\alpha\eta}\) and \(\ell_g=1+\log(1/g)\), gives
\begin{equation}
 \norm{\mathsf K^n}_{L^2_0\to L^2_0}
 \le2\exp\left(-\frac{cng}{\sqrt{\ell_g}}\right).
 \label{eq:persistent-l2}
\end{equation}
\end{proposition}

The proof has three steps.  We follow three slowly changing coordinates,
invert their three-by-three update on the unit circle, and use refreshment
to control all remaining conditional variation.  The coordinates do not
form a decomposition of the full state space; this is why the final
refreshment step is essential.

We allow complex-valued observables below so that \(z\) may range over the
complex unit circle; this does not change any operator norm.  Let
\(k=1-\rho\), and let \(\mathsf K\) be the one-step operator in
\cref{persistent:eq:K-def}.  Recall
\(\mathsf T,\Gamma,\mathsf V,\mathsf A\) from
\cref{persistent:lem:outer-geometry,persistent:eq:normalized-transfer}.

\begin{lemma}[Unit-circle resolvent bound]
\label{persistent:lem:resolvent}
If \(\rho\asymp g\sqrt{\ell_g}\), then
\[
 \sup_{\abs z=1}
 \norm{(z-\mathsf K)^{-1}}_{L^2_0\to L^2_0}
 \le \frac{C\sqrt{\ell_g}}{g}.
\]
\end{lemma}

\begin{proof}
\emph{Step 1: isolate the three slow coordinates.}
Fix \(\abs z=1\) and solve
\((z-\mathsf K)f=b\).  Write \(f_\perp=\mathsf Qf\).
Let
\(\Pi_{\rm slow}=\mathbf1_{[1/2,1]}(\mathsf T)\), and split the
\(Y\)-only component into
\(x_{\rm slow}=\Pi_{\rm slow}\mathsf Pf\) and
\(x_{\rm fast}=(I-\Pi_{\rm slow})\mathsf Pf\).
In the three-coordinate calculation, every occurrence of
\(\mathsf T,\Gamma,\mathsf V,\mathsf A,F_j\) is restricted to
\(\mathcal H_+=\operatorname{range}(\Pi_{\rm slow})\); equivalently,
\(\mathsf V\) means \(\mathsf V\Pi_{\rm slow}\), and similarly for the
other maps.  The other two tested coordinates are then
\(v=\mathsf V^*f_\perp\) and \(m=\mathsf A^*f_\perp\).

\begin{center}
\begin{tabular}{@{}lp{0.67\linewidth}@{}}
\toprule
coordinate & information it measures\\
\midrule
\(x_{\rm slow}\) & the slowly changing component stored in the auxiliary
variable \(Y\)\\
\(v=\mathsf V^*f_\perp\) & conditional variation aligned with the position
signal\\
\(m=\mathsf A^*f_\perp\) & conditional variation aligned with the momentum
signal\\
\(x_{\rm fast}\) & the remaining \(Y\)-modes, controlled directly by the
spectral gap of \(z-\mathsf T\)\\
\bottomrule
\end{tabular}
\end{center}

The next identities encode the approximate quarter-rotation in these
coordinates; their error terms are precisely the factorized \(F_j\)'s from
the preceding subsection.  The adjoint identities used below are
\[
 \mathsf V^*\mathsf Su=-\mathsf A^*u+F_1^*\mathsf Ru,\quad
 \mathsf A^*\mathsf Su=\mathsf V^*u-F_2^*\mathsf Ru,\quad
 \mathsf V^*\mathsf U\mathsf Q=-\mathsf T\mathsf V^*,\quad
 \mathsf A^*\mathsf U=\mathsf A^*+F_3^*.
\]
They follow from \cref{persistent:eq:normalized-transfer}, momentum
reversal, and the even/odd momentum parity recorded there.  No
\(x_{\rm fast}\) term is omitted below: spectral orthogonality gives
\(\mathsf V^*\mathsf U\mathsf P x_{\rm fast}=0\), while momentum parity
gives \(\mathsf A^*\mathsf U\mathsf P x_{\rm fast}=0\) and
\(F_3^*\mathsf P x_{\rm fast}=0\).

Projecting the resolvent equation onto these coordinates gives
\begin{align}
 (z-\mathsf T)x_{\rm slow}+k\Gamma m
   &=b_0+k\Gamma r_1,\notag\\
 zv-\Gamma x_{\rm slow}-k\mathsf Tm
   &=b_V-k\mathsf Tr_1,\label{persistent:eq:three-coordinate-system}\\
 zm-kv&=b_A+kr_2,\notag
\end{align}
where \(b_0=\Pi_{\rm slow}\mathsf Pb\),
\(b_V=\mathsf V^*b\), \(b_A=\mathsf A^*b\), and
\[
 r_1=F_1^*\mathsf Rf_\perp,\qquad
 r_2=-F_2^*\mathsf Rf_\perp+F_3^*\mathsf Sf_\perp.
\]
\cref{persistent:cor:factorized-errors} implies
\[
 \norm{\Gamma^{-1}r_1}+\norm{\Gamma^{-1}r_2}
 \le C\sqrt{\ell_g}\norm{f_\perp}.
\]
This is the only logarithmic loss.  The factorization
\(F_j^*=\Gamma B_j^*\) is used in this order; no commutation is assumed.

\emph{Step 2: invert the three-coordinate system.}
All coefficients on the left of
\cref{persistent:eq:three-coordinate-system} are functions of the single
self-adjoint operator \(\mathsf T\).  At a scalar spectral value
\(t\in[1/2,1)\), put \(\gamma_t=\sqrt{1-t^2}\).  The determinant of the
coefficient matrix is
\[
 d_{t,k}(z)=z^3-tz^2-k^2tz+k^2.
\]
Taking the first row of the adjugate gives the exact operator identity
\begin{align}
 d_{\mathsf T,k}(z)x_{\rm slow}
 ={}&(z^2-k^2\mathsf T)b_0-k^2\Gamma b_V-kz\Gamma b_A\notag\\
 &+kz^2\Gamma r_1-k^2z\Gamma r_2.
 \label{persistent:eq:adjugate-row}
\end{align}
The cancellation that produces the final two terms in this form is what
prevents an extra factor \(g^{-1}\).

We use two elementary scalar estimates.  If
\(\abs{z-t}<\gamma_t/16\), then
\[
 \abs{d_{t,k}(z)}\ge\frac{15}{32}\gamma_t^2,\qquad
 \abs{z^2-k^2t}\le4(\rho+\gamma_t).
\]
For the first, write \(z=\e^{2\i u}\).  The near condition gives
\(\abs u\le\gamma_t/16\), while the real part of
\(\e^{-3\i u}d_{t,k}(\e^{2\i u})\) is
\((1+k^2)\{\cos(3u)-t\cos u\}\).
The identity
\(\cos(3u)-t\cos u=(1-t)\cos u-2\sin(2u)\sin u\), together with
\(1-t=\gamma_t^2/(1+t)\), proves the lower bound.  The second estimate
follows from
\(\abs{z^2-k^2t}\le2\abs{z-1}+(1-t)+(1-k^2)\).

On the spectral region
\(\abs{z-t}\ge\gamma_t/16\), the first row of
\cref{persistent:eq:three-coordinate-system} directly controls
\(x_{\rm slow}\).  On its complement, apply the two scalar estimates to
\cref{persistent:eq:adjugate-row}.  The terms involving \(b\) cost at most
\(C(\rho/g^2+g^{-1})\norm b\), while
\cref{persistent:eq:factorized-errors} bounds the two error terms as
follows.  If \(E\) denotes this near spectral region, then
\[
 \norm{\mathbf1_Ed_{\mathsf T,k}(z)^{-1}\Gamma F_j^*q}
 =\norm{\mathbf1_Ed_{\mathsf T,k}(z)^{-1}\Gamma^2B_j^*q}
 \le C\sqrt{\ell_g}\norm q.
\]
Here the lower bound on \(d_{t,k}\) cancels the two powers of
\(\gamma_t\) supplied by \(\Gamma^2\).  Since
\(\Gamma\succeq cgI\) and
\(\rho\asymp g\sqrt{\ell_g}\), both regions yield
\begin{equation}
 \norm{\mathsf Pf}
 \le C\sqrt{\ell_g}
 \left(\norm{f_\perp}+\frac{\norm b}{g}\right).
 \label{persistent:eq:macro-resolvent}
\end{equation}
Here the fast component has not been dropped: on
\([-1/2,1/2)\), \(z-\mathsf T\) has inverse norm at most two, so
\(\norm{x_{\rm fast}}\le2(\norm b+\norm{f_\perp})\).

\emph{Step 3: control everything not tested above.}
\cref{persistent:eq:dissipation} and
\(\mathsf Kf=zf-b\) imply
\(\rho\norm{f_\perp}^2\le2\norm f\norm b\).
Combine this with
\cref{persistent:eq:macro-resolvent}, use
\(f=\mathsf Pf+f_\perp\), and apply Young's inequality.  The result is
\[
 \norm f\le\frac{C\sqrt{\ell_g}}g\norm b.
\]
Thus \(z-\mathsf K\) is bounded below and has closed range.

It is also onto.  Otherwise some nonzero \(u\) would satisfy
\(\mathsf K^*u=\bar z u\).  Because \(\mathsf K\) is a contraction and
\(\abs z=1\), equality of norms implies
\(\mathsf K\mathsf K^*u=u\), and hence \(\mathsf Ku=zu\).  This
contradicts the lower bound at \(z\), proving the asserted resolvent
estimate.
\end{proof}

The next elementary Hilbert-space lemma converts the resolvent estimate
into a bound on powers.  We retain its short proof because the precise
linear dependence on \(M\) is needed for the final rate.

\begin{lemma}[Unit-circle bound implies geometric decay]
\label{persistent:lem:resolvent-decay}
Let \(Q\) be a contraction on a Hilbert space.  If
\(\sup_{\abs z=1}\norm{(z-Q)^{-1}}\le M\) with \(M\ge1\), then
\(\norm{Q^n}\le2\exp(-cn/M)\).
\end{lemma}

\begin{proof}
For \(a_j=\sin((j+1)\pi/(N+2))\), apply the resolvent bound to
\(\sum_{j=0}^Na_jz^{-j-1}Q^jf\) and integrate over the unit circle.
Parseval's identity gives
\[
 \sum_{j=0}^Na_j^2\norm{Q^jf}^2
 \le M^2\sum_{j=0}^{N+1}(a_j-a_{j-1})^2\norm{Q^jf}^2,
 \qquad a_{-1}=a_{N+1}=0.
\]
Contractivity and
\(\sum(a_j-a_{j-1})^2
=2\{1-\cos(\pi/(N+2))\}\sum a_j^2\)
imply \(\norm{Q^N}\le\pi M/(N+2)\).  Taking
\(N_0=\lceil2\pi M\rceil\) gives
\(\norm{Q^{N_0}}\le1/2\); writing \(n=qN_0+r\) proves the claim.
\end{proof}

\begin{proof}[Proof of \cref{persistent:prop:ideal-contraction}]
\cref{persistent:lem:row-identities,persistent:eq:normalized-transfer,persistent:cor:factorized-errors}
provide the invariance, dissipation,
row identities, and controlled errors used in
\cref{persistent:lem:resolvent}; these ultimately follow from
\cref{persistent:lem:fixed-invariance,persistent:prop:three-side}.
Apply
\cref{persistent:lem:resolvent-decay} with
\(M=C\sqrt{\ell_g}/g\).  This proves \cref{eq:persistent-l2}.
\end{proof}

\begin{proof}[Proof of \cref{thm:persistent-contraction}]
Choose the universal constant in \(n_{\mathrm{blk}}\) sufficiently large that \cref{eq:persistent-l2} gives \(\norm{\mathsf K^{n_{\mathrm{blk}}}}_{L^2_0\to L^2_0}^2\le\e^{-1}\).  If \(r_\nu=\d\nu/\d\widehat\Pi_\eta\), then the density ratio evolves under \(\mathsf K^*\) and \(\chi^2(\nu\Vert\widehat\Pi_\eta)=\norm{r_\nu-1}_2^2\).  Since an operator and its adjoint have the same norm, applying the block estimate \(J\) times proves \cref{eq:persistent-ideal-contraction}.

For the total-variation conclusion, initialize \(Y_0\mid X_0\sim\cN(X_0,\eta I_d)\) and \(P_0\sim\gamma_d\).  The initial and stationary augmented laws have the same conditional law of \((Y,P)\) given \(X\), so augmentation preserves \(D_2\).  Choose the smallest admissible \(n_{\mathrm{blk}}\) and put \(J=\lfloor n/n_{\mathrm{blk}}\rfloor\).  The remaining iterations cannot increase \(\chi^2\) because an invariant Markov kernel is an \(L^2\) contraction.  Combining \cref{eq:persistent-ideal-contraction} with \(\TV\le\frac12\sqrt{\chi^2}\), and enlarging the universal constant, gives \cref{eq:persistent-iterations}; marginalization proves the claim for \(X_n\).
\end{proof}

\subsection{Implementable Algorithm and Warm-Start Accounting}
\label{persistent:sec:implementation}

We begin with the result proved in this subsection and then verify its four
implementation components: the shifted Gaussian reference, a deterministic
iteration cap for the mode solver, a capped event intensity simulated by
Poisson thinning, and the approximate RGO used on the RGO branch.

\begin{theorem}[Implementation theorem]
\label{persistent:thm:finite-compiler}
Let \(D_2(\nu_0\Vert\pi)\le\Delta\), \(0<\eps\le1/2\), and define
\[
 \Lambda_\eps=\Delta+\log\frac{Cd\kappa}{\eps},
 \qquad
 \eta=\frac{c}{\beta(\sqrt{d\Lambda_\eps}+\Lambda_\eps)}.
\]
Put \(g=\sqrt{\alpha\eta}\), \(\ell_g=1+\log(1/g)\), choose
\(\rho\asymp g\sqrt{\ell_g}\), and set
\[
 N=\left\lceil\frac{C\sqrt{\ell_g}}g
       \{\Delta+\log(8/\eps)\}\right\rceil.
\]
Initialize \(X_0\sim\nu_0\),
\(Y_0\mid X_0\sim\cN(X_0,\eta I_d)\), and
\(P_0\sim\gamma_d\), independently conditional on \(X_0\).
There is an implementable algorithm for the persistent iteration in
\cref{sec:persistent}, using \cref{alg:persistent-quarter-turn} as its
accelerated branch, that returns \(X_N\) with
\(\TV(\Law(X_N),\pi)\le\eps\) and
\[
 \E N_{\mathrm{qry}}
 \le\widetilde O\left[
 \sqrt\kappa\,(\Delta+\log(8/\eps))
 (\sqrt{d\Lambda_\eps}+\Lambda_\eps)^{1/2}\right],
 \qquad
 \E N_{\mathrm{rgo}}
 \le\widetilde O(\Delta+\log(8/\eps)).
\]
\end{theorem}

\subsubsection{A Comparison Principle for the Complete Path}

The chain induced by the implementable algorithm is not invariant because it contains truncation and
approximation.  We therefore keep the exact chain as a comparison process
and couple the two chains until their first disagreement.

\begin{lemma}[Warm path comparison]
\label{persistent:lem:warm-path-transfer}
Let an exact Markov kernel preserve \(\pi\), and let
\(\nu\ll\pi\) satisfy \(D_2(\nu\Vert\pi)\le\Delta\).  If \(A\) is an
event determined by a Markov path together with all its internal
randomness, then
\[
 \Pr_\nu(A)\le \e^{\Delta/2}\Pr_\pi(A)^{1/2}.
\]
Consequently, if the stationary one-step probability that an exact and an
approximate transition disagree is at most \(a\), their \(N\)-step laws
started from \(\nu\) can be coupled with disagreement probability at most
\(N\e^{\Delta/2}\sqrt a\).
\end{lemma}

\begin{proof}
For \(R=\d\nu/\d\pi\), conditional Cauchy--Schwarz gives
\(\E_\pi[R\Pr(A\mid X_0)]\le\norm R_2\Pr_\pi(A)^{1/2}\).
An invariant exact kernel cannot increase \(D_2\).  Apply the same bound
before each exact transition and take a union bound over the first
disagreement.
\end{proof}

\subsubsection{Replacing the Ideal Exact-Conditional Reference on the Accelerated Branch}

The ideal contraction proof draws \(E\sim q_y\) on the accelerated branch,
but this draw is used only as an independent reference for the event-driven
motion; it is distinct from the actual refresh on the RGO branch.  The finite
algorithm replaces this ideal reference by a Gaussian centered near the
conditional mode.  The next comparison lemma controls the resulting
normalized shift and is used only to re-establish the local quarter-turn
error bound.

Fix \(y\), let \(x_*(y)\) minimize the conditional potential, and set
\(z=(x-x_*(y))/\sqrt\eta\).  After removing a constant and a linear term,
the normalized conditional density has the form
\[
 q(z)\propto \e^{-\norm z^2/2-r(z)},\qquad
 mI\preceq\nabla^2r\preceq LI,
\]
where \(m=\alpha\eta\), \(L=\beta\eta\), and \(\nabla r(0)=0\).  The
letters \(m\) and \(L\) are local to this normalization and to
\cref{persistent:lem:affine-compiler}; they are not the Poincar\'e gap
\(m\) of \cref{ass:pi-extension}.

\begin{lemma}[Shifted-Gaussian comparison]
\label{persistent:lem:gaussian-comparison}
Let \(\gamma_\delta=\cN(\delta,I_d)\), and set
\(\tau=2+\log(L/m)\).  If
\(\tau L\le c\) and
\(\tau(L^2d+\norm\delta^2)\le c\), then
\[
 \left\|\frac{\d\gamma_\delta}{\d q}\right\|_{L^\tau(q)}
 \le \e^{C\tau(L^2d+\norm\delta^2)},
 \qquad
 \chi^2(\gamma_\delta\Vert q)
 \le C(L^2d+\norm\delta^2).
\]
\end{lemma}

\begin{proof}
Normalize \(r(0)=0\).  Convexity and \(L\)-smoothness give
\(0\le r(z)\le L\norm z^2/2\) and
\(\norm{\nabla r(z)}^2\le2Lr(z)\).  The Gaussian log-Sobolev inequality
and the standard entropy argument \citep{chewi2023log} imply, for
\(0\le tL<1\),
\[
 \log\E_\gamma \e^{t(r-\E_\gamma r)}
 \le\frac{t^2L^2d}{2(1-tL)},\qquad
 \log\E_\gamma \e^{-t(r-\E_\gamma r)}
 \le\frac{t^2L^2d}{2}.
\]
For completeness, the second inequality is not obtained by simply changing
the sign in the first.  If
\(A_-(t)=\log\E_\gamma \e^{-t(r-\E_\gamma r)}\), the same entropy
calculation gives
\(tA_-'(t)-A_-(t)\le t^2L\{\E_\gamma r-A_-'(t)\}\).
Convexity and \(A_-(0)=A_-'(0)=0\) imply \(A_-'(t)\ge0\), so integrating
\((A_-(t)/t)'\le L\E_\gamma r\le L^2d/2\) proves the displayed
negative-tilt bound.
Applying these bounds to \(\d\gamma/\d q\) gives
\(\log\E_q(\d\gamma/\d q)^\tau\le C\tau^2L^2d\).
For the shift,
\(\d\gamma_\delta/\d\gamma
=\exp\{\delta^\top z-\norm\delta^2/2\}\).
The law \(q\) is one-strongly log-concave, so its linear functionals are
one-subgaussian after centering.  Weak integration by parts gives
\(\E_qZ=-\E_q\nabla r(Z)\), while \(\E_q\norm Z^2\le d\); hence
\(\norm{\E_qZ}\le L\sqrt d\).  Cauchy--Schwarz and Young's inequality
therefore add at most
\(C\tau^2(L^2d+\norm\delta^2)\) to the logarithmic moment.  Taking the
\(\tau\)-th root proves the first claim.  Repeating the calculation at
exponent two and using \(\e^u-1\le2u\) in the stated smallness regime proves
the \(\chi^2\) bound.  Mollification and cutoff justify the calculation for
\(C^{1,1}\) potentials.
\end{proof}

Starting from \(y\), run gradient descent on the conditional potential
until it returns a measurable point \(\widehat x_y\) satisfying
\[
 \norm{y-\eta\nabla V(\widehat x_y)-\widehat x_y}
 \le\sqrt{d\eta}.
\]
The accelerated branch then draws \(G\sim\cN(0,I_d)\) and uses
\begin{equation}
 E_y=y-\eta\nabla V(\widehat x_y)+\sqrt\eta\,G.
 \label{persistent:eq:shifted-reference}
\end{equation}
Strong monotonicity of
\(x\mapsto x+\eta\nabla V(x)-y\) shows that the normalized shift of
\cref{persistent:eq:shifted-reference} from \(x_*(y)\) is at most
\(L\sqrt d/(1+\alpha\eta)\).  Thus
\cref{persistent:lem:gaussian-comparison} applies whenever
\(\tau L\lesssim1\) and \(\tau L^2d\lesssim1\).
Because the reference depends only on \(y\) and fresh Gaussian randomness,
\cref{persistent:lem:fixed-invariance} gives exact conditional invariance
before numerical caps are imposed.

\begin{lemma}[Shifted-reference transfer]
\label{persistent:lem:shifted-transfer}
If \(\tau L\lesssim1\) and \(\tau L^2d\lesssim1\), the conditional kernel using
\cref{persistent:eq:shifted-reference} satisfies
\cref{persistent:prop:three-side,persistent:cor:factorized-errors}, with
possibly different universal constants.
\end{lemma}

\begin{proof}
Let \(h=\nabla r\), \(\Sigma=\Cov_q(Z)\), and \(D=I-\Sigma\).  For the
reference force \(\mathcal U=h(G+\delta)\), direct expansion of the
Gaussian momentum moments in the proof of
\cref{persistent:lem:joint-quarter} gives
\[
 \E k_{a,E}^2
 =6\Var_q(a^\top h)+5\Var(a^\top\mathcal U)
  +5\{\E_q(a^\top h)-\E(a^\top\mathcal U)\}^2.
\]
The weak integration-by-parts calculation used in
\cref{persistent:lem:affine-control} yields
\(\Var_q(a^\top h)\le CL\,a^\top Da\).

It remains to control the two Gaussian-reference terms.  If
\(H=\nabla^2r\) denotes the weak Hessian, then
\(\E_qH=D+\Cov_q(h)\preceq C D\).  H\"older's inequality,
\cref{persistent:lem:gaussian-comparison}, and
\(0\preceq H\preceq LI\) give
\(\E_{\gamma_\delta}[a^\top Ha]\le C a^\top Da\).
Indeed, \(a^\top Ha\le L\norm a^2\),
\(\E_q[a^\top Ha]\ge m\norm a^2\), and the choice
\(\tau=2+\log(L/m)\) makes the resulting H\"older ratio universal.
Gaussian Poincar\'e therefore implies
\(\Var_{\gamma_\delta}(a^\top h)\le CL\,a^\top Da\).
The \(\chi^2\) part of
\cref{persistent:lem:gaussian-comparison}, followed by Cauchy--Schwarz
under \(q\), bounds the squared difference of means by the same quantity.
Hence the local quarter-turn estimate holds with the shifted reference.
The affine lifting in \cref{persistent:prop:three-side}, momentum parity,
and the factorization proof are unchanged.
\end{proof}

\subsubsection{Mode, Clock, and Refresh Subroutines}

\begin{lemma}[Mode certificate and deterministic cap]
\label{persistent:lem:mode-cap}
Assume \(\beta\eta\le1\).  Gradient descent with step
\((\beta+\eta^{-1})^{-1}\), initialized at \(y\), finds the certificate
used in \cref{persistent:eq:shifted-reference} using at most
\[
 C\left[1+\log_+
 \left(\frac{\sqrt\eta\norm{\nabla V(y)}}{\sqrt d}\right)\right]
\]
oracle queries.  Under the stationary augmented law, stopping after
\[
 J_B=C\left[1+\log(2+\kappa)
       +\log\left(1+\sqrt{B/d}\right)\right]
\]
iterations fails with probability at most \(2\e^{-B}\).
The uncapped output is a Borel function of \(y\).
\end{lemma}

\begin{proof}
The conditional potential has Hessian between
\(\alpha+\eta^{-1}\) and \(\beta+\eta^{-1}\), hence condition number at
most two.  One gradient step therefore reduces the gradient norm by a
universal factor, which gives the pointwise logarithmic bound.

For the stationary tail, let \(x^\star\) be the global minimizer of \(V\)
and write \(Y=X+\sqrt\eta Z\), where
\(Z\sim\cN(0,I_d)\).  Strong log-concavity gives
\(\E\norm{X-x^\star}^2\le d/\alpha\) and Gaussian concentration around that
scale.  Combining it with the Gaussian norm tail and
\(\norm{\nabla V(Y)}\le\beta\norm{Y-x^\star}\) gives, outside an event of
probability \(2\e^{-B}\),
\[
 \frac{\sqrt\eta\norm{\nabla V(Y)}}{\sqrt d}
 \le2\sqrt\kappa\left(1+\sqrt{2B/d}\right).
\]
Substitution into the pointwise bound gives \(J_B\).  Every finite
gradient-descent iterate and the first iterate meeting the certificate are
Borel functions of \(y\).  Returning the complete pre-reflection state on
cap failure therefore defines a total Markov kernel without violating
conditional independence.
\end{proof}

\begin{lemma}[Clock cap]
\label{persistent:lem:capped-clock}
Suppose the normalized residual potential \(r\) has
\(L\)-Lipschitz gradient with \(0\le L\le1\), and the normalized reference is
\(G+\delta_y\), where \(G\) is standard
Gaussian, \(y\mapsto\delta_y\) is measurable, and
\(\sup_y\norm{\delta_y}\le L\sqrt d\).  For \(B\ge1\), set
\(\bar\lambda_B=CL(\sqrt{dB}+B)\).
Replacing the event rate by its minimum with \(\bar\lambda_B\) changes one
stationary transition with probability at most
\(C(1+L\sqrt d)\e^{-cB}\).  The capped flight is simulated by homogeneous
Poisson thinning and uses at most
\(1+(\pi/2)\bar\lambda_B\) expected oracle queries, including the query at
the reference.
\end{lemma}

\begin{proof}
Conditionally on \(Y=y\), the normalized position is one-strongly
log-concave, while \(G\) and the momentum are Gaussian.  Their norm tails,
the bound \(\norm{\nabla r(w)-\nabla r(e)}\le L\norm{w-e}\), and a
one-dimensional Gaussian tail give, uniformly in \(y\),
\[
 \Pr\{\lambda_t>CL(\sqrt{du}+u)\}\le6\e^{-u}.
\]
Invariance from \cref{persistent:lem:fixed-invariance} makes the bound valid
at every time.  Integrating the tail over the quarter-period bounds the
expected excess hazard by \(C(1+L\sqrt d)\e^{-cB}\).
Using the same Poisson random measure for the exact and capped clocks
couples the paths until the first excess event.  Conditioning on \(Y=y\)
and then integrating proves the claim for the measurable random center.
Ordinary thinning gives the query bound.
\end{proof}

\begin{proposition}[Approximate RGO]
\label{persistent:prop:rgo-compiler}
There is a universal \(K_{\rm rgo}\) such that the following holds.  Given
\(y\), \(0<\delta<1/2\), and a point \(\widehat x_y\) satisfying
\[
 \norm{y-\eta\nabla V(\widehat x_y)-\widehat x_y}
 \le\sqrt{d\eta},
\]
if
\[
 \eta^{-1}\ge K_{\rm rgo}\beta
 \{\sqrt{d\log(1/\delta)}+\log(1/\delta)\},
\]
the first-order rejection sampler of
\citet[Thms. 3.1 and D.1]{chen2026high} returns
\(\widehat q_y\) with
\(\chi^2(q_y\Vert\widehat q_y)\le\delta^2\) and
\(\chi^2(\widehat q_y\Vert q_y)\le\delta^2\), using at most
\(K_{\rm rgo}\) expected oracle queries after construction of the mode
certificate.
\end{proposition}

\subsubsection{Proof of the Implementation Theorem}

\begin{proof}[Proof of \cref{persistent:thm:finite-compiler}]
Let \(A_\eps=\Delta+\log(8/\eps)\) and take
\[
 \delta_{\rm rgo}=\frac{\eps}{8(1+N\rho)},\qquad
 B=C\left\{\Delta+\log\frac{8N}{\eps}
                  +\log(2+d+\kappa)\right\}.
\]
The exact comparison chain uses the measurable shifted reference in
\cref{persistent:eq:shifted-reference}, without mode or clock caps, and an
exact refresh.  \cref{persistent:lem:shifted-transfer} lets us apply
\cref{persistent:prop:ideal-contraction}; the choice of \(N\) makes its
distance from stationarity at most \(\eps/2\).

The implementable algorithm first reflects \(y\).  It then runs the
\(y\)-initialized mode solver for at most \(J_B\) steps.  On failure it
returns the complete pre-reflection state \((x,y,p)\).  On the refresh
branch it invokes \cref{persistent:prop:rgo-compiler} with tolerance
\(\delta_{\rm rgo}\); otherwise it uses the shifted Gaussian reference and
the clock cap in \cref{persistent:lem:capped-clock}.

Couple both chains until their first disagreement.  The complete error
allocation is
\[
\begin{array}{c|c}
\text{source}&\text{total TV contribution}\\ \hline
\text{ideal mixing}&\eps/2\\
\text{mode-solver cap}&\eps/8\\
\text{event-clock cap}&\eps/8\\
\text{approximate refresh}&N\rho\delta_{\rm rgo}/2\le\eps/16 .
\end{array}
\]
For the mode and clock errors,
\cref{persistent:lem:warm-path-transfer} converts stationary exponential
tails into path probabilities under the warm exact chain.  The RGO error
is uniform in \(y\), so maximal coupling and kernel telescoping apply
directly.  The entries sum to less than \(\eps\), and projection to \(X\)
cannot increase total variation.

We next check the local conditions.  With
\(\tau=2+\log\kappa\) and
\(L=\beta\eta=c/(\sqrt{d\Lambda_\eps}+\Lambda_\eps)\), reducing \(c\)
gives \(\tau L\lesssim1\) and \(\tau L^2d\lesssim1\).
Also \(\log(1/\delta_{\rm rgo})\lesssim\Lambda_\eps\), so the same step
size verifies the gate in \cref{persistent:prop:rgo-compiler}.

The expected number of refresh calls is
\(N\rho\le C\ell_gA_\eps\).  Each iteration uses at most
\[
 C\{J_B+1+L(\sqrt{dB}+B)+\rho K_{\rm rgo}\}
\]
expected oracle queries.  Here \(J_B\) and \(\ell_g\) are logarithmic, and
the definition of \(\eta\) makes \(L(\sqrt{dB}+B)=O(1)\).  Substituting
\(g^2=\alpha\eta\) proves the theorem and hence
\cref{thm:persistent-finite}.
\end{proof}

\subsubsection{Affine Form Used by the Quadratic Baseline}

\begin{lemma}[Affine local compiler]
\label{persistent:lem:affine-compiler}
Suppose a conditional potential is
\[
 \Phi_y(x)=R_y(x)+\frac{\omega}{2}\norm{x-y}^2,\qquad
 m\omega I\preceq\nabla^2R_y\preceq L\omega I,\qquad
 0<m\le L\le1.
\]
Let \(x_*(y)\) denote the unique minimizer of \(\Phi_y\), let
\(G\sim\cN(0,I_d)\), and fix \(0<\delta<1/2\).
Let \(\tau=2+\log(L/m)\), and assume
\(\tau L\lesssim1\) and \(\tau L^2d\lesssim1\).
If a measurable point \(\widehat x_y\) satisfies
\[
 \norm{\omega(\widehat x_y-y)+\nabla R_y(\widehat x_y)}
 \le\sqrt{d\omega},
\]
then the reference
\[
 E_y=y-\omega^{-1}\nabla R_y(\widehat x_y)+\omega^{-1/2}G
\]
satisfies the shifted-Gaussian comparison, the local quarter-turn forcing
estimate, and the clock estimate with local smoothness \(L\).  The
certificate is found from \(y\) using
\[
 C\left[1+\log_+
 \frac{\norm{\nabla R_y(y)}}{\sqrt{d\omega}}\right]
\]
oracle queries.  The RGO gate becomes
\[
 1\ge K_{\rm rgo}L
 \{\sqrt{d\log(1/\delta)}+\log(1/\delta)\}.
\]
\end{lemma}

\begin{proof}
In the coordinate \(z=\sqrt\omega(x-x_*(y))\), the residual has Hessian
between \(mI\) and \(LI\).  Strong monotonicity of
\(x\mapsto\omega(x-y)+\nabla R_y(x)\) bounds the normalized Gaussian shift
by \(L\sqrt d/(1+m)\).  Thus
\cref{persistent:lem:gaussian-comparison,persistent:lem:capped-clock}
apply after replacing \(\beta\eta\) by \(L\).
The proof of \cref{persistent:lem:shifted-transfer} up to its local
quarter-turn estimate applies in these normalized coordinates; lifting
that estimate to the global reflected chain requires the separate scaling
calculation in \cref{persistent:sec:alpha-baseline}.
The condition number of \(\Phi_y\) is at most
\((1+L)/(1+m)\le2\), giving the displayed logarithmic mode cost.
Finally, applying \cref{persistent:prop:rgo-compiler} after the affine
change of variables gives exactly the stated RGO gate.
\end{proof}

\subsection{Quadratic-Baseline Refinement}
\label{persistent:sec:alpha-baseline}

\begin{corollary}[Quadratic-baseline refinement]
\label{cor:persistent-near-gaussian}
For \(0<\eps\le1/2\), if \(\alpha\) is known and the initial position law
\(\nu_0\) satisfies \(D_2(\nu_0\Vert\pi)=O(1)\), a quadratic-baseline
implementable algorithm under the same augmented initialization attains TV error
\(\eps\) with
\(\widetilde O(1+\sqrt{\kappa-1}\,d^{1/4})\) expected oracle queries and
\(\operatorname{polylog}(d,\kappa,1/\eps)\) expected RGO calls.
\end{corollary}

We prove \cref{cor:persistent-near-gaussian} by changing only the quadratic
part integrated exactly by the conditional motion.  The outer reflection
and its gap are unchanged.

The parameter choice below is easiest to read from its design goals.  Write
\(s=\kappa-1\) for the curvature above the known \(\alpha\)-quadratic
baseline, and let \(\theta\) be the desired normalized smoothness of the
residual force.  We choose \(\lambda=1-\delta_0\) and \(\eta\) so that
\begin{enumerate}[leftmargin=*,label=(\roman*)]
\item subtracting the \(\lambda\alpha\)-quadratic leaves residual strong
      convexity \(\delta_0\alpha>0\);
\item the residual smoothness divided by the new harmonic curvature is
      exactly \(\theta\);
\item the original small-step condition remains valid, so the outer
      reflection gap is still governed by \(g=\sqrt{\alpha\eta}\).
\end{enumerate}
The displayed formulas are chosen to meet these requirements; \(C_0\)
provides fixed slack in the first one.

Let \(s=\kappa-1\) and, for a sufficiently small universal \(c\), set
\[
 \Lambda_\eps=\Delta+\log\frac{Cd\kappa}{\eps},
 \qquad
 \theta=\frac{c}{\sqrt{d\Lambda_\eps}+\Lambda_\eps}.
\]
Choose a sufficiently large universal \(C_0\), put
\(\delta_0=C_0\theta\) and \(\lambda=1-\delta_0\), and select \(\eta\) by

\begin{equation}
 \alpha\eta
 =\frac{\theta}
 {s+(C_0-1)\theta+C_0\theta^2}.
 \label{persistent:eq:baseline-step}
\end{equation}
Choosing \(C_0\) first and then reducing \(c\) ensures
\[
 0<\lambda<1,\qquad \beta\eta\le c_0,\qquad \theta\le c_0,
\]
where \(c_0\) is the universal small-step constant in
\cref{thm:persistent-contraction}.  To see the middle inequality, split
into \(s\ge\theta\) and \(s<\theta\) in
\(\beta\eta=(s+1)\theta/
\{s+(C_0-1)\theta+C_0\theta^2\}\).

For the implementable algorithm, put
\[
 g=\sqrt{\alpha\eta},\quad \ell_g=1+\log(1/g),\quad
 \rho\asymp g\sqrt{\ell_g},\quad A_\eps=\Delta+\log(8/\eps),
\]
and choose
\[
 N=\left\lceil\frac{C\sqrt{\ell_g}}gA_\eps\right\rceil,
 \quad \delta_{\rm rgo}=\frac{\eps}{8(1+N\rho)},\quad
 B=C\left\{\Delta+\log\frac{8N}{\eps}+\log(2+d+\kappa)\right\}.
\]
These are the same mixing horizon and error budgets used in
\cref{persistent:thm:finite-compiler}, evaluated at the baseline step
size in \cref{persistent:eq:baseline-step}.

For fixed \(y\), split the conditional potential as
\[
 V(x)+\frac{\norm{x-y}^2}{2\eta}
 =R_y(x)+\frac{\omega_\eta}{2}\norm{x-y}^2,
 \quad
 R_y(x)=V(x)-\frac{\lambda\alpha}{2}\norm{x-y}^2,
 \quad
 \omega_\eta=\lambda\alpha+\eta^{-1}.
\]
The residual curvature satisfies
\[
 \delta_0\alpha I\preceq\nabla^2R_y
 \preceq(\beta-\lambda\alpha)I,
 \qquad
 \frac{\beta-\lambda\alpha}{\omega_\eta}=\theta,
 \qquad \omega_\eta\eta\asymp1.
\]
Thus the harmonic motion integrates both the proximal quadratic and the
known \(\lambda\alpha\)-quadratic baseline; collisions see only the
remaining curvature.

\paragraph{Local validity and implementation.}
In the coordinate \(z=\sqrt{\omega_\eta}(x-y)\), use the same
fixed-reference
process as in \cref{eq:persistent-generator}, with
\[
 b_y^{\rm res}(z)=\omega_\eta^{-1/2}
 \nabla R_y(y+\omega_\eta^{-1/2}z)
\]
in place of \(b_y\).  \cref{persistent:lem:fixed-invariance} applies after
this affine change of variables, so the process still preserves the
original conditional law.  Its normalized residual curvatures are
\[
 m_\star=\frac{\delta_0\alpha}{\omega_\eta},
 \qquad L_\star=\theta.
\]
If \(\tau_\star=2+\log(L_\star/m_\star)\), then
\[
 \frac{L_\star}{m_\star}
 =\frac{\omega_\eta}{C_0\alpha}
 =1+\frac{s}{C_0\theta},
 \qquad \tau_\star\lesssim\Lambda_\eps.
\]
Consequently \(\tau_\star L_\star\lesssim1\) and
\(\tau_\star L_\star^2d\lesssim1\), after reducing \(c\).
The definition above also gives
\(\log(1/\delta_{\rm rgo})\le C\Lambda_\eps\).  Hence
\[
 K_{\rm rgo}\theta
 \{\sqrt{d\log(1/\delta_{\rm rgo})}
    +\log(1/\delta_{\rm rgo})\}\le1
\]
after another universal reduction of \(c\).  This is exactly the affine
RGO gate in \cref{persistent:lem:affine-compiler}.

The reference and its certificate are
\[
 \norm{\omega_\eta(\widehat x_y-y)+\nabla R_y(\widehat x_y)}
 \le\sqrt{d\omega_\eta},
 \qquad
 E_y=y-\omega_\eta^{-1}\nabla R_y(\widehat x_y)
     +\omega_\eta^{-1/2}G.
\]
In the remainder of this subsection, \(\mathsf S\) denotes the rescaled
conditional quarter-turn averaged over this shifted-Gaussian reference.
\cref{persistent:lem:affine-compiler} therefore supplies the shifted
Gaussian comparison, event-clock bound, mode solver, and RGO gate with
local smoothness \(L_\star=\theta\).  The stationary mode cap remains
logarithmic: since \(\nabla R_Y(Y)=\nabla V(Y)\) and
\(\omega_\eta\ge\eta^{-1}\),
\[
 \frac{\eta}{d}\E\norm{\nabla V(Y)}^2
 \le\kappa(\beta\eta)+(\beta\eta)^2\le\kappa+1.
\]
The concentration argument in
\cref{persistent:lem:mode-cap} then applies without a polynomial query
factor.

\paragraph{Transfer estimate after rescaling.}
It remains to verify that the ideal contraction theorem sees the same
position--momentum geometry.  Let
\[
 \Sigma_y=\eta^{-1}\Cov(X\mid Y=y),
 \qquad
 D_y=I-\Sigma_y,
 \qquad
 D_z(y)=I-\omega_\eta\eta\,\Sigma_y.
\]
Thus \(\Sigma_y,D_y\) are exactly the conditional covariance quantities
defined in \cref{persistent:sec:graph}, while \(D_z\) is their rescaled
counterpart.  Because \(\omega_\eta\eta\ge1\),
\[
 D_z(y)
 =\omega_\eta\eta D_y-(\omega_\eta\eta-1)I
 \preceq\omega_\eta\eta D_y.
\]
The correctly scaled momentum observable is
\[
 M_f^{(\alpha)}
 =\frac{2}{\sqrt{\omega_\eta}}a_f(Y)^\top P.
\]
Applying the local quarter-turn estimate in the \(z\)-coordinate yields
\[
 \norm{\mathsf SL_f-M_f^{(\alpha)}}^2
 +\norm{\mathsf SM_f^{(\alpha)}+L_f}^2
 \le C\theta\eta\mathcal D_f.
\]
Its defect under the outer reflection is at most
\(C\eta\mathcal R_f\), and
\(\norm{M_f^{(\alpha)}}^2\le C\eta\norm{\nabla f}^2\).
Combining these bounds with
\cref{persistent:eq:affine-control,persistent:lem:outer-geometry} proves
the same position--momentum transfer estimate
\cref{persistent:eq:three-side}.  Momentum parity is unchanged, so
\cref{persistent:cor:factorized-errors,persistent:lem:resolvent} apply
with the same universal constants.  Notice that the outer gap is still
\(g=\sqrt{\alpha\eta}\); only the local collision scale changes from
\(\beta\eta\) to \(\theta\).

\paragraph{Complexity.}
From \cref{persistent:eq:baseline-step},
\[
 \frac1{\alpha\eta}
 =\frac{s}{\theta}+C_0-1+C_0\theta,
\]
and hence
\[
 \frac1{\sqrt{\alpha\eta}}
 \le C\left(1+\frac{\sqrt{\kappa-1}}{\sqrt\theta}\right)
 =\widetilde O\left(1+\sqrt{\kappa-1}\,d^{1/4}\right).
\]
Also \(\theta\sqrt d=O(1)\), so the clock uses only
\(\widetilde O(1)\) expected oracle queries per iteration.
The ideal contraction bound and our choice of \(N\) give TV error at most
\(\eps/2\).  With the explicit \(B\) above,
\cref{persistent:lem:warm-path-transfer,persistent:lem:mode-cap,persistent:lem:capped-clock}
bound the two cap errors by \(\eps/8\) each.
Uniform maximal coupling for the approximate refresh contributes at most
\(N\rho\delta_{\rm rgo}/2\le\eps/16\).  Thus the complete error is below
\(\eps\).  The per-iteration mode and clock costs are polylogarithmic,
while \(N\rho\) is polylogarithmic, so the same accounting gives
\[
 \TV(\Law(X_N),\pi)\le\eps,\qquad
 \E N_{\mathrm{qry}}
 =\widetilde O\left(1+\sqrt{\kappa-1}\,d^{1/4}\right),
 \qquad
 \E N_{\mathrm{rgo}}=\operatorname{polylog}(d,\kappa,1/\eps)
\]
when \(\Delta=O(1)\).  This proves
\cref{cor:persistent-near-gaussian}.

\section{Proof of the Warm-Start Wrapper}
\label{app:warm-start-removal}

This appendix proves both parts of \cref{thm:warm-removal-distributional}.
For the finite-divergence part, we first state two implementable samplers for the
localized conditionals and a capped interface for the final reflected
sampler.  We then show that the ideal coarse proximal chain reduces a finite
\(D_2\) budget to a constant, couple those steps to an algorithmic chain at
adaptive observation centers, optimize the coarse step size, and transfer the
final guarantee through the coupling.  The last subsection proves the
supplied-point part by constructing the required finite \(D_2\) budget from
one point.

The proof uses two external posterior-sampling results, stated precisely in
\cref{prop:warm-removal-posteriors}, and the standard exact-proximal
hypercontractive inequality stated in
\cref{lem:warm-removal-hypercontractivity}.  All remaining steps---including
the random observation centers, the distinction between ideal coarse and
algorithmic laws, the final error transfer, and the parameter
optimization---are proved below.

At a high level, the wrapper compares two routes.  The direct route applies a
reflected sampler immediately to the initial law.  The coarse route first
constructs two endpoints in parallel for the analysis: an ideal coarse endpoint
\(\nu_m^\dagger\), which is proved to be constant-\(D_2\)-warm, and an
algorithmic endpoint \(\widehat\nu_m\), which is close to
\(\nu_m^\dagger\) in total variation.  The same capped reflected kernel is
then applied to both endpoints.  Only \(\widehat\nu_m\) is generated by the
algorithm; \(\nu_m^\dagger\) is the ideal coarse endpoint used to prove convergence.

\subsection{Posterior Samplers Used by the Wrapper}

At a coarse step size \(u>0\), an observation \(y\) defines the
conditional density
\[
 q_y(x)\propto\exp\left\{-V(x)-\frac{\norm{x-y}^2}{2u}\right\}.
\]
Its strong-convexity and smoothness parameters are
\(m_u=\alpha+u^{-1}\) and \(L_u=\beta+u^{-1}\), so its condition number is
\(K_u=L_u/m_u\).

We use two ways to sample this conditional.  The \emph{regular interface}
first constructs a constant-warm law and then applies one of our reflected
samplers.  The \emph{harmonic-base interface} instead starts from a localized
Gaussian and invokes exact-ULD/FORS.  We abbreviate the latter by
``HB'' because its reference is the quadratic, or harmonic, part of the
conditional potential; HB does not mean a heat-bath update.  The regular
interface is advantageous in the intermediate- and large-dimensional
regimes, whereas the HB interface supplies the small-dimensional endpoint.

We will repeatedly turn expected-work routines into total kernels with a
deterministic work limit.  The following elementary observation records why
this does not introduce a polynomial dependence on the requested accuracy.

\begin{lemma}[Logarithmic work truncation]
\label{lem:warm-removal-truncation}
Suppose a routine has a deterministic number of stages and every
random-length operation is one of the following: a geometric trial count
with uniformly positive success probability; a Poisson count with known
mean \(\Lambda\); or a linearly convergent optimization for which
\[
 N(z)\le C\{1+\log(1+H(z)/a)\},\qquad \E H(Z)^2\le M^2
\]
under the comparison law relevant to that call.  For every \(\xi>0\), all
such operations can be stopped at deterministic limits so that the chance
of any stop is at most \(\xi\).  The limits add only logarithmic factors to
the original work, apart from the already charged Poisson mean \(\Lambda\).
Returning a fixed fallback output after a stop changes the output law by at
most \(\xi\) in total variation.
\end{lemma}

\begin{proof}
Geometric tails require \(O(\log(1/\xi))\) trials.  A Poisson variable of
mean \(\Lambda\) is at most
\(\Lambda+C\{\sqrt{\Lambda\log(1/\xi)}+\log(1/\xi)\}\) except with
probability \(\xi\).  Finally, Markov's inequality gives
\(H(Z)\le M/\sqrt\xi\) outside an event of probability \(\xi\); because
\(N\) depends logarithmically on \(H\), this increases its limit by only
\(O(\log(1/\xi))\).  Divide \(\xi\) by the deterministic number of
operations and use a union bound.  Coupling the capped and uncapped routines
until the first stop proves the total-variation statement.
\end{proof}

\begin{proposition}[Finite posterior interfaces]
\label{prop:warm-removal-posteriors}
Fix a law \(\nu\) with \(D_2(\nu\Vert\pi)\le B\), draw
\(X\sim\nu\) and, independently, a standard Gaussian vector
\(G\sim\cN(0,I_d)\), where \(I_d\) is the \(d\)-dimensional identity; set
\(Y=X+\sqrt uG\), and denote the law of \(Y\) by
\(\lambda\).  Suppose \(u\le1/\alpha\).  For every \(0<\delta<1/2\),
each of the two interfaces above gives a jointly measurable, work-capped
kernel \(y\mapsto\widehat P_y\) such that
\begin{equation}
 \int\TV(\widehat P_y,q_y)\,\lambda(\d y)\le\delta.
 \label{eq:warm-removal-posterior-error}
\end{equation}
For the regular interface, every call uses at most
\begin{equation}
 C_{\rm reg}(K_u,d)
 :=K_u^{7/6}d^{1/6}+K_u^{1/2}d^{1/4}
 \label{eq:warm-removal-regular-posterior}
\end{equation}
oracle queries up to factors logarithmic in \(d,K_u,B,1/\delta\).  For the
HB interface, the corresponding bound is
\begin{equation}
 C_{\rm HB}(K_u,d):=1+K_u^{2/3}d^{1/3}
 \label{eq:warm-removal-harmonic-posterior}
\end{equation}
oracle queries, again up to the same type of logarithmic factors.
\end{proposition}

\begin{proof}
\paragraph{Regular interface.}
After rescaling a
potential whose condition number is \(K\),
\citet[Theorem~7.1(ii)]{chen2026smoothed} constructs the output law of an
implementable algorithm and an analysis-only comparison law that are
arbitrarily close in total
variation, with the latter having constant order-two R\'enyi divergence
from the target.  At constant target divergence its query cost is
\(\widetilde O(K^{7/6}d^{1/6}+K)\); the second term is absorbed by the
first for \(K,d\ge1\).  Applying either reflected finalizer from
this paper to the comparison law and the same capped Markov kernel to the
algorithmic output law costs another \(\widetilde O(K^{1/2}d^{1/4})\).  Total-
variation data processing and the triangle inequality give
\cref{eq:warm-removal-regular-posterior}.

\paragraph{HB interface.}
First find a point \(z_y\) satisfying the
observable test
\(\norm{\nabla\Phi_y(z_y)}\le m_u/\sqrt{L_u}\), where
\(\Phi_y(x)=V(x)+\norm{x-y}^2/(2u)\).  Strong convexity then gives
\(\sqrt{L_u}\norm{z_y-x_y^*}\le1\), where \(x_y^*\) minimizes
\(\Phi_y\).  The Gaussian
\(G_y=\cN(z_y,(2L_u)^{-1}I_d)\) then obeys
\(D_q(G_y\Vert q_y)=O_q(d\log(\e K_u))\) for every fixed \(q>1\), by the
usual quadratic sandwich and completion of squares.  Applying
\citet[Thm. 3.2(ii)]{chen2026exact} with output order two, initialize
phase space from \(G_y\) tensored with stationary Gaussian momentum.  The
source theorem's initial budget is then
\(d+D_4(G_y\Vert q_y)=\widetilde O(d)\).  Set its R\'enyi tolerance to
\(O(\delta^2)\), project the output to position, and use
\(\TV^2\le D_2/2\).  This gives
\cref{eq:warm-removal-harmonic-posterior}.  Its high-probability query bound
also includes rejected trials.

\paragraph{Random centers and work caps.}
It remains to justify localization at the random center \(y\).  The
conditional objective has gradient \(\nabla V(y)\) at \(y\).  For the
regular interface, the observable test
\(\norm{\nabla\Phi_y(z)}\le\sqrt{2m_ud}\) implies
\(\Phi_y(z)-\min\Phi_y\le d\), as required by the imported warmifier; for
the HB interface we use the test in the preceding paragraph.  Accelerated
optimization started at \(y\) reaches either certificate in
\(\widetilde O(\sqrt{K_u})\) queries, with only logarithmic dependence on
\(1+\norm{y-x_*}\).  This cost is absorbed by both displayed bounds.  Under
the law \(\lambda\), the tail needed to cap this search uniformly follows
from \cref{lem:warm-removal-random-centers} below.  On cap failure the
kernel returns a fixed point; the cap probability is included in
\cref{eq:warm-removal-posterior-error}.
The cap depends only on the known parameters
\(\alpha,\beta,d,B,\delta\), not on the unknown law \(\lambda\).

We next verify the work cap rather than inferring it from an expected-query
statement.  The recursive warmifier has a deterministic number of stages.
Its proof controls the second moments of every center passed to its
gradient-descent subroutines, while its pointwise bound for each such solve
is logarithmic in the center gradient; see
\citet[Thms. 6.1(ii) and 6.5, their proofs, and
Lems. D.3--D.4]{chen2026smoothed}.  Thus these
solves satisfy the third case of \cref{lem:warm-removal-truncation}.  Its
remaining rejection loops have geometric tails once their observable mode
certificate holds; this is stated for the RGO component in
\citet[Thms. 3.1 and D.1]{chen2026high}.  The appended reflected
finalizer in the regular interface has only geometric rejection counts and
Poisson event counts.
The exact-ULD/FORS interface already gives a high-probability query bound.
Applying \cref{lem:warm-removal-truncation} makes the displayed work bounds
deterministic up to their logarithmic factors and charges every stop to
\cref{eq:warm-removal-posterior-error}.  Finally, finite optimization iterates,
Gaussian draws, acceptance tests, and capped stopping rules are Borel
functions of \(y\) and the random seed.  Thus \(\widehat P_y\) is a genuine
Markov kernel rather than a pointwise collection of samplers.
\end{proof}

We next isolate the final-stage property needed by the wrapper.  Stating it
as a capped Markov-kernel interface is important: after the algorithmic and
comparison paths separate, the final sampler must still have a deterministic
work limit.

\begin{proposition}[Capped reflected finalizer]
\label{prop:warm-removal-capped-finalizer}
Fix \(B\ge1\) and \(0<\eps,\xi<1/2\).  Either reflected sampler admits an
implementable realization as a total Markov kernel
\(\widehat K_{B,\eps,\xi}\) such that,
for every \(\nu\) satisfying \(D_2(\nu\Vert\pi)\le B\),
\(\TV(\nu\widehat K_{B,\eps,\xi},\pi)\le\eps+\xi\).  Up to logarithmic
factors, the kernel uses at most
\begin{equation}
 \widetilde O\left(
 \sqrt\kappa B(\sqrt{dB}+B)^{1/2}\right)
 \label{eq:warm-removal-finalizer}
\end{equation}
oracle queries.  In particular, constant \(B\) costs
\(\widetilde O(\sqrt\kappa d^{1/4})\).  The same kernel may be applied to a
second, arbitrary input law: its deterministic work bound remains valid,
and total-variation distance between the two inputs cannot increase.
\end{proposition}

\begin{proof}
The uncapped guarantees are \cref{flux:cor:general-warmness} and
\cref{eq:persistent-general-complexity}.  The implementable algorithms
have a fixed outer horizon.  Their mode solves have logarithmic pointwise
work bounds and warm-path second-moment control, their rejection counts are
geometric, and their event proposal counts are Poisson with the means
already included in their query-complexity bounds.  Hence
\cref{lem:warm-removal-truncation} applies.  For any additional error budget
\(\xi\), each implementation can therefore be replaced by a total Markov
kernel with a deterministic work bound of the same polynomial order and with at most
\(\xi\) additional total-variation error from every \(D_2\)-warm input.
For the flux method the required estimates are in
\cref{flux:prop:finite-compiler,flux:lem:clock-coupling}; for the
persistent method it follows from
\cref{persistent:lem:warm-path-transfer,persistent:lem:mode-cap,persistent:lem:capped-clock}
together with the same elementary Poisson and
geometric tail bounds.  We use this capped version below so that the query bound remains valid even if the algorithmic and comparison paths have
already separated.
The final assertion is the data-processing inequality for a common Markov
kernel.
\end{proof}

\subsection{Ideal Coarse Proximal Contraction}

Let \(S_u\) denote one exact proximal Gibbs update: observe
\(Y=X+\sqrt uG\), then draw \(X'\sim q_Y\).  The kernel is
\(\pi\)-reversible.  The following standard hypercontractive form makes its
warm-start effect explicit.

\begin{lemma}[Exact proximal warmification]
\label{lem:warm-removal-hypercontractivity}
For every density \(f=\d\nu/\d\pi\), every \(p>1\), and every \(u>0\),
\[
 \norm{S_uf}_{1+(1+\alpha u)^2(p-1)}\le\norm{f}_p.
\]
Consequently, if \(G_m=\prod_{j=1}^m(1+\alpha u_j)\), then
\begin{equation}
 D_2(\nu S_{u_1}\cdots S_{u_m}\Vert\pi)
 \le\frac{2D_2(\nu\Vert\pi)}{1+G_m^2}.
 \label{eq:warm-removal-d2-conversion}
\end{equation}
\end{lemma}

\begin{proof}
Let \(\pi_u\) be the observation marginal and define
\(A_u^*f(y)=\E[f(X)\mid Y=y]\), which maps density ratios on \(\pi\) to
density ratios on \(\pi_u\).  Its adjoint is
\(A_ug(x)=\E[g(Y)\mid X=x]\), and one Gibbs row has density operator
\(S_u=A_uA_u^*\).  Put \(c_u=1+\alpha u\).  The two-measure heat-flow
inequality of \citet[Thm. 1.2]{kook2025zeroth} states that
\(A_u^*:L^{p_0}(\pi)\to L^{p_1}(\pi_u)\) is a contraction whenever
\(p_1-1=c_u(p_0-1)\).  Duality gives the analogous gain for \(A_u\):
\(A_u:L^{p_1}(\pi_u)\to L^{p_2}(\pi)\) is a contraction when
\(p_2-1=c_u(p_1-1)\).  To see the direction of this second relation, the
conjugate-exponent form of the first one is
\(p_0'-1=c_u(p_1'-1)\), where a prime denotes the H\"older conjugate.

Thus the exponent flow through one Gibbs row is
\[
 L^{p_0}(\pi)
 \xrightarrow{\ A_u^*\ }
 L^{p_1}(\pi_u)
 \xrightarrow{\ A_u\ }
 L^{p_2}(\pi),
 \qquad
 p_2-1=c_u^2(p_0-1).
\]
The two excess-exponent gains therefore multiply and prove the first claim.
After iteration it gives
\(\norm{S_{u_1}\cdots S_{u_m}f}_2\le\norm f_{1+G_m^{-2}}\).

For \(p=1+G_m^{-2}\), interpolation between \(L^1(\pi)\) and
\(L^2(\pi)\), using \(\norm f_1=1\), gives
\(\log\norm f_p\le (p-1)D_2(\nu\Vert\pi)/p\).  Since the output
order-two divergence is twice the logarithm of its \(L^2\) density norm,
\cref{eq:warm-removal-d2-conversion} follows.
\end{proof}

Set \(r=\beta u\) and restrict \(1\le r\le\kappa\).  Then
\(\alpha u=r/\kappa\), and
\(\log(1+r/\kappa)\asymp r/\kappa\).  Therefore
\cref{eq:warm-removal-d2-conversion} makes the ideal coarse endpoint
constant-\(D_2\) warm after
\begin{equation}
 m=O\left(\frac\kappa r\log(\e B)\right)
 \label{eq:warm-removal-row-count}
\end{equation}
equal-step-size updates.  Notice also that their conditional condition number
is
\begin{equation}
 K_u=\frac{\beta+u^{-1}}{\alpha+u^{-1}}
     =\frac{\kappa(1+r)}{\kappa+r}\le1+r\le2r.
 \label{eq:warm-removal-conditional-kappa}
\end{equation}

\subsection{Random Centers and the Algorithmic Path}

The posterior centers in the preceding chain are adaptive: the next center
is sampled from the current position.  We now verify that the conditional
interfaces can be applied without assuming that these centers are fixed or
uniformly bounded.

\begin{lemma}[Random-center and path accounting]
\label{lem:warm-removal-random-centers}
Let \(\nu_0\) satisfy \(D_2(\nu_0\Vert\pi)\le B\), and run \(m\) exact
proximal updates with step sizes at most \(1/\alpha\).  Then every position law
\(\nu_j\) satisfies
\[
 \E_{\nu_j}\norm{X-x_*}^2\le\frac{C(d+B)}\alpha.
\]
Moreover, all observation-dependent mode searches along the path can be
capped so that the probability of any cap failure is at most \(\delta\),
while their query cost changes only by a factor logarithmic in
\(d,\kappa,B,m,1/\delta\).
\end{lemma}

\begin{proof}
Every exact proximal kernel preserves \(\pi\), so data processing gives
\(D_2(\nu_j\Vert\pi)\le B\).  Hence
\(\KL(\nu_j\Vert\pi)\le B\), and the \(T_2\) inequality for an
\(\alpha\)-strongly log-concave law gives
\(W_2^2(\nu_j,\pi)\le2B/\alpha\).  Combining an optimal coupling with
\(\E_\pi\norm{X-x_*}^2\le d/\alpha\) proves the moment bound.  Since
\(Y=X+\sqrt uG\) and \(u\le1/\alpha\), the same order of control holds for
the observation.

For the tail statement, let \(\mathbb P_{\nu_0}^{(m)}\) be the law of the
entire exact path, including its observations, and let
\(\mathbb P_\pi^{(m)}\) be the law of the same path started at stationarity.
Because their transition kernels agree,
\[
 \frac{\d\mathbb P_{\nu_0}^{(m)}}
      {\d\mathbb P_\pi^{(m)}}
 =\frac{\d\nu_0}{\d\pi}(X_0),
 \qquad
 \mathbb P_{\nu_0}^{(m)}(E)
 \le\e^{B/2}\mathbb P_\pi^{(m)}(E)^{1/2}
\]
for every path event \(E\), where the second inequality is
Cauchy--Schwarz.  Under the stationary path, strong log-concavity and the
independent Gaussian observation give exponential tails for
\(\norm{Y-x_*}\).  Because the conditional mode solver converges
geometrically and its stopping test is observable, a threshold with
stationary failure probability at most
\(\e^{-B}\delta^2/(4m^2)\) increases its iteration limit only
logarithmically.  A union bound followed by the displayed path inequality
makes the probability of any cap failure at most \(\delta\).  This bound is
for the whole adaptive path, so no independence between successive centers
is used, and its threshold is determined entirely by the displayed problem
parameters.
\end{proof}

\begin{lemma}[Finite approximation of the coarse path]
\label{lem:warm-removal-coarse-compiler}
Let \(\nu_m^\dagger=\nu_0S_u^m\) be the ideal coarse endpoint with \(m\) as in
\cref{eq:warm-removal-row-count}.  Using either posterior interface in
\cref{prop:warm-removal-posteriors}, there is an algorithmic endpoint
\(\widehat\nu_m\) such that
\(\TV(\widehat\nu_m,\nu_m^\dagger)\le\delta\).  Its number of oracle
queries is deterministically bounded, up to logarithmic factors, by
\(mC_{\rm reg}(K_u,d)\) or \(mC_{\rm HB}(K_u,d)\), respectively.
\end{lemma}

\begin{proof}
The row-wise coupling has the following form as long as the two paths have
not yet disagreed:
\[
\begin{array}{ccccc}
 X_{j-1}^\dagger
 &\xrightarrow{\;Y_j^\dagger=X_{j-1}^\dagger+\sqrt uG_j\;}
 &Y_j^\dagger
 &\xrightarrow{\;q_{Y_j^\dagger}\;}
 &X_j^\dagger,\\[0.25em]
 \widehat X_{j-1}
 &\xrightarrow{\;\widehat Y_j=\widehat X_{j-1}+\sqrt uG_j\;}
 &\widehat Y_j
 &\xrightarrow{\;\widehat P_{j,\widehat Y_j}\;}
 &\widehat X_j.
\end{array}
\]
The shared Gaussian \(G_j\) makes
\(Y_j^\dagger=\widehat Y_j\) whenever
\(X_{j-1}^\dagger=\widehat X_{j-1}\).  At that common observation, we use a
maximal coupling of the two conditional draws.

At row \(j\), apply \cref{prop:warm-removal-posteriors} with total error
budget \(\delta/m\); this budget includes the posterior approximation and
all of its mode and work caps.  Couple the exact and approximate
conditional draws maximally whenever the two paths have agreed so far.  To
avoid conditioning on the agreement event, first generate the entire ideal coarse
path and let \(e_j(y)=\TV(\widehat P_{j,y},q_y)\).  The probability
that the first mismatch occurs in row \(j\) is at most
\(\E[1_{\text{no earlier mismatch}}e_j(Y_j)]\), which is bounded by
the unconditional exact-path average
\(\E e_j(Y_j)\le\delta/m\).  Summing over the rows gives total mismatch at
most \(\delta\).  The simultaneous choice of all random-center caps is
justified by \cref{lem:warm-removal-random-centers}.

This is ordinary total-variation telescoping; a Hellinger chain rule is not
needed because the requested accuracy affects both posterior interfaces
only logarithmically.  Crucially, every call is work-capped.  Thus the
query bound remains valid after the first coupling disagreement, even
though the later algorithmic centers no longer have the exact center law.
The ideal coarse endpoint \(\nu_m^\dagger\), rather than the algorithmic endpoint,
is the law to which \cref{eq:warm-removal-d2-conversion} is applied.
\end{proof}

\subsection{Cost Optimization and Final Error Transfer}

\begin{proof}[Proof of the finite-divergence part of
\cref{thm:warm-removal-distributional}]
The optimization selects one of the following three coarse-stage
implementations.  The table states the resulting choices; the algebra below
derives them.
\begin{center}
\small
\begin{tabular}{@{}llll@{}}
\toprule
Regime & Interface & \(r\) & Coarse-stage cost \\
\midrule
\(d\le\kappa^{16/7}\)
& HB & \(r=\kappa\)
& \(\widetilde O(\kappa^{2/3}d^{1/3}\log(\e B))\) \\
\(\kappa^{16/7}<d<\kappa^8\)
& Regular & \(r=d^{1/8}\)
& \(\widetilde O(\kappa d^{3/16}\log(\e B))\) \\
\(d\ge\kappa^8\)
& Regular & \(r=\kappa\)
& \(\widetilde O(\sqrt\kappa d^{1/4}\log(\e B))\) \\
\bottomrule
\end{tabular}
\end{center}

Fix \(1\le r\le\kappa\), set \(u=r/\beta\), and use the regular posterior
interface.  Multiplying \cref{eq:warm-removal-regular-posterior} by the row
count in \cref{eq:warm-removal-row-count}, and using
\cref{eq:warm-removal-conditional-kappa}, gives the coarse-stage query complexity
\begin{equation*}
 Q_{\rm reg}(r)
 =\widetilde O\left(
 \kappa\{r^{-1}+r^{1/6}d^{1/6}+r^{-1/2}d^{1/4}\}
 \log(\e B)\right).
\end{equation*}
The \(r^{-1}\) term records the constant work of each row.

If \(d\le\kappa^8\), the admissible choice \(r=d^{1/8}\) balances the
last two terms and gives
\(\widetilde O(\kappa d^{3/16}\log(\e B))\).  If
\(d\ge\kappa^8\), choose \(r=\kappa\).  The two nonconstant terms are then
\(\kappa^{7/6}d^{1/6}\) and \(\sqrt\kappa d^{1/4}\); the first is no
larger than the second precisely when \(d\ge\kappa^8\).

For the remaining endpoint, again choose \(r=\kappa\), but use the HB
interface in
\cref{eq:warm-removal-harmonic-posterior}.  There are only
\(O(\log(\e B))\) rows and \(K_u=(\kappa+1)/2\), so this route costs
\(\widetilde O(\kappa^{2/3}d^{1/3}\log(\e B))\).  The harmonic and
regular-interface bounds cross at \(d=\kappa^{16/7}\), while the two regular choices
cross at \(d=\kappa^8\).  Selecting the cheapest route gives exactly
\(\widetilde O(\Psi(\kappa,d)\log(\e B))\).

We now close the distributional argument.  Choose the constants in
\cref{eq:warm-removal-row-count} so that
\(D_2(\nu_m^\dagger\Vert\pi)\le1\), and use
\cref{lem:warm-removal-coarse-compiler} with \(\delta=\eps/4\).  Lift both
\(\nu_m^\dagger\) and \(\widehat\nu_m\) through the same Gaussian
augmentation required by the selected finalizer, and append the same
independent Gaussian momentum when needed.  The lifted ideal coarse law has
the same density ratio relative to stationarity, while total-variation
data processing preserves the \(\eps/4\) discrepancy.

Invoke \cref{prop:warm-removal-capped-finalizer} with final-stage convergence
tolerance \(\eps/4\) and work-cap budget \(\eps/4\), and apply the resulting
single kernel to both augmented laws.  Its output from the ideal coarse law is
within \(\eps/2\) of stationarity.  Because applying the same Markov kernel
cannot increase total-variation distance, the two final outputs remain
within \(\eps/4\) of one another.  The triangle inequality therefore puts
the algorithmic output within \(3\eps/4\) of stationarity.  The finalizer's
\(\widetilde O(\sqrt\kappa d^{1/4})\) query cost is bounded by every branch
of \(\Psi\) in its selected regime.  This proves the second term in the
minimum of \cref{eq:warm-removal-distributional-cost}.  Alternatively,
skip the coarse stage and apply the general-budget bound
\cref{eq:warm-removal-finalizer} directly to \(\nu_0\); this proves the
first term.

Finally, every branch of \(\Psi\) is at most
\(\kappa^{6/7}d^{1/4}\).  For the first branch the ratio is
\(\kappa^{-4/21}d^{1/12}\le1\) because
\(d\le\kappa^{16/7}\).  For the second it is
\(\kappa^{1/7}d^{-1/16}\le1\) because
\(d\ge\kappa^{16/7}\), and the third follows from \(\kappa\ge1\).
This proves the claimed envelope.
\end{proof}

\subsection{Initialization from One Point}

Let \(x_*\) be the minimizer of \(V\), let
\(\log_+(t)=\max\{1,\log t\}\), and recall the localization budget
\[
 Q_{\rm loc}:=\sqrt\kappa\,
 \log_+\left[\kappa\{2+\sqrt\beta\norm{x_0-x_*}\}\right].
\]
The quantity \(Q_{\rm loc}\) appears only in the query bound: the algorithm
does not know \(x_*\) and stops using the observable gradient test stated
below.

\begin{lemma}[Observable localization]
\label{lem:warm-removal-localization}
Starting from \(x_0\), accelerated optimization followed by one explicit
Gaussian draw produces a law \(q_0\) such that
\[
 D_2(q_0\Vert\pi)\le1+\frac d2\log(2\kappa)
\]
using \(\widetilde O(Q_{\rm loc})\) oracle queries.
\end{lemma}

\begin{proof}
Run accelerated gradient descent until the observable certificate
\(\norm{\nabla V(x_{\rm ref})}\le\alpha/\sqrt\beta\) holds.  The standard
strongly-convex accelerated rate and
\(V(x_0)-V(x_*)\le\beta\norm{x_0-x_*}^2/2\) give the stated
\(Q_{\rm loc}\) bound.  Strong convexity then implies
\(\sqrt\beta\norm{x_{\rm ref}-x_*}\le1\).

Draw \(q_0=\cN(x_{\rm ref},(2\beta)^{-1}I_d)\).  Strong convexity bounds
the target normalizer by
\(Z\le\e^{-V(x_*)}(2\pi/\alpha)^{d/2}\), while smoothness gives
\(V(x)\le V(x_*)+\beta\norm{x-x_*}^2/2\).  Completing the square yields
\[
 \left\|\frac{\d q_0}{\d\pi}\right\|_\infty
 \le\e(2\kappa)^{d/2}.
\]
An \(L^\infty\) density-ratio bound by \(M\) implies
\(D_2\le\log M\), proving the result.
\end{proof}

\begin{proof}[Proof of the supplied-point part of
\cref{thm:warm-removal-distributional}]
Apply \cref{lem:warm-removal-localization}, followed by
\cref{thm:warm-removal-distributional} with
\(B=1+(d/2)\log(2\kappa)\).  Since \(\log(\e B)\) is polylogarithmic in
\(d\) and \(\kappa\), the coarse-wrapper branch costs
\(\widetilde O(\Psi(\kappa,d))\).  Adding the one-time localization cost
proves the first inequality in \cref{eq:warm-removal-cold-cost}; the
envelope proved above gives the second.
\end{proof}

\clearpage
\section{Proofs for the Persistent Quarter-Turn under a Poincar\'e Inequality}
\label{app:pi-extension}

Set \(a=m\eta\), \(h=K\eta\), and \(s=\beta\eta\). With the notation of \cref{persistent:sec:graph}, both methods act on the signal \(\mathsf C f=\mathsf Q\mathsf U\mathsf P f\) created by outer reflection. Here \(\mathsf P f=\E[f\mid Y]\) and \(\mathsf Q=I-\mathsf P\). For the step size in \cref{eq:pi-step-size}, \(a,h\le s\) and \(s\) is sufficiently small, where \(a\le s\) uses \(m\le\beta\) from \cref{lem:pi-gap-smoothness}.

\begin{lemma}[Poincar\'e gap and smoothness]\label{lem:pi-gap-smoothness}
Under \cref{ass:pi-extension}, \(m\le\beta\).
\end{lemma}
\begin{proof}
Fix a unit vector \(v\in\R^d\), let \(X\sim\pi\), and write \(\mu=\E\langle v,X\rangle\). The Poincar\'e inequality implies exponential integrability of \(\pi\), so \(\pi\) has finite moments of every order. Since \(\norm{\nabla^2V}\le\beta\), every \(C^1\) function \(g\) with \(g\) and \(\nabla g\) of at most linear growth satisfies, by the divergence theorem applied to \(g\,\e^{-V}v\) with smooth cutoffs,
\[
 \E\langle v,\nabla g(X)\rangle=\E[g(X)\langle v,\nabla V(X)\rangle].
\]
Taking \(g\equiv1\) gives \(\E\langle v,\nabla V(X)\rangle=0\). Taking \(g(x)=\langle v,x\rangle\) gives \(\E[(\langle v,X\rangle-\mu)\langle v,\nabla V(X)\rangle]=1\). Taking \(g(x)=\langle v,\nabla V(x)\rangle\) and using \(\nabla^2V\preceq\beta I\) gives \(\E[\langle v,\nabla V(X)\rangle^2]=\E\langle v,\nabla^2V(X)v\rangle\le\beta\). The Cauchy--Schwarz inequality then yields
\[
 1\le\Var\langle v,X\rangle\,\E[\langle v,\nabla V(X)\rangle^2]\le\beta\Var\langle v,X\rangle .
\]
On the other hand, the Poincar\'e inequality applied to \(f(x)=\langle v,x\rangle\) gives \(\Var\langle v,X\rangle\le1/m\). Hence \(1/\beta\le1/m\).
\end{proof}

\begin{lemma}[Reflection under PI]\label{lem:pi-outer-readable}
Let \(Y_\pm=X\pm\sqrt\eta Z\), where \(X\sim\pi\) and \(Z\sim\cN(0,I_d)\) are independent. Write \(\nu\) for the common law of \(Y_\pm\), \(\mathsf T f(y)=\E[f(Y_-)\mid Y_+=y]\), and \(\Gamma=(I-\mathsf T^2)^{1/2}\). On centered \(Y\)-observables,
\begin{equation}
 \eta\norm{\nabla\mathsf T f}^2\le c_h\norm{\Gamma f}^2,\qquad
 \Gamma^2\succeq g^2I,\qquad
 c_h=\frac{(1+h)^2}{4(1-h)},\quad
 g^2=\frac{a}{a+c_h(1+a)}\asymp a .
 \label{pi:eq:outer-gap}
\end{equation}
On full \(L^2(\nu)\), \(\mathsf T\succeq-r_0 I\) for a universal \(r_0<1\). On the spectral subspace \(\mathcal H_+=\mathbf 1_{[1/2,1]}(\mathsf T)L^2_0(\nu)\),
\[
 \sqrt\eta\norm{\nabla f}\le C\norm{\Gamma f},
 \qquad \sqrt\eta\norm{\nabla(I-\mathsf T)f}\le C\norm{\Gamma^3 f}.
\]
\end{lemma}
\begin{proof}
Conditional variance for \(f(X+\sqrt\eta Z)\), followed by Gaussian PI in \(Z\) and target PI in \(X\), gives \(\Var_\nu(f)\le(m^{-1}+\eta)\E_\nu\norm{\nabla f}^2\). The conditional law of \(Y_-\) given \(Y_+=y\) has curvature at least \((1-h)/(4\eta)\). Its \(y\)-score has derivative in \(Y_-\) bounded by \((1+h)/(4\eta)\). For every unit \(v\), conditional PI bounds the variance of the directional score by \(c_h/\eta\). The score-covariance identity of \cref{persistent:lem:outer-geometry} then gives \(\eta\norm{\nabla\mathsf Tf(y)}^2\le c_h\Var(f(Y_-)\mid y)\). Averaging gives the first bound in \cref{pi:eq:outer-gap}. Applying the PI for \(\nu\) to \(\mathsf Tf\) gives
\[
 \norm{\mathsf Tf}^2\le \frac{c_h(1+a)}a
          \bigl(\norm f^2-\norm{\mathsf Tf}^2\bigr).
\]
Rearranging proves the gap. For the other edge of the spectrum, exchangeability and Gaussian PI conditional on \(X\) give \(\langle u,(I-\mathsf T)u\rangle\le2\eta\norm{\nabla u}^2\). Take \(u=\mathsf Tf\) and use the first bound again. Each spectral value \(t<1\) satisfies \(t^2\le2c_h(1+t)\), whose negative root stays strictly above \(-1\) for small \(h\). Constants have eigenvalue one, so the lower bound applies also to vector-valued gradients componentwise. Finally \cref{persistent:lem:closed-gradient-smoothing} applied to the first inequality gives both slow-sector gradient estimates.
\end{proof}

\begin{lemma}[Signed conditional covariance]\label{lem:pi-signed-readable}
Condition on \(Y_+=y\), put \(W=(X-y)/\sqrt\eta\) and \(b_y(w)=\sqrt\eta\nabla V(y+\sqrt\eta w)\). Let \(\Sigma_y=\Cov(W\mid y)\) and write \(D_y=I-\Sigma_y=D_y^+-D_y^-\) for its positive and negative parts. Then
\begin{equation}
 -\frac h{1-h}I\preceq D_y\preceq\frac s{1+s}I,\qquad
 \Cov(b_y(W)\mid y)\preceq C(sD_y^++shI).
 \label{pi:eq:signed-covariance}
\end{equation}
\end{lemma}
\begin{proof}
The conditional negative log density has Hessian between \((1-h)I\) and \((1+s)I\). Brascamp--Lieb gives \(\Sigma_y\preceq(1-h)^{-1}I\). For the reverse inequality put \(\psi=W+b_y(W)\). Conditional integration by parts gives \(\E[(W-\E W)\psi^\top]=I\) and \(\E[\psi\psi^\top]\preceq(1+s)I\). Matrix Cauchy--Schwarz gives \(\Sigma_y\succeq(1+s)^{-1}I\).

For the force covariance, let \(C_y=\Cov(b_y(W)\mid y)\) and \(H_y=\nabla b_y\). The spectral bounds \(-hI\preceq H_y\preceq sI\) imply \(H_y^2\preceq(s-h)H_y+shI\). Conditional PI and the identity \(\E[H_y\mid y]=D_y+C_y\) therefore imply
\[
 (1-h)C_y\preceq(s-h)(D_y+C_y)+shI,\qquad
 (1-s)C_y\preceq(s-h)D_y+shI.
\]
Replacing \(D_y\) by \(D_y^+\) proves the second part of \cref{pi:eq:signed-covariance}.
\end{proof}

\begin{lemma}[The conditional signal created by reflection]\label{lem:pi-affine-readable}
For a smooth centered \(f\) with finite gradient norm, set \(u=\nabla f\), \(a_f=\mathsf T u\), \(\mathcal R_f=\norm u^2-\norm{a_f}^2\), and \(\mathcal D_f^+=\E[a_f(Y_+)^\top D_{Y_+}^+a_f(Y_+)]\). Write \(\mathsf Cf=\mathsf Q\mathsf U\mathsf P f=L_f+e_f\), where \(L_f=2\sqrt\eta\,a_f(Y_+)^\top(W-\E[W\mid Y_+])\). Then
\begin{align}
 \norm{e_f}^2 &\le \frac{4\eta}{1-h}\mathcal R_f, \label{pi:eq:affine-rem}\\
 \mathcal R_f+\mathcal D_f^+
 &\le C\norm{\nabla(I-\mathsf T)f}\norm u+Ch\norm u^2 .
 \label{pi:eq:signed-affine}
\end{align}
\end{lemma}
\begin{proof}
The conditional \(W\)-gradient of \(e_f\) is \(2\sqrt\eta(\nabla f(Y_-)-a_f(y))\). Conditional PI gives \cref{pi:eq:affine-rem}. For the remaining estimate use the exact identity from the proof of \cref{persistent:lem:affine-control},
\[
 \nabla(I-\mathsf T)f=u-a_f+2D_Ya_f+\eta^{-1/2}\Cov(e_f,b_Y(W)\mid Y).
\]
Take its inner product with \(w=u+a_f\) and average over \(Y\). Separating \(D_Y^+\) and \(D_Y^-\) gives
\begin{align*}
 \mathcal R_f+4\mathcal D_f^+
 ={}&\langle\nabla(I-\mathsf T)f,w\rangle
   -2\E[a_f^\top D_Y(u-a_f)]\\
 &-\eta^{-1/2}\E[w^\top\Cov(e_f,b_Y(W)\mid Y)]
   +4\E[a_f^\top D_Y^-a_f].
\end{align*}
The last term is at most \(Ch\norm u^2\) by \cref{pi:eq:signed-covariance}. For the other two error terms, the lower spectral edge of \(\mathsf T\) gives \(\norm{u-a_f}^2\le C\mathcal R_f\). The positive part of \(D_Y\) therefore contributes at most \(C\sqrt{s\mathcal D_f^+\mathcal R_f}\) to \(\abs{\E[a_f^\top D_Y(u-a_f)]}\), and its negative part contributes at most \(Ch\norm u\sqrt{\mathcal R_f}\).

For the conditional covariance term, use \(w=2a_f+(u-a_f)\). Then \(\E[w^\top D_Y^+w]\le C(\mathcal D_f^++s\mathcal R_f)\). Conditional Cauchy--Schwarz, \cref{pi:eq:affine-rem}, and the force covariance bound in \cref{pi:eq:signed-covariance} give
\[
 \eta^{-1/2}\bigl\lvert\E[w^\top\Cov(e_f,b_Y(W)\mid Y)]\bigr\rvert
 \le C\sqrt{\mathcal R_f}\,
   \big(s\mathcal D_f^++s^2\mathcal R_f+sh\norm u^2\big)^{1/2}.
\]
Since \(s\) is small, Young's inequality absorbs a fixed fraction of \(\mathcal R_f+\mathcal D_f^+\) on the left. The remaining curvature error is \(Ch\norm u^2\), and \(\norm w\le2\norm u\) bounds the gradient term. This proves \cref{pi:eq:signed-affine}.
\end{proof}

\begin{proof}[Proof of \cref{thm:pi-extension}(ii), persistent quarter-turn]
\emph{Position--momentum transfer.} First use an exact reference \(E\sim q_y\), independent of the incoming conditional position and momentum. Conditional curvature is at least \(\eta^{-1}-K>0\). The fixed-reference invariance and momentum reversal in \cref{persistent:lem:fixed-invariance} thus apply, as does the refresh dissipation of \cref{persistent:lem:row-identities}.

Let \(\mathsf S\) be the conditional quarter-turn and \(M_f=2\sqrt\eta\,a_f(Y)^\top P\). For a direction \(c\), the forcing calculation in \cref{persistent:lem:joint-quarter} gives the second moment \(11\Var(c^\top b_y(W)\mid y)\). Apply \cref{pi:eq:signed-covariance} with \(c=2\sqrt\eta\,a_f(y)\) and use \cref{pi:eq:affine-rem} for the remainder of \(\mathsf C f\). The first two errors below are then bounded by \(C\eta(\mathcal R_f+s\mathcal D_f^++sh\norm u^2)\). For the third, reflection exchanges \(Y_+\) and \(Y_-\), giving \(\norm{\mathsf U M_f-M_f}^2\le C\eta\mathcal R_f\). Now \cref{pi:eq:signed-affine} and the two gradient estimates of \cref{lem:pi-outer-readable} give
\begin{equation}
 \norm{\mathsf S\mathsf C f-M_f}^2+
 \norm{\mathsf S M_f+\mathsf C f}^2+
 \norm{\mathsf U M_f-M_f}^2
 \le C\norm{\Gamma f}\norm{\Gamma^3 f}+Ch\norm{\Gamma f}^2 .
 \label{pi:eq:leakage}
\end{equation}
The graph closure following \cref{persistent:eq:pre-spectral-transfer} gives this estimate on \(\mathcal H_+\).

\emph{Contraction of the exact row.} Let \(\ell_g=1+\log(1/g)\) and \(A=\ell_g(1+h/g^2)\). Substitute \(f=\Gamma^{-1}w\) in \cref{pi:eq:leakage} for each normalized error \(F_j\) of \cref{persistent:eq:normalized-transfer}. Then
\[
 \norm{F_jw}^2\le C\norm w\norm{\Gamma^2w}+Ch\norm w^2
 \le C'\norm w\norm{(\Gamma^2+hI)w}.
\]
Partition the spectrum of \(\Gamma^2+hI\) into dyadic intervals between \(g^2+h\) and \(1+h\). On each interval the last inequality bounds the squared error by a constant times its upper endpoint times \(\norm w^2\). Cauchy--Schwarz over the \(O(\ell_g)\) intervals, as in \cref{persistent:cor:factorized-errors}, gives
\[
 F_j^*F_j\preceq C\ell_g(\Gamma^2+hI)\preceq CA\Gamma^2 .
\]
The second inequality uses \(\Gamma^2\succeq g^2 I\). This is where the curvature penalty \(1+h/g^2\asymp1+K/m\) enters. Setting \(B_j=F_j\Gamma^{-1}\) gives \(F_j=B_j\Gamma\) with \(\norm{B_j}\le C\sqrt A\).

Choose \(\rho=c_\rho g\sqrt A\le1/4\). In the proof of \cref{persistent:lem:resolvent}, the three-coordinate system and its scalar adjugate still apply. For \((z-\mathsf K)f=b\), \(\abs z=1\), the two normalized error terms now satisfy \(\norm{\Gamma^{-1}r_1}+\norm{\Gamma^{-1}r_2}\le C\sqrt A\norm{\mathsf Q f}\). Substitution in the near- and far-spectrum estimates of that proof gives
\[
 \norm{\mathsf P f}\le C\sqrt A(\norm{\mathsf Q f}+\norm b/g),
 \qquad \rho\norm{\mathsf Q f}^2\le2\norm f\norm b.
\]
Young's inequality and \(A/\rho\asymp\sqrt A/g\) give the resolvent bound \(C\sqrt A/g\). The surjectivity argument in that lemma and \cref{persistent:lem:resolvent-decay} yield \(\norm{\mathsf K^n}_{L^2_0\to L^2_0}\le2\e^{-cng/\sqrt A}\). Augmentation preserves \(D_2\), so the TV calculation following \cref{persistent:prop:ideal-contraction} gives ideal error at most \(\eps/2\) after
\begin{equation}
 N=\left\lceil C\frac{\sqrt A}{g}
     \left(1+\Delta+\log\frac{8}{\eps}\right)\right\rceil .
 \label{pi:eq:ideal-horizon}
\end{equation}

\emph{Replacing the exact reference on accelerated rows.} Fix a small numerical \(b_0>0\), set \(\delta=b_0m\), and write the conditional potential as
\[
 \Phi_y(x)=R_y(x)+\tfrac{\omega}{2}\norm{x-y}^2,\qquad
 R_y(x)=V(x)+\tfrac{K+\delta}{2}\norm{x-y}^2,\qquad
 \omega=\eta^{-1}-K-\delta.
\]
Now \(\delta I\preceq\nabla^2 R_y\preceq(\beta+K+\delta)I\). Apply \cref{persistent:lem:affine-compiler} with its two curvature parameters equal to \(\mu_r=\delta/\omega\) and \(L_r=(\beta+K+\delta)/\omega\); we rename them here because \(m\) now denotes the Poincar\'e gap. Starting at \(y\), find a point \(\widehat x_y\) with \(\norm{\omega(\widehat x_y-y)+\nabla R_y(\widehat x_y)}\le\sqrt{d\omega}\). Use the reference \(E_y=y-\omega^{-1}\nabla R_y(\widehat x_y)+\omega^{-1/2}G\), \(G\sim\cN(0,I_d)\), in the normalized harmonic motion for time \(\pi/2\).

It remains to check that this different, conditionally invariant kernel satisfies \cref{pi:eq:leakage}. In coordinates \(z=\sqrt\omega(x-y)\), its covariance weight satisfies
\[
 D_y^\omega=I-\omega\Cov(X\mid y)
   =\theta D_y+(h+b_0a)I
   \preceq\theta D_y^++(h+b_0a)I,\qquad \theta=\omega\eta\asymp1 .
\]
The affine position signal has coefficient \(2a_f(y)/\sqrt\omega\) in \(z\), so its momentum partner is \(M_f^\omega=2a_f(Y)^\top P/\sqrt\omega\). Using the local forcing bound of \cref{persistent:lem:affine-compiler} with this coefficient bounds the three squared errors by \(C\eta[\mathcal R_f+s\mathcal D_f^++s(h+a)\norm u^2]\). The additional \(a\)-term is absorbed because
\[
 a\norm{\Gamma f}^2\le C\norm{\Gamma^2 f}^2
       \le C\norm{\Gamma f}\norm{\Gamma^3 f}.
\]
Therefore the same factorization, resolvent estimate, and horizon \cref{pi:eq:ideal-horizon} apply to the shifted-reference chain.

\emph{Finite accuracy and cost.} Use \(\delta_R=\eps/[16(1+N\rho)]\) for each refresh RGO and set
\[
 B=C[1+\Delta+\log(N/\eps)+\log(d+\beta/m+K/m)].
\]
To check the parameter choices, use \(g^2\asymp m\eta\) and \(A\asymp[1+\log(1/(m\eta))](1+K/m)\) in \cref{pi:eq:ideal-horizon}. This gives \(\log N\lesssim\Lambda_{\mathrm{PI}}\) and \(B\lesssim\Lambda_{\mathrm{PI}}\). Moreover,
\[
 g^2A\lesssim (m+K)\eta[1+\log(1/(m\eta))],\qquad
 L_r\lesssim\beta\eta,\qquad \tau=2+\log(L_r/\mu_r)\lesssim\Lambda_{\mathrm{PI}}.
\]
Since \(K\le\beta\) by \cref{ass:pi-extension} and \(m\le\beta\) by \cref{lem:pi-gap-smoothness}, the step size \cref{eq:pi-step-size} makes \(g^2A,\tau L_r,\tau L_r^2d\) sufficiently small. It also meets the RGO gate and gives \(L_r(\sqrt{dB}+B)=O(1)\).

Cutoff integration by parts gives \(\E_\pi\norm{\nabla V(X)}^2\le\beta d\). PI concentration applied to the \(\beta\)-Lipschitz function \(\norm{\nabla V}\) gives, with probability at least \(1-\e^{-B}\),
\[
 \norm{\nabla V(X)}\le\sqrt{\beta d}+C\beta B/\sqrt m .
\]
Together with \(Y-X\sim\cN(0,\eta I_d)\), this controls \(\norm{\nabla V(Y)}\) at the start of the conditional solve. Its condition number is at most two. Hence gradient descent reaches the certificate within \(C[1+\log(d+\beta/m+B)]\) queries except on a stationary event of probability \(C\e^{-cB}\). Return the pre-row state on failure. The stationary clock-cap error from \cref{persistent:lem:affine-compiler} is at most \(C(1+L_r\sqrt d)\e^{-cB}\).

Apply \cref{persistent:lem:warm-path-transfer} to each cap event on the ideal path, then sum over \(N\) rows. Increasing the constant in \(B\) bounds the contribution of each cap by \(\eps/8\). Maximal coupling of the RGO is needed only on refresh rows, so its total error is at most \(N\rho\delta_R\le\eps/16\). These three errors together with the ideal mixing error \(\eps/2\) give the TV claim.

Every row uses polylogarithmically many queries in expectation. The two costs therefore satisfy
\[
 \E N_{\mathrm{qry}}\le\widetilde O(N),\qquad
 \E N_{\mathrm{rgo}}\le N\rho\le CA\big(1+\Delta+\log(8/\eps)\big).
\]
Substitute \(g^2\asymp m\eta\), \(A\asymp[1+\log(1/(m\eta))](1+K/m)\), and \cref{eq:pi-step-size} in \cref{pi:eq:ideal-horizon} to obtain \cref{eq:pi-query-bound,eq:pi-rgo-bound}.
\end{proof}

\section{Proofs for the Flux-Resampled Half-Turn under a Poincar\'e Inequality}
\label{app:flux-pi}

We use \(a,h,s,g,A\) from \cref{app:pi-extension}, where \(g\) plays the role of the strongly log-concave gap \(\gamma_a\) of \cref{app:flux} and the normalized force \(b_y\) is written \(g_y\) there. In \cref{alg:flux-half-turn} take \(\gamma=\gamma_0\sqrt{as/(1+a)}\), \(0<\gamma_0\le1\), and \(\rho=c_\rho g\sqrt A\). The conditional curvature is at least \((1-h)/\eta\) and the residual force is Lipschitz. Thus the Lyapunov and momentum-reversal calculations of \cref{flux:prop:fixed-validity} still give a reversible position kernel \(F_y^w\). In particular, its restriction \(\mathsf F_Q\) to conditionally centered observables is self-adjoint.

\begin{lemma}[Half-turn error under negative curvature]\label{lem:fluxpi-readable}
For a centered affine observable \(\ell_u(w)=u^\top(w-\E[W\mid y])\), the conditional half-turn satisfies
\begin{equation}
 \norm{(I+F_y^w)\ell_u}^2
 \le C\{s u^\top D_y^+u+sh\norm u^2+\gamma^2u^\top\Sigma_yu\}.
 \label{fluxpi:eq:local}
\end{equation}
On \(\mathcal H_+\) this yields the global estimate
\begin{equation}
 \norm{(I+\mathsf F_Q)\mathsf C f}^2
 \le C\norm{\Gamma f}\norm{\Gamma^3f}+Ch\norm{\Gamma f}^2 .
 \label{fluxpi:eq:global}
\end{equation}
\end{lemma}
\begin{proof}
At the half-period, the momentum term in \cref{flux:eq:halfturn-duhamel} vanishes. The squared damping error is at most \(C\gamma^2u^\top\Sigma_yu\). For an independent conditional reference the forcing has second moment \((6+\pi/2)\Var(u^\top b_y(W)\mid y)\) by the identity in \cref{flux:eq:forcing-moment}. Apply \cref{pi:eq:signed-covariance} to this variance. Contractivity and Cauchy--Schwarz in the bounded time integral give \cref{fluxpi:eq:local}.

For the global estimate insert \(u=2\sqrt\eta\,a_f(y)\) in \cref{fluxpi:eq:local}, with \(a_f=\mathsf T\nabla f\) as in \cref{lem:pi-affine-readable}. Use \(\norm{I+\mathsf F_Q}\le2\) on \(e_f\). Since \(\Sigma_y\preceq(1-h)^{-1}I\) and \(\gamma^2\le as\), the resulting bound is
\[
 \norm{(I+\mathsf F_Q)\mathsf C f}^2
 \le C\eta\{\mathcal R_f+s\mathcal D_f^+
                     +s(h+a)\norm{\nabla f}^2\}.
\]
Apply \cref{pi:eq:signed-affine} and the gradient estimates of \cref{lem:pi-outer-readable}. The damping contribution \(Cas\norm{\Gamma f}^2\) is absorbed into \(\norm{\Gamma f}\norm{\Gamma^3f}\) because \(a\norm{\Gamma f}^2\lesssim\norm{\Gamma^2 f}^2\le\norm{\Gamma f}\norm{\Gamma^3 f}\). Graph closure gives \cref{fluxpi:eq:global} for all \(f\in\mathcal H_+\).
\end{proof}

\begin{proof}[Proof of \cref{thm:pi-extension}(i), flux half-turn]
\emph{Contraction.} On the slow sector, insert \(f=\Gamma^{-1}w\) in \cref{fluxpi:eq:global}. For \(\mathsf A_{\rm ht}=(I+\mathsf F_Q)\mathsf V\), dyadic comparison on \(\Gamma^2+hI\) gives \(\mathsf A_{\rm ht}^*\mathsf A_{\rm ht}\preceq C\ell_g(\Gamma^2+hI)\preceq CA\Gamma^2\). On the complementary sector \(\Gamma^2\succeq\min(1-r_0^2,3/4)I\) and \(\norm{\mathsf A_{\rm ht}}\le2\). Combining the two estimates gives \(\mathsf A_{\rm ht}=B_0\Gamma\), \(\mathsf A_{\rm ht}^*=\Gamma B_0^*\), with \(\norm{B_0}^2\le CA\).

In the notation of \cref{flux:lem:one-step-energy}, \(e_{\rm ht}=\mathsf A_{\rm ht}^*\mathsf Q f\). Hence, with \(\Lambda_\rho=A+(\rho/g)^2\),
\[
 \norm{\Gamma^{-1}e_\rho}^2\le C\Lambda_\rho\norm{\mathsf Qf}^2,
 \qquad \vartheta=c_\star\rho/\Lambda_\rho .
\]
This bounds the cross term in the energy increment of that lemma. Energy equivalence follows from \(\rho/\Lambda_\rho\le g/(2\sqrt A)\). Young's inequality and refresh dissipation then give
\[
 \mathcal E_\vartheta(\mathsf Kf)-\mathcal E_\vartheta(f)
 \le -\tfrac{\vartheta}{2}\norm{\mathsf P f}^2
      -\tfrac{\rho}{2}\norm{\mathsf Q f}^2 .
\]
With the chosen \(\rho\), iteration yields \(\norm{\mathsf K^n}_{L^2_0\to L^2_0}\le\sqrt3\,\e^{-cng/\sqrt A}\). The augmented initial law has the same \(D_2\) as \(\nu_0\), so \cref{pi:eq:ideal-horizon} gives TV error at most \(\eps/2\) for the ideal position marginal.

\emph{Finite rows.} Use the construction of \cref{flux:app:implementation} with tolerance \(\delta_R=\eps/(16N)\) for every conditional draw. Use the same \(B\) as in \cref{app:pi-extension}, a mode-solver cap \(C[1+\log(d+\beta/m+B)]\), and clock cap \(Cs(\sqrt{dB}+B)\). After reflection the normalized initial residual at the incoming position is at most \(\sqrt{\eta/d}\norm{\nabla V(X)}+\norm Z/\sqrt d\). The PI tail estimate above therefore also controls this solver, run with gradient-descent step \(\eta/(1+s)\). The stationary mode and clock error probabilities are at most \(C\e^{-cB}\) by conditional strong log-concavity and the excess-intensity argument in \cref{flux:lem:clock-coupling}.

Couple the finite and ideal paths until the first disagreement. A cap can be charged to its occurrence on the unconditional ideal path. The certified RGO can be coupled uniformly to the ideal draw with TV error \(\delta_R\) by \citet[Thms. 3.1 and D.1]{chen2026high}. Here a draw is required on every completed row. Thus \cref{persistent:lem:warm-path-transfer} and a union bound give
\[
 \Pr(\text{disagreement by row }N)
 \le 2N\e^{\Delta/2}\sqrt{C\e^{-cB}}+N\delta_R
 \le \eps/4+\eps/16
\]
on increasing the constant in \(B\). Adding the ideal error \(\eps/2\) proves the TV claim.

Finally, \(B\lesssim\Lambda_{\mathrm{PI}}\) and \(\log(1/\delta_R)\lesssim\Lambda_{\mathrm{PI}}\). Hence \cref{eq:pi-step-size} satisfies the RGO gate and gives \(s(\sqrt{dB}+B)=O(1)\). Exact affine OU transitions at homogeneous Poisson proposals cost \(O(1)\) expected event tests per row. The mode-solver and RGO costs are polylogarithmic. Substitution in \cref{pi:eq:ideal-horizon} gives the query bound, with at most \(N\) RGO calls.
\end{proof}

\end{document}